\documentclass[11pt]{article}
\usepackage[utf8]{inputenc}
\usepackage{fullpage}

\usepackage[sc]{mathpazo} %
\usepackage[T1]{fontenc} %
\usepackage{microtype} %

\usepackage{hyperref}
\hypersetup{hidelinks}
\hypersetup{bookmarksdepth=3}

\usepackage{common}

\usepackage[%
    backend=bibtex,
    firstinits=true, %
    maxbibnames=99,
    style=alphabetic, %
    natbib=true, %
]{biblatex}
\DefineBibliographyStrings{english}{%
  bibliography = {References},
}

\usepackage{multirow}
\usepackage{makecell}
\usepackage{pbox}
\usepackage{enumitem}
\setlist{noitemsep}

\let\oldi\i
\renewcommand{\i}{\textit{(i)}\xspace}
\newcommand{\ii}{\textit{(ii)}\xspace}
\newcommand{\iii}{\textit{(iii)}\xspace}
\newcommand{\iv}{\textit{(iv)}\xspace}

\newenvironment{LIST} {\begin{list} {\small$\bullet$}
    {\setlength{\leftmargin}{6mm} \setlength{\parsep}{1mm}
      \setlength{\itemsep}{0mm} \setlength{\topsep}{0mm}}}
  {\end{list}}

\newcommand{\from}[2]%
{\noindent\textcolor{red}{\textsc{Note from #1:}\textsf{ #2}}}
\newcommand{\funcf}{\textsf{\textit{fn}}\xspace}

\providecommand{\LocROne}{\ensuremath{\sv{loc}(\constant{r1})}\xspace}

\def \decomposition#1 {\xrightarrow{~#1~}_D}
\def \insertion#1 {\xrightarrow{~#1~}_I}
\def \restriction#1 {\xrightarrow{~#1~}_R}
\def \fullchronicle {\ensuremath{\phi = (\plan_\phi, \assertions_\phi, \constraints_\phi)}\xspace}

\def \FapePlan {\procedure{Fape}}

\usepackage{todonotes}

\usepackage{authblk}
\begin{document}

\title{~\\FAPE: a Constraint-based Planner for \\ Generative and Hierarchical Temporal Planning \\~\\~}

\author[1]{Arthur Bit-Monnot}
\author[1]{Malik Ghallab}
\author[1]{Félix Ingrand}
\author[2]{David E. Smith}
\affil[1]{LAAS-CNRS, Université de Toulouse, INSA, CNRS, France}
\affil[2]{Formerly at NASA Ames Research Center, USA}
\date{~}

\maketitle

\begin{abstract}

Temporal planning offers numerous advantages when based on an expressive representation. Timelines have been known to provide the required expressiveness but at the cost of search efficiency. We propose here a temporal planner, called FAPE, which supports many of the expressive temporal features of the ANML modeling language without loosing efficiency.

FAPE's representation coherently integrates flexible timelines with hierarchical refinement methods that can provide efficient control knowledge. A novel reachability analysis technique is proposed and used to
  develop {\em causal networks} to constrain the search space. It is employed for the design of informed heuristics, inference methods and efficient search strategies. Experimental results on common benchmarks in the field permit to assess the components and search strategies of FAPE, and to compare it to IPC planners. The results show the proposed approach to be competitive with less
  expressive planners and often superior when hierarchical control knowledge is provided. FAPE, a freely available system, provides other features, not covered here, such as the integration of planning with acting, and the handling of sensing actions in partially observable   environments.
\end{abstract}

\newpage
\setcounter{tocdepth}{2}
{\footnotesize  \tableofcontents}

\newpage

\section{Introduction}
\label{sec:intro}

There are numerous advantages in making time a central entity in planning. An explicit time representation permit in particular: 
\begin{LIST}
\item to model action duration;
\item to model effects, conditions, and resources borrowed or
  consumed by an action at various moments along its span,
  as well as delayed effects;
\item to handle goals with relative or absolute temporal constraints;
\item to plan with respect to expected future exogenous events;
\item to plan with actions that maintain a value while being
  executed, as opposed to just changing that value (e.g., tracking a
  target, keeping a spring latch in position);
\item to handle the concurrency of actions that have interacting and
  joint effects; and
  \item to allow for flexible plan execution through dispatching and synchronization mechanisms.
\end{LIST}
Several success stories in the application of automated planning  have illustrated these advantages, e.g., in logistics \cite{Wilkins:2008bv} or space planning \cite{Jonsson2000}.

Planning with explicit time relies either on a \textit{state-space representation} or a \textit{timeline representation}. The former uses temporally
qualified durations between states, i.e., instantaneous snapshots of the entire system.  The latter describes possible evolutions of
individual state variables over time, i.e., partial local views of state trajectories,  together with temporal constraints between elements of timelines.

Recent contributions to temporal planning have generally used a state-space representation based on
PDDL2.1~\cite{Fox2003}. This is explained by the wealth of search techniques, heuristics, and test domains  that have been developed for state-space planning, resulting in significant performance improvements. However, these planners, but of a few exceptions, do
not have all the abilities listed above, in particular the handling  of required concurrency.

The timeline representation is more expressive, as it focuses on local changes and partial plans, which correspond to entire sets of states.
A few timeline modeling languages have been proposed, e.g.,  IxTeT \cite{Ghallab1994}, AML \cite{Rabideau1999}, NDDL \cite{Frank2003}, or the
more expressive ANML \cite{anml2008}. The timeline representation is often perceived as having an efficiency drawback compared to state-space
planning: it mostly relies on plan-space search algorithms and CSP techniques, for which known heuristics are not as efficient as state-space heuristics. 
Only a few planners have been proposed for the timeline languages. In particular, to our knowledge, no planner but FAPE fully support the temporal or the hierarchial features of the powerful ANML language.

Our purpose in developing FAPE (\textit{Flexible Acting and Planning Environment}) was mainly to benefit from the expressivity of a representation combining timelines and task hierarchy. We developed a reachability analysis and heuristics which make our
timeline approach competitive. We specifically report here on the following contributions:
\begin{LIST}
\item[1.] We introduce a temporal planning representation, which consistently blends timelines with hierarchical refinement methods to allow for additional control knowledges;
\item[2.] We present a planning algorithm for the proposed representation, analyze its search space and prove its soundness and completeness;
\item[3.] We propose an original reachability analysis method for a relaxed problem integrating earliest times; we use it to develop specific
  causal networks and potential causal chains toward goals; we exploit the preceding structures for the design of informed heuristics;
\item[4.] We propose search strategies that take advantage of the designed heuristics and the hierarchical control knowledge, when available;
\item[5.] We present empirical evaluation results on classical benchmarks of the field, which allows us to assess the various components and search strategies of FAPE,  and to compare it to IPC planners; the results show that FAPE is competitive with less expressive planners, and generally superior when  hierarchical control knowledge is introduced.
\end{LIST}
Other features of FAPE are not covered here for the sake of focus, e.g., the integration with acting, or observation actions in partially observable environments \cite{bit-monnot-2016-ijcai,bit-monnot2016}.\footnote{FAPE and the tested domains are publicly available at: \url{https://github.com/laas/fape}}

The following five sections detail successively the five preceding items. 
Related work is discussed in \autoref{sec:soa},  before a conclusion.

\newcommand{\taskdependent}{dependent}

\section{A Hierarchical Time-oriented Representation}

Our representation is based on the ANML language \cite{anml2008}.
We use the same notion of actions whose conditions and effects are stated as assertions on multi-valued state variables.
Assertions are related to an action's start and end timepoints through temporal constraints and can appear at arbitrary times.
As in ANML, our representation allows for describing the expected evolution of the environment and temporally extended goals.

A subtle difference with ANML lies in the encoding of hierarchical knowledge.
A high-level action in ANML has subactions expressed through additional assertions on dedicated state variables.
Instead, we provide a separate  notion of \textit{task} which relates to the HTN planning.
This is complemented by an optionnal notion of \emph{task-dependency}, encoding the fact that an action must be part of a task network as in a classical HTN.
Task-dependency enables to handle generative planning as well as HTN planning domains in a uniform way.

\begin{example}[Running Example]\label{fig-running-example}
  For the purpose of illustrating the representation, we use a restricted version of the \emph{dock worker} domain of \cite{Ghallab2004}. There, a  harbor with several connected \emph{docks} is served by automated \emph{trucks} and \emph{cranes} that move containers between different locations to load and unload ships.  We have three primitive actions:
  \begin{LIST}
  \item \act{Move}$(r, d, d')$: truck $r$ moves from dock $d$ to dock $d'$ if the two  are \sv{connected}. A dock can contain a single truck.
    The duration depends on the distance between $d$ and $d'$.
  \item \act{Load}$(r,c,d)$: container $c$ is loaded onto a truck $r$ if both are in the same dock $d$ where a crane is available to perform the loading.
  \item \act{Unload}$(r,c,d)$: container $c$ is unloaded from a truck $r$ in a dock $d$ if $c$ is currently on $r$ and $r$ is at dock $d$ where a crane is available to perform the task.
  \end{LIST}
We are interested in planning the loading and unloading of  ships that will be docked at a specific location for a limited amount of time.\EndOfTheoremMark
\end{example}

\subsection{Main Components}

\paragraph{Temporal Variables and Constraints.}
We use a quantitative time representation based on timepoints.
We rely on temporal variables (e.g. $t,t_1$), designating a timepoint and ranging over integers.
Temporal variables are constrained through the usual arithmetic operators to specify absolute (e.g. $t \geq 9$) or relative constraints (e.g. $t_1 +1 \leq t_2 \leq t_3-2$).
Temporal variables are attached to specific events such as the start of an action or the instant at which a given condition must be fulfilled.

\paragraph{Atemporal Variables and Constraints.}
We consider a finite set of domain constants \objects, e.g., docks or trucks in \autoref{fig-running-example}.
A \emph{type} as a subset of \objects, e.g., $\constant{Docks} = \set{\constant{dock1}, \constant{dock2}, \constant{dock3}}$.
A type can be composed from other types by union or intersection (e.g. $\constant{Vehicles} = \constant{Ships} \cup \constant{Trucks}$).

An \emph{object variable} $o$ with type $T$ is a variable whose domain $\domain{o}$ is a subset of $T$.
A \emph{numeric variable} $i$ is a variable whose domain is a finite subset of the integers.%
\footnote{Note that numeric variables are distinct from temporal variables, which can have infinite domains.}

A \emph{constraint} over a set of variables $\set{x_1,\dots,x_n}$ is a pair $(\tuple{x_1,\dots,x_n}, \gamma_R)$ where $\gamma_R \subseteq \domain{x_1} \times \dots \times \domain{x_n}$ is the relation of the constraint, giving the allowed values for the tuple of variables. $\gamma_R$ can be given as a table of allowed values or a function, e.g., $\function{travel-time}(d,d') = \delta$ is met when $\delta$ is the time it takes from $d$ to $d'$.

Numeric variables can also appear in temporal constraints.
For instance, ($\function{travel{-}time}(d,d') = \delta) \AND (t_s + \delta \leq t_e$) enforces a delay $\delta$ between the timepoints $t_s$ and $t_e$  whose value is constrained by the time needed to travel from $d$ to a location $d'$.

\paragraph{State variables and fluents.}
The state evolution over time is captured by multi-valued \emph{state variables}.
A state variable maps time and a tuple of objects to an object.
For instance $\sv{loc}: \constant{Time} \times \constant{Trucks} \to \constant{Docks}$ gives the position of a truck over time.
The time parameter of state variables is usually kept implicit and we say that the state variable $sv(x_1,\dots,x_n)$ has the value $v$ at time $t$ meaning $sv(t,x_1,\dots,x_n) = v$.
A complete definition of the state of the environment at a given time is specified by taking a snapshot of all state variables at that moment.

A \textit{fluent} is a pair of a state variable $sv$ and its value $v$, denoted $\fluent{sv}{v}$, and is said to hold at time $t$ if $sv$ has the value $v$ at time $t$.

\paragraph{Temporally Qualified Assertions.}
An \textit{assertion} is a temporally qualified fluent.
We use \textit{persistence} and \textit{change} assertions  to express knowledge or constraints on the evolution of a state variable (as in \cite[Sec. 4.2.1]{Ghallab2016}).

A \emph{persistence} assertion, denoted $\persistencetuple{t_s,t_e}{sv}{v}$, requires the state variable $sv$ to keep the same value $v$ over the temporal interval $[t_s,t_e]$.
In planning domains, persistence assertions are typically used to express requirements such as goals or pre-conditions of  actions.
For instance, the persistence assertion $\persistencetuple{400,500}{\sv{loc}(\constant{r1})}{\constant{dock2}}$ can represent the objective that the truck $\constant{r1}$ is at $\constant{dock2}$ at time 400 and stays there through time 500.
We also allow a persistence assertion to be defined at an instant $t$ (denoted $[t]$) rather than over an interval.

A \emph{change} assertion, denoted $\changetuple{t_s,t_e}{sv}{v_1}{v_2}$, asserts that the state variable $sv$ changes from having the value $v_1$ at time $t_s$ to having value $v_2$ at time $t_e$.
It expresses knowledge on the evolution of the environment, whether it results from the proper dynamics of the environment or is a consequence of an agent's activity.
For an action, a change assertion represents a pre-condition that must be met ($sv$ must have the value $v_1$ at time $t_s$) and an effect of the action ($sv$ will have the value $v_2$ at time $t_e$).
Over the temporal interval $]t_s,t_e[$ the value of $sv$ is unspecified.
This allows to represent durative change on a state variable without explicitly specifying the details of the transition.
For instance, the change assertion $\changetuple{100,150}{\LocROne}{\constant{dock1}}{\constant{dock2}}$, means that the truck \constant{r1} will move from \constant{dock1} to \constant{dock2} over the temporal interval $[100,150]$.
The details of its location over $]100,150[$ are unspecified.

A temporally qualified \emph{assignment}, denoted  $\assignmenttuple{t}{sv}{v}$, asserts that the state variable $sv$ will take the value $v$ at time $t$.
For instance the assignment $\assignmenttuple{0}{\LocROne}{\constant{dock1}}$ states that the truck \constant{r1} is at location $\constant{dock1}$ at the initial time.
An assignment $\assignmenttuple{t}{sv}{v}$ is a special case of a change assertion $\changetuple{t-1,t}{sv}{any}{v}$ where $any$ is an unconstrained variable that can take any value.

Assertions can involve object and temporal variables.
A planner has to find activities to be performed and constraints on the variables such that the expected evolution of the system is both feasible and achieves the desired goals.

Given a temporally qualified assertion $\alpha = \persistencetuple{t_s,t_e}{sv}{v}$ or $\alpha =  \changetuple{t_s,t_e}{sv}{v_1}{v_2}$, we denote $t_s$ and $t_e$ as $start(\alpha)$ and $end(\alpha)$ respectively.

\paragraph{Timelines.}
A timeline is a pair $(\assertions,\constraints)$ where \assertions is a set of temporal assertions on a state variable and \constraints is a set of constraints on object and temporal variables appearing in \assertions.
A timeline gives a partial view of the evolution of a state variable over time.
For instance, the following timeline describes the whereabouts of a truck at different points in time.

\begin{example} \label{ex:timeline}
  The following is a timeline containing three temporally qualified assertions on the state variable $\LocROne$.
  \begin{align*}
    \langle & \set{ \change{t_1,t_2}{\LocROne}{\constant{dock1}}{d},~\persistence{t_2,t_3}{\LocROne}{d},~\change{t_4,t_5}{\LocROne}{d}{\constant{dock4}} } \\
            & \set{t_1 < t_2 < t_3 < t_4 <t_5,~\function{connected}(\constant{dock1},d),~\function{connected}(d,\constant{dock4}) }  \rangle\\
  \end{align*}%
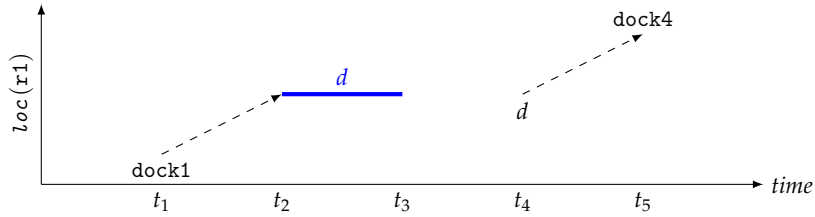
  \begin{center}
    \begin{tikzpicture}[x=2cm,y=.5cm,scale=0.8, every node/.style=transform shape
      ]
      \coordinate (y) at (0,6);
      \coordinate (x) at (6,0);
      \foreach \i in {1,2,3,4,5} {
        \node[anchor=north] at (\i,0) {$t_\i$};
      }
      \draw[->] (0,0) -- node[sloped,midway,above] {$\LocROne$} (y);
      \draw[->] (0,0) --  (x) node[right] {\emph{time}};
      
      \draw[timeline/change] (1,1) node[below] {\constant{dock1}} -- (2,3);
      \draw[timeline/persistence] (2,3) -- node[above,sloped,midway] {$d$} (3,3);
      \draw[timeline/change] (4,3) node[below] {$d$}--  (5,5) node[above] {\constant{dock4}};
    \end{tikzpicture}
  \end{center}

  The truck \constant{r1} is at \constant{dock1} at time $t_1$.
  Over the time span $[t_1,t_2]$, it moves to a yet undetermined dock $d$.
  Its location over $]t_1,t_2[$ is unspecified.
  The truck is then constrained to stay at $d$ until time $t_3$.
  Its whereabouts are not constrained over the period $]t_3,t_4[$.
  However it must be at the same dock $d$ at time $t_4$ from where it will move to \constant{dock4}.
  The constraints impose a total order on all temporal events and restrict the possible instantiations of $d$ to places connected to both \constant{dock1} and \constant{dock4}.\EndOfTheoremMark

\label{ex-timeline}
\end{example}

A timeline may have uninstantiated temporal and object variables: it represents a set of possible evolutions of a state variable.
In \autoref{ex:timeline}, the state variable \LocROne  ~will go through different values depending on the value assigned to $d$. Assertions related to an instantiated state variable, such as $loc(r)$ for $r \in trucks$, refer to a set of timelines. 

A timeline $(\assertions,\constraints_1)$ is an instantiation of a timeline $(\assertions,\constraints_2)$ if \i $\constraints_2 \subseteq \constraints_1$ and \ii all variables in $\assertions$ and $\constraints_1$ are instantiated.
While a state variable can only have a single value, it is possible for two assertions to require conflicting values for the state variable.
For instance, two persistences \persistencetuple{t_1,t_2}{sv}{v} and \persistencetuple{t_3,t_4}{sv}{v'} are conflicting if ($[t_1,t_2] \cap [t_3,t_4] \neq \emptyset) \AND (v \neq v'$).
These constraints need to be taken into account in a timeline to unsure its consistency.

\begin{definition}[Possible Consistency of a Timeline]
  A timeline $(\assertions,\constraints)$ is \emph{possibly consistent} if it has an instantiation $(\assertions,\constraints')$ such that $\constraints'$ is a consistent set of constraints and the state variable has no conflicting values in this instance.
\end{definition}

\begin{definition}[Necessary Consistency of a Timeline]
  A timeline $(\assertions,\constraints)$ is \emph{necessarily consistent} if all its instantiations are consistent timelines.
\end{definition}

A timeline that is not possibly consistent cannot describe a valid evolution of its state variable.
A necessarily consistent timeline does not contain any conflicting assertions regardless of the choices made when binding its variables: it must have a set of constraints such that \i no two change assertions can overlap, \ii no persistence can overlap a change assertion, \iii any two overlapping persistences must be of the same value (by definition of a timeline, they are on the same state variable).
The timeline of Example~\ref{ex-timeline} is necessarily consistent because the temporal constraints impede assertions from overlapping.

\begin{definition}[Causal Support]
\label{def:causal-support}
An assertion $\alpha \in \assertions$ is causally supported in a  timeline $(\assertions,\constraints)$, for $\alpha = \persistencetuple{t_1,t_2}{sv}{v}$ or $\alpha = \changetuple{t_1,t_2}{sv}{v}{v'}$ iff:
\begin{LIST}
\item there is an assertion $\beta \in \assertions$ that produces $\fluent{sv}{v}$ at a time $t_0 \leq t_1$; $\beta$, called the causal supporter of $\alpha$, is either
 an assignment \assignmenttuple{t_0}{sv}{v} or if it is a change  \changetuple{t,t_0}{sv}{v''}{v}; and
\item $sv$ keeps the value $v$ from the end of $\beta$ until the start of $\alpha$, i.e., $\fluent{sv}{v}$ holds in $[t_0,t_1]$.
\end{LIST}
Any assignment assertion $\assignmenttuple{t}{sv}{v}$ is \emph{a priori} causally supported.\EndOfTheoremMark
\end{definition}

In \autoref{ex-timeline}, the persistence assertion, \persistence{t_2,t_3}{\LocROne}{d}  is causally supported by the first change assertion, 
\change{t_1,t_2}{\LocROne}{\constant{dock1}}{d}, since \emph{(i)} their temporal intervals meet, and \emph{(ii)} the fluent \fluent{\LocROne}{d} produced is the one required.
However none of the two change assertions are causally supported.
This might seem surprising for the second change assertion because the fluent \fluent{\LocROne}{d} is produced by the first assertion.
This fluent is however not constrained to hold during $]t_3,t_4[$ and, as a consequence, \LocROne might still be modified during that interval.
This second change assertion could be made causally supported by the addition of a persistence condition ending at $t_4$, e.g. \persistence{t_3,t_4}{sv}{d}.
Another way of supporting it would be the introduction of a new change assertion $\changetuple{t,t_4}{sv}{d'}{d}$, such an assertion would itself need to be causally supported.

To have all assertions of a timeline causally supported, the earliest one must be an \emph{assignment}, which is by definition \emph{a priori} supported.

\subsection{Tasks and Action Templates}
\label{sec:actions}

An action has a unique name and a set of parameters.
The pre-conditions and the effects of an action are encoded by a set of temporally qualified assertions.
A set of constraints restrict the allowed values taken by the  parameters and temporal variables in the action.
We augment this usual view of an action with hierarchical features: \i we define a set of parametrized task symbols \tasks, \ii each action is associated with a task symbol in \tasks representing the task this action achieves, \iii an action can have a set of subtasks representing tasks that must be achieved for this action to provide its desired effects.
Furthermore, an action can be marked as \emph{task dependent} in which case it cannot be used freely by the planner but only as a subtask of some other action.

An action template $A$ is a tuple $(head(A), task(A), \taskdependent, subtasks(A), \assertions_ {A}, \constraints_{A})$ where: 
\begin{LIST}
\item $head(A)$ is the name and the list of typed parameters of $A$.
  The parameters refer to object variables. Temporal variables are kept implicit in the parameter list; $\tstart$ and $\tend$  are respectively the start and end timepoints of $A$. An action template may also use  a \textit{duration variable} which is not a binding argument  nor a free variable, but a mean to couple temporal and object constraint networks (see \autoref{sec:dur-var}).
\item $task(A) \in \tasks$ is the task it achieves with parameters from those of A.
\item $\taskdependent \in \set{\true,\false}$ is true if the action is \emph{task dependent}, otherwise the action is said to be \emph{free}. 
The  intuition (formalized in the next section) is that a \emph{task dependent} action can only be inserted in a plan if it achieves a task whose achievement was required either in the problem definition or through a subtask of another action.%
  \footnote{
    The current draft of the ANML manual proposes a notion similar to our task-dependency, by the introduction of the keyword \emph{motivated} \cite{AnmlManual}.
    When placed in an action $A$, an instance $a$ of $A$ can only appear in the plan if there is an action instance $b$ that has a subtask achieved by $a$ and such that $[start(a),end(a)] \subseteq [start(b),end(b)]$.
    Conceptually, the presence of such an action must be ``motivated'' by the presence of a higher level action that requires its presence and temporally envelops it.
    Our \emph{task-dependency} setting differs as it does not require the ``motivating'' task to be part of an action nor the subaction to be temporally contained by the said action.
    This simple difference allows us to motivate the presence of actions from tasks placed in the problem definition.
    This capability is key in relating our model to HTN planning in which all actions are derived from the initial task network.
  }
\item $subtasks(A)$ is a set of temporally qualified subtasks.
  For any action instance in a plan, a subtask $\timedfacttuple{t_s,t_e}{\tau}$ states that the plan should also contain an action instance $a$ that \i starts at $t_s$, \ii ends at $t_e$, and \iii has $task(a) = \tau$.
  The timepoints $t_s$ and $t_e$ need not be grounded; they are related to other timepoints through temporal constraints, possibly inducing a partial order on the subtasks.
\item $(\assertions_ {A},\constraints_{A})$ is a set of  possibly consistent timelines.
 $\assertions_{A}$ provides the conditions and effects of $A$ on some state variables through persistence and change assertions.
 $\constraints_{A}$ imposes constraints on the possible instantiations of the temporal and object variables of $A$.
\end{LIST}

In hierarchical planning, one often distinguishes between high-level actions (or abstract tasks) and primitive actions (or primitive tasks).
A \emph{high-level action} provides a \textit{method} to achieve a given task, expressed as a set of conditions and subtasks.  A \emph{primitive action} is intended to be executed with a command whose effects are represented as assertions on state variables.
While our model does not make a distinction between the two, we will use the same terminology when needed to convey the role of an action in a planning model.

\autoref{fig:move-action} illustrates a primitive action template for moving a robot $r$ from dock $d$ to $d'$.
It requires $r$ to be in $d$ at $\tstart$, and $d'$ to be empty at some point $t'$ before $\tend$.
Its effects are to make the location of the truck be $d'$ at $\tend$, and to have $d$ empty at some point $t$ after $\tstart$ and before $t'$.
The action has no subtask and is the only achiever of the eponymous task $\tsk{move}(r, d, d')$.
\autoref{fig:transport-method} exemplifies  high-level actions offering two methods for achieving the \tsk{transport} task of moving a container $c$ to a location $d$.

\begin{figure}[htbp]
 \centering\fbox{
 \begin{pcode}
    \phead{$\act{move}(r, d, d')$}
    \pkey{task} $\tsk{move}(r,d,d')$ 
    \pkey{dependent} no
    \pkey{subtasks} $\emptyset$ 
    \pkey{assertions}  $\change{\tstart, \tend}{\sv{loc}(r)}{d}{d'}$ 
    \1 $\change{\tstart,t}{\sv{occupant}(d)}{r}{\nil}$ 
    \1 $\change{t', \tend}{\sv{occupant}(d')}{\nil}{r}$ 
    \pkey{constraints} $\function{connected}(d,d') = \true$  
    \1 $\tend - \tstart = \function{travel{-}time}(d,d')$
    \1   $t<t'$ 
  \end{pcode}}
  \caption{
    An action to move a truck $r$ from dock $d$ to dock $d'$. The duration of the action $(\tend - \tstart)$ is set to be equal to the travel time from $d$ to $d'$.
     $\function{connected}(d,d') = \true$ is a constraint on objects variables and requires that the two docks be connected, a temporally invariant property of the planning domain.
    Left implicit for readability are the types of variables and the constraints $\tstart < \tend, \tstart<t$ and $t'< \tend$.
  }
  \label{fig:move-action}
\end{figure}

\begin{figure}[htbp]
  \centering
  \begin{multicols}{2}
   \fbox {\begin{pcode}
      \phead{\act{m1-transport}$(c,d)$} 
      \pkey{task} $\tsk{transport}(c,d)$ 
      \pkey{\taskdependent} yes
      \pkey{assertions}  $\persistence{\tstart,\tend}{\sv{pos}(c)}{d}$
      \pkey{subtasks} $\emptyset$
      \pkey{constraints}  $\emptyset$
    \end{pcode}}   
    
    \fbox{\begin{pcode}
      \phead{\act{m2-transport}$(r,c,d_s,d)$} 
      \pkey{task} $\tsk{transport}(c,d)$ 
      \pkey{\taskdependent} yes
      \pkey{assertions} $[\tstart] \sv{loc}(r) = d_s$ 
      \1 $[\tstart] \sv{pos}(c) = d_s$ 
      \pkey{subtasks} $\timedfact{\tstart,t_1}{\tsk{load}(r,c,d_s)}$ 
      \1 $\timedfact{t_2,t_3}{\tsk{move}(r,d_s,d)}$ 
      \1  $\timedfact{t_4, \tend}{\tsk{unload}(r,c,d)}$ 
      \pkey{constraints}  $\sv{connected}(d_s, d)$ 
      \1 $d_s \neq d$
      \1 $t_1 < t_2 < t_3 < t_4$
    \end{pcode}}
  \end{multicols}
  \caption{
    High-level actions for achieving the task of transporting a container $c$ to a location $d$.
    The first one requires nothing to be done if $c$ is already at its destination $d$.
    The second one states that transporting $c$ from $d_s$ to $d$ can be achieved by a sequence of \act{load}, \act{move} and \act{unload} subtasks, using a truck $r$.}
  \label{fig:transport-method}
\end{figure}

\subsection{Chronicles}\label{actions-as-chronicles}

A \emph{chronicle} is a triple $(\plan, \assertions, \constraints)$ where \plan is a partial plan composed of action instances and unrefined tasks  while (\assertions, \constraints) is a
set of timelines. A chronicle is a temporal and hierarchical extension of partial plans,
as in the plan-space planning approaches.
It extends the existing notion of chronicle \cite[Sec. 4.2.4]{Ghallab2016} with a new member $\plan$ that keep tracks of tasks and actions in the plan.

A planning domain is a tuple $\Sigma = (\objects,\statevariables,\tasks,\actions)$ specifying a set of typed domain objects \objects, a set of state variables \statevariables, a set of task symbols \tasks, and of action templates \actions. A planning problem is a pair $\tuple{\Sigma,\chronicle_0}$ where 
 the chronicle $\chronicle_0 = (\plan_0, \assertions_0, \constraints_0)$ defines:
\begin{LIST}
\item $\plan_0$ a set of temporally qualified tasks that must be achieved.
\item $\assertions_0$ a set of temporally qualified assertions that describe the initial state of the environment and its expected evolution together with a set of goals.
  The initial state and its evolution is typically depicted by \emph{assignment assertions} (that do not require causal support).
  Goals are typically represented by \emph{persistence assertions} whose causal support will require additional activity to be considered by the planner.
\item $\constraints_0$ is a set of constraints restricting the allowed values for the temporal and object variables in $\plan_0$ and $\assertions_0$.
They express temporally extended goals, timed initial literals and ordering constraints on the tasks to be achieved.
\end{LIST}

The chronicle below represents a planning problem with two trucks \constant{r1} and \constant{r2}, initially at \constant{dock1} and \constant{dock2} respectively, and a ship that is expected to be docked at \constant{pier1} at a future interval of time.
The problem is to perform a \tsk{transport} task of moving container \constant{c1} that is initially on \constant{ship1} to \constant{dock3} and to have the two trucks in their initial locations  at the end.
Note that $\phi_0$ states the planning objectives \i as tasks to perform, as in HTN, and \ii as goals to satisfy, as in generative planning.

\begin{figure}[htbp!]
  \centering\fbox{
  \begin{pcode}
    \phead{$\phi_0$} 
    \pkey{tasks} $\timedfact{t,t'}{\tsk{transport}(\constant{c1}, \constant{dock3})}$
    \pkey{assertions} $\assignment{\tstart}{\sv{loc}(\constant{r1})}{\constant{dock1}}$ 
    \1 $\assignment{\tstart}{\sv{loc}(\constant{r2})}{\constant{dock2}}$ 
    \1 $\assignment{\tstart}{\sv{on}(\constant{c1})}{\constant{ship1}}$
    \1 $\assignment{\tstart+10}{\sv{docked}(\constant{ship1})}{\constant{pier1}}$
    \1 $\assignment{\tstart+\delta}{\sv{docked}(\constant{ship1})}{\nil}$
    \1 $\persistence{\tend}{\sv{loc}(\constant{r1})}{\constant{dock1}}$
    \1 $\persistence{\tend}{\sv{loc}(\constant{r2})}{\constant{dock2}}$
    \pkey{constraints}   $\tstart<t < t'< \tend \wedge 20 \leq \delta\leq30 $
  \end{pcode}}
  \caption{The initial chronicle of a planning problem.}
  \label{fig:initial-chronicle}
\end{figure}

Planning with this representation is a sequence of transformations of a chronicle, starting from $\phi_0$, until a solution plan is found.
We now describe the possible transformations and the conditions a solution plan must meet.

\subsection{Plan: Transformations and Solutions} 
\label{sec:transformation-solutions}

HTN planners build plans by systematically decomposing tasks, while generative planers synthesizes them by iteratively introducing new action instances. In FAPE, both processes are combined, interleaved with the insertion of temporal and binding constraints.

\subsubsection{Task Refinement}
\label{sec:task-decomposition}

A chronicle $\chronicle=(\plan_\chronicle, \assertions_\chronicle, \constraints_\chronicle)$ can be refined into a chronicle $\phi' = (\plan_{\phi'}, \assertions_{\phi'}, \constraints_{\phi'})$ by decomposing an unrefined task $\tau \in \plan_\phi$ with a new action instance $a$ such that $task(a) = \tau$.
This transformation is denoted by $\phi \decomposition{\tau, a} \phi'$ and results in the following $\chronicle'$:
\begin{align*}
  \plan_{\phi'} & \gets \plan_\phi \cup  \set{a} \cup subtasks(a) \setminus \{\tau \} \\
  \assertions_{\phi'} &\gets \assertions_\phi \cup assertions(a) \\
  \constraints_{\phi'} &\gets \constraints_\phi \cup constraints(a) \cup \set{ task(a) = \tau}
\end{align*}

The refining action $a$ may introduce additional tasks and assertions representing both the conditions and the effects of this action.
In addition to the constraints from the action template, the constraint $task(a)=\tau$ unifies all parameters of $task(a)$ and $\tau$ and enforces $a$ to start and end at the times specified by the temporal qualification of this task.

\subsubsection{Action Insertion}
\label{sec:action-insertion}

A chronicle \fullchronicle can be refined into a chronicle $\phi' = (\plan_{\phi'}, \assertions_{\phi'}, \constraints_{\phi'})$ by the insertion of a \emph{free} action instance $a$, i.e., $a$ is not task-dependent.
This transformation is denoted by $\chronicle \insertion{a} \chronicle'$ and results in the following $\chronicle'$:
\begin{align*}
  \plan_{\phi'} & \gets \plan_\phi \cup  \set{a} \cup subtasks(a) \\
  \assertions_{\phi'} &\gets \assertions_\phi \cup assertions(a) \\
  \constraints_{\phi'} &\gets \constraints_\phi \cup constraints(a)
\end{align*}

Note that task dependent actions can only be used through task decomposition and must respect all constraints specified on that task.

\subsubsection{Plan Restriction Insertion}
\label{sec:restriction-insertion}

A chronicle \fullchronicle can be refined into a chronicle $\phi' = (\plan_{\phi'}, \assertions_{\phi'}, \constraints_{\phi'})$ by the insertion of restrictions $(\assertions,\constraints)$, where \assertions is a set of persistence assertions and \constraints is a set of constraints over temporal and object variables.
This transformation is denoted by $\chronicle \restriction{(\assertions,\constraints)} \chronicle'$ and results in the following $\chronicle'$:
\begin{align*}
  \plan_{\phi'} & \gets \plan_\phi \\
  \assertions_{\phi'} &\gets \assertions_\phi \cup \assertions \\
  \constraints_{\phi'} &\gets \constraints_\phi \cup \constraints
\end{align*}

This transformation restricts a chronicle by either \i  adding persistence conditions to achieve the causal support of an assertion, (as for causal links in plan-space planning), or \ii   restricting allowed values of some variables to remove potential inconsistencies in the chronicle, for instance by imposing an order on two actions with conflicting requirements, or to unify task and actions instances that can be unified.

\subsubsection{Reachable and Solution Plans}

A chronicle $\phi'$ is \emph{reachable} from a chronicle $\phi$ if there is a sequence of task refinements, action insertions, and plan restrictions that transformes \chronicle into $\chronicle'$.
The planner has to find a sequence of transformations such that the resulting chronicle corresponds to a feasible plan.

\begin{definition}[Solution plan]\label{def:solution-plan}
  A chronicle $\chronicle^* = (\plan^*,\assertions^*,\constraints^*)$ is a solution to a planning problem $(\Sigma,\chronicle_0)$ if:
  \begin{LIST}
  \item $\chronicle^*$ is reachable from $\chronicle_0$,
  \item $\plan^*$ has no unrefined tasks,
  \item all assertions in $\assertions^*$ are causally supported.
  \item $(\assertions^*,\constraints^*)$ is a necessarily consistent set of timelines.\EndOfTheoremMark
  \end{LIST}
\end{definition}

The solution plan is given by the actions in $\plan^*$ with the constraints from $\constraints^*$.
Since $\chronicle^*$ must be reachable from $\chronicle_0$, the actions in $\chronicle^*$ must fulfill hierarchical constraints: if an action in $\plan^*$ is \emph{task-dependent} then it was inserted through task refinement and must respect all constraints placed on the task it refines.
Furthermore, all assertions in $\chronicle^*$ are causally supported. Finally, all instantiations of variables in $\plan^*$ are consistent. Uninstantiated variables allow to delay choices  until execution time. 
This is typically the case for temporal variables whose instantiation is often decided dynamically at execution time.

\subsection{Discussion}
\label{sec:discussion}

The proposed representation builds on the rich temporal semantics of ANML.
It supports \i temporal actions with assertions placed at arbitrary timepoints, \ii the description of the current state of the environment as well as its expected evolution over time and \iii temporally extended goals.
It also supports a unique mix of generative and hierarchical planning.
The notion of task-dependency allows for a seamless integration of generative and hierarchical models especially allowing for partial hierarchies and the capability of allowing task insertion.

This representation extends the expressivity of existing planners allowing to encode generative planning problems (i.e. with no tasks nor task-dependent actions) as well as HTN problems (i.e. where all actions are task-dependent). 
In the following sections, we introduce a planning algorithm supporting this representation as well as several dedicated search control methods.

\section{A Plan-Space Planning Procedure}
\label{sec:planning-procedure}

\subsection{Overview}

FAPE extends the chronicle planning approach \cite[Sec. 4.3]{Ghallab2016}, which was initially proposed in IxTeT \cite{Ghallab1994}.
Our contribution is a planning procedure that seamlessly supports both hierarchical and generative planning.

Given a planning problem with an initial chronicle $\phi_0$, the planner attempts to transform $\phi_0$ into a solution chronicle $\phi^*$ where all tasks have been refined, all assertions are causally supported  and the set of temporal assertions defines necessarily consistent timelines.
For that, the planner detects \emph{flaws} in a chronicle: features that prevent the chronicle from being a solution.
A \emph{flaw} in a chronicle $(\plan,\assertions,\constraints)$ is either:
\begin{LIST}
\item an \emph{unrefined task} $\tau \in \plan$,
\item an \emph{unsupported assertion} $\alpha \in \assertions$,  or
\item \textit{conflicting assertions}: a pair of assertions $(\alpha,\beta) \in \assertions\times\assertions$  that can be conflicting given $\constraints$. 
\end{LIST}

Each type of flaw matches one of the necessary conditions for a chronicle to be a solution plan (\autoref{def:solution-plan}); the last two types generalize respectively the notions of open-goals and threats in plan-space planning.
Every flaw must be resolved to transform a chronicle into a solution plan. Resolving a flaw requires the application of one or several plan transformations (\autoref{sec:transformation-solutions}).

\autoref{alg:plan} is a nondeterministic abstract view of the planning procedure.
For a given chronicle, the planner chooses a flaw $\varphi$ to solve.
Since all flaws must eventually be solved, this choice is not a backtracking point.
Next, the planner nonderterministically chooses of a resolver $\rho$ to handle the flaw; it may backtrack on this choice.
The resolver results in a transformation of  $\phi$ into a refined chronicle $\phi'$ in which the flaw $\varphi$ is absent, and from which the search proceeds recursively.
The algorithm returns when it encounters a solution plan (i.e. a flaw-less chronicle) or a deadend (i.e. a flaw with no resolvers).

\begin{algorithm}[!htb]
\begin{algorithmic}
\Function{\FapePlan}{$\Sigma,\phi$}
    \State $Flaws \gets$ flaws in $\phi$
    \If{$Flaws = \emptyset$} \Return $\phi$ \EndIf
    \State $\varphi \gets$ select a flaw in $Flaws$
    \State $Resolvers \gets$ resolvers for $\varphi$
    \If{$Resolvers = \emptyset$} \Return failure \EndIf
    \State nondeterministically choose $\rho \in Resolvers$
    \State $\phi' \gets$ \Call{Transform}{$\phi, \rho$}
    \State \Return \Call{FapePlan}{$\Sigma,\phi'$}
  \EndFunction
\end{algorithmic}
\caption{\FapePlan algorithm: returns a solution
  plan achieving the tasks and goals in $\phi$ for the domain $\Sigma$, or failure if none exists.}
\label{alg:plan}
\end{algorithm}

This procedure follows the general refinement planning framework \cite{kambhampati1995,Schattenberg2009}, which has been used by partial order planners  \cite{Penberthy1992,Younes2003}, generative temporal planners \cite{Ghallab1994,Vidal2006},  and hierarchical planners, e.g.,  HiPOP \cite{Bechon2014} and PANDA \cite{Schattenberg2009}.
The differences between all those planners lie in four aspects: \i their definition of flaws and resolvers, \ii the type of constraints and propagation used to reason about variables in the plan, \iii their internal representation of a chronicle, and \iv the strategy for exploring the search space

The remainder of this section details successively these first three aspects, while the search guidance will be addressed in \autoref{sec:search-control}.

\subsection{Flaws and Resolvers}
\label{flaws-resolvers}

The definition of flaws is critical for the soundness of the planning procedure. Their  resolvers condition the completeness. Ill-defined resolvers may also result in redundancies in the search space. Let consider the three types of flaws and their resolvers.

\subsubsection{Unsupported Assertions}
\label{sec:unsupported-assertions}

The first type of flaw is a lack of causal support for some assertion in the current chronicle.
The planner incrementally tracks this property by associating every assertion that requires causal support to an assertion that produces the desired fluent.
The causal link relation $\beta \rightarrow \alpha$ denotes that the assertion $\beta$ is the causal supporter of $\alpha$.
A causal link is created by inserting an additional persistence assertion \persistence{end(\beta),start(\alpha)}{sv}{v}, which prevents any change on the value of the state variable $sv$ from the end of $\beta$ until the start of $\alpha$.

With the exception of \emph{a priori} supported assertions, the planner assumes all assertions to be unsupported until a causal link is added to the chronicle.
  Any change or persistence assertion that does not have an incoming causal link constitutes an \emph{unsupported assertion flaw}.
Such a flaw in a chronicle $\chronicle$ is solved by finding another assertion $\beta$ that can serve as a support: either $\beta$ is already in $\assertions_\chronicle$, or $\beta$ can be added through the insertion of an action.

Action insertion must take into account the hierarchical features of the domain.
Instead of directly inserting an action that provides causal support, we select or create a refinement tree in which the supporting assertion will be chosen.

\paragraph{Possible effects.}
\newcommand{\optimisticeffects}[1]{\ensuremath{E_{#1}^+}\xspace}

We say that the fluent $f = \tuple{sv=v}$ is a \emph{direct effect} of an action $a$ if  $a$ has an assertion of the form $\changetuple{t,t'}{sv}{v'}{v}$ or  \assignmenttuple{t''}{sv}{v}. 

Possible effects represent fluents that can be produced as part of the refinement tree of a task or an  action. A fluent $f$ is a \emph{possible effect} of an action $a$, denoted $f \in \optimisticeffects{a}$, if it is either a direct effect of $a$ or a possible effect of the refinement of one of its subtasks. A fluent $f$ is a \emph{possible effect} of a task $\tau$, denoted $f \in \optimisticeffects{\tau}$, if $f$ is a possible effect of an action $a$ that refines $\tau$.

\begin{align*}
\optimisticeffects a &= \textit{direct\_effects}(a)  \Union_{\tau \in \textit{subtasks}(a)} ~~\optimisticeffects{\tau}\\
  \optimisticeffects{\tau} &=  \bigcup_{m~\in~\{ a~|~\textit{task}(a) = \tau \}} \optimisticeffects{m}                           
\end{align*}

\def \delayposeff {\ensuremath{\Delta_{PosEff}}\xspace}

The possible effect $f$ of an action $a$ (resp. a task $\tau$) is also associated with a duration representing the minimal delay from the moment the action (resp. task) starts to the instant at which the possible effect can be produced, noted $\delayposeff(a,f)$.
This minimal delay is computed by taking an optimistic view of the delays enforced as temporal constraints in the actions, as illustrated in the following example.

\begin{example}
  The \act{move}(r,d,d') action of \autoref{fig:move-action} has a direct effect $\fluent{\sv{loc}(r)}{d'}$. %
  It is produced at time $\tend$, thus its associated delay is the minimal duration of the action: $\tend-\tstart$.  

  \[\delayposeff(\act{move}(r,d,d'),\fluent{\sv{loc}(r)}{d'}) = \text{min-delay}(\tstart,\tend) = \function{travel-time}(d,d')\]

  The fluent $\fluent{loc(r)}{d'}$ is also a possible effect of the action \act{m2-transport} (\autoref{fig:transport-method}) because the presence of its \tsk{move} subtask means it can be  produced as a result of inserting \act{m2-transport} in the plan.
  The delay associated with the possible effect is defined recursively as:
  \begin{align*} 
    \delayposeff(\act{m2-transport}&(r,c,d,d'),\fluent{\sv{loc}(r)}{d'}) \\ &= \text{min-delay}(\tstart,t_2) + \delayposeff(\act{move}(r,d,d'),\fluent{\sv{loc}(r)}{d'}) \\
                                   &= \text{min-delay}(\tstart,t_2) + \function{travel-time}(d,d') \\
                                   &= \text{min-delay}(\tstart,t_1)+\text{min-delay}(t_1,t_2) + \function{travel-time}(d,d') \\
                                   &= duration(\act{load}(r,d,c)) +1 + \function{travel-time}(d,d')
  \end{align*}
  The computation of minimal delays is done by a shortest path computation in an STN containing all timepoints and temporal constraints of the action.\EndOfTheoremMark
\end{example}

A supporting assertion $\beta$ for an unsupported assertion $\alpha$ is either:
\begin{LIST}
\item  already in the plan,
\item  introduced by decomposing an unrefined task $\tau \in \plan$, such that $\tau$ has a possible effect supporting $\alpha$, denoted $cond(\alpha) \in \optimisticeffects{\tau}$), or
\item introduced through the addition of a \emph{free} action $a$, such that $a$ has a possible effect supporting $\alpha$, i.e., $cond(\alpha) \in \optimisticeffects{a}$.
\end{LIST}

It is impractical for the planner to branch over all the assertions that can be used for causal support.
Instead, the planner branches on the choice of a source for providing the causal supporter, e.g, for a given unsupported assertion, the planner commits to select its causal supporter in the set of assertions resulting from the future refinement of a particular unrefined task $\tau \in \plan_\chronicle$.

\paragraph{Supporting tasks commitments.}
When doing hierarchical planning, it might happen that several decomposition steps occur before a task is fully decomposed into primitive actions.
On the contrary in plan-space planning, the resolution of an unsupported condition typically requires the immediate insertion of a supporting action in the plan.
To reconcile the two approaches, we now present a mechanism that restricts the possible supporters of a condition to the actions descending from a particular task. 
Effectively this allows expressing search commitments of the form "a condition $c$ must be supported by the yet unrefined task $\tau$" and its corollary "$\tau$ must decomposed in a way that supports $c$".
This construct is to be exploited by the resolvers that we present immediately after.

\newcommand{\descendanttasks}[1]{\ensuremath{DT_{#1}}}
We associate an unsupported assertion $\alpha$ to a set of tasks \descendanttasks{\alpha} representing a commitment on the origin of the supporter for $\alpha$.
More specifically, for any task $\tau \in \descendanttasks{\alpha}$, 
the causal supporter of $\alpha$ can only be chosen in the possible effects of an action that is a descendant of $\tau$.
An action is a descendant of $\tau$ if it refines $\tau$ or if it refines a subtask of an action descending from $\tau$.

This mechanism allows tying the resolution of an unsupported assertion to a tree or subtree of the task network.
$\descendanttasks{\alpha}$ is initially empty and might be extended during search to track commitments made on the origin of the support of $\alpha$.

If there is an unrefined task $\tau \in \descendanttasks{\alpha}$, the resolution of $\alpha$ is postponed until $\tau$ is refined.

\paragraph{Resolvers.}

Let us now explore the three different resolvers for an unsupported assertion $\alpha = \persistencetuple{t_s,t_e}{sv}{v}$ (resp. an unsupported change assertion $\alpha = \changetuple{t_s,t_e}{sv}{v}{v'}$):
\begin{LIST}
\item \textbf{Direct Supporter.}
Let $\eventsupporters{\alpha}$ be the set of all assertions $\beta \in \assertions_\chronicle$ such that $\beta$ produces $\tuple{sv=v}$.
For any change assertion $\beta \in \eventsupporters{\alpha}$, a possible resolver for $\alpha$ is  a causal link $\beta \rightarrow \alpha$.

The transformation resolving the unsupported assertion $\alpha$ is a plan restriction that introduces a causal link from $\beta$ to $\alpha$: $\chronicle \restriction{(\set{\beta \rightarrow \alpha},C)} \chronicle'$.
The additional set of constraints $C$ simply enforces that the value produced by $\beta$ is the one needed by $\alpha$.
Furthermore, the assertion representing the causal link is \emph{a priori} supported since it is inseparable from $\beta$ that causally supports it.

For instance, if the assertion $\changetuple{t_1,t_2}{loc(\constant{r1})}{\constant{d1}}{d}$ was selected to be the causal support of an assertion $\persistencetuple{t_3,t_4}{loc(r)}{\constant{d4}}$, we would have an additional set of constraints \mbox{$t_2 \leq t_3 \wedge \constant{r1}=r \wedge d=\constant{d4}$}.
The causal link would take the form of the persistence assertion $\persistence{t_2,t_3}{\LocROne}{\constant{d4}}$.

To avoid introducing redundancies in the search space, we must also take into account any previous commitment that has been made regarding the source of the causal support of $\alpha$.
For this reason, an assertion $\beta$ is only considered as possible supporter if the action that added $\beta$ is a descendant of all tasks $\tau \in \descendanttasks{\alpha}$.
Put otherwise, the supporter must originate from the refinement trees that were previously selected.

\item \textbf{Delayed support from existing task.}
Another possible source of supporting assertions is in the unrefined tasks of $\phi$.
Let \tasksupporters{\alpha} be a subset of the unrefined tasks $\tau \in \plan_\chronicle$ such that: 

\begin{LIST}
\item $\tuple{sv=v}$ is a possible effect of $\tau$, i.e., $\fluent{sv}{v} \in \optimisticeffects{\tau}$, and
\item there is enough time for the possible effect to occur before $\alpha$, meaning that the addition of the temporal constraint $start(\alpha) - start(\tau) \geq \delayposeff(\tau,\fluent{sv}{v})$ does not make the temporal network inconsistent.
\end{LIST}

For any unrefined task $\tau \in \tasksupporters{\alpha}$, a possible resolver for $\alpha$ is to add $\tau$ to \descendanttasks{\alpha}. While this does not directly resolve the flaw, it makes a commitment to a subset of resolvers in the possible effects of $\tau$. The choice of a supporting assertion for $\alpha$ will be delayed until $\tau$ is refined. The chosen refinement for $\tau$ will need to provide an enabler for $\alpha$.
It should be noted that this resolver does not bring any change to the chronicle itself but constrains the planner to a restricted set of solution plans.

Once again, to avoid redundancies in the search space, we must account for the previous commitments made on the causal supporter of $\alpha$.
For this reason, we only consider as a resolver a task $\tau \in \tasksupporters{\alpha}$ if, for any task $\tau' \in \descendanttasks{\alpha}$, $\tau$ is a descendant of $\tau'$.

\item \textbf{Delayed support from new action.}
The last possible source of supporting assertions comes from the introduction of free actions that have $\fluent{sv}{v}$ as a possible effect.

If the planner already made a commitment to support $\alpha$ from a particular task (i.e., $\descendanttasks{\alpha} \neq \emptyset$), then no such resolver is applicable. 
Indeed, any supporter must appear by means of task refinement only.
When $\descendanttasks{\alpha}=\emptyset$ we need to consider the insertion of free actions outside of any existing decomposition trees.
Let \actionsupporters{\alpha} be a set of action templates $A$ such that $\fluent{sv}{v} \in \optimisticeffects{A}$.
For any action  $A \in \actionsupporters{\alpha}$, a possible resolver for $\alpha$ is brought by inserting an instance $a$ of  $A$ in $\phi$ and adding $task(a)$ to \descendanttasks{\alpha}.
This resolver ensures that $\alpha$ will be supported either by a direct effect of $a$ or by an effect of one of its subactions. The choice of which assertion will support $\alpha$ is delayed.
\end{LIST}

\medskip
To summarize, every assertion needs a causal supporter with and explicit causal link.
We distinguish three cases for the selection of the causal supporter depending on whether it is already in the plan, can be introduced by refining pending tasks in the plan, or requires the introduction of new actions outside of existing decomposition trees.
In the two latter choices, the selection of the actual causal supporter is delayed until the task network is refined.

\subsubsection{Unrefined Tasks}
\label{sec:unrefined-tasks}

A task is \emph{unrefined} if \textit{(i)} it has been required either in the problem specification or as a subtask of an action in the plan, and \textit{(ii)} it has not been refined by a \emph{task refinement} transformation (\autoref{sec:task-decomposition}).
Given a chronicle $\chronicle = \tuple{\plan_\chronicle,\assertions_\chronicle,\constraints_\chronicle}$, any task $\tau$ appearing in $\plan_\chronicle$ is an unrefined task.
Resolving an unrefined task requires selecting an action that achieves it and applying the corresponding task decomposition transformation. Our procedure extends the usual HTN task decomposition to account for the commitments previously made while resolving unsupported assertion flaws.

Let $\tau: \tsk{tsk}(x_1,\dots,x_n)$ be an unrefined task, and %
$\possiblerefinements{\tau}$ be the set of action templates $A$ such that $task(A)$ is unifiable with $\tau$.
For any action template $A \in \possiblerefinements{\tau}$, a possible resolver to the unrefined task $\tau$ is to apply the decomposition transformation $\chronicle \decomposition{\tau,a} \chronicle'$, where $a$ is an instance of $A$.
In the resulting chronicle $\chronicle'$, the unrefined task $\tau$ has been replaced by a new action instance $a$.

Furthermore, we need to account for the possible commitments already made for the support of unsupported assertions.
For this reason, we only consider as a resolver an action template $A \in \possiblerefinements{\tau}$ if 
 for any unsupported assertion $\alpha$ such that $\tau \in \descendanttasks{\alpha}$, $A$ has a possible effect $e \in \optimisticeffects{\tau}$ that could support $\alpha$.

\subsubsection{Possibly Conflicting Assertions}

This flaw occurs when a chronicle $\chronicle =(\plan_\chronicle,\assertions_\chronicle,\constraints_\chronicle)$ has an instantiation that makes two assertions conflicting.
We start by studying the case of persistence assertions and then consider change assertions.

Two persistence assertions $p_1=\persistencetuple{t^s_1,t^e_1}{sv_1}{v_1}$
  and $p_2=\persistencetuple{t^s_2,t^e_2}{sv_2}{v_2}$ are conflicting when they concurrently require distinct values for the same state variable.
  More precisely they are \textit{possibly conflicting} when the conjunction of these three conditions holds:

  \begin{LIST}
  \item they can refer to the same state variable, that is $sv_1$ and $sv_2$ are respectively of the form $\sv{sv}(x_1,\dots,x_n)$ and $\sv{sv}(y_1,\dots,y_n)$ and $\constraints_\chronicle \union \set{x_1=y_1 \wedge \dots \wedge x_n=y_n}$ is a consistent set of constraints,
 \item they can temporally overlap, i.e., $[t^s_1,t^e_1] \intersection [t^s_2,t^e_2] \neq \emptyset$ is consistent with the temporal constraints in $\constraints_\chronicle$,
 and
  \item their values can be different, i.e., $v_1 \neq v_2$ is consistent with $\constraints_\chronicle$.
  \end{LIST}

Resolving a conflict requires enforcing that the two assertions cannot be conflicting in any instantiation of the chronicle.
A conflict involving a pair of assertions $(p_1,p_2)$ can be resolved by a plan restriction $\chronicle \restriction{(\emptyset,\set{c})} \chronicle'$ where $c$ is one of the following:
\begin{LIST}
\item a state variable separation constraint, i.e., $c$ is one of $\set{x_1 \neq y_1, \dots, x_n\neq y_n}$ where $x_{1..n}$ and $y_{1..n}$ are the arguments of the state variables of $p_1$ and $p_2$ respectively, 
\item a temporal separation constraint, i.e., $c$ is either $t^e_1 < t^s_2$ or $t^e_2 < t^s_1$,
or
\item a unification constraint ($v_1 = v_2$) that requires the two values to be the same .
\end{LIST}

\def \undefined {\text{\textsc{undefined}}\xspace}
  
A change assertion \change{t_s,t_e}{sv}{v_s}{v_e} is more complex and, for the purpose of identifying conflicts, can be seen as a combination of three persistences:
\begin{LIST}
\item a start persistence assertion \persistence{t_s}{sv}{v_s},
\item an intermediate persistence $\tuple{]t_s,t_e[\ sv=\undefined}$, where \undefined is a special value that is not unifiable with any variable including \undefined itself (i.e.,  $\forall v, \undefined \neq v$), and
\item an end persistence \persistence{t_e}{sv}{v_e}.
\end{LIST}

A conflict involving a change assertion can be identified, and resolved, by reasoning on its three components, as for persistences.
We note that this definition allows the start and end points of the change assertion to overlap with other assertions.
However, the inner part of the change assertion that denotes the interval over which the value is changing cannot overlap with any other assertion on the same state variable.

For the purpose of identifying conflicts, an assignment assertion $\assignmenttuple{t}{sv}{v}$ is seen as a change assertion \changetuple{t-1,t}{sv}{any}{v} where $any$ is an unbounded variable that can take any value.

Conflicting assertions extend the threats of plan space planning with assertions spanning over temporal intervals.
The resolvers are separation constraints on temporal as well as on object variables  in the two conflicting assertions.

\subsection{Constraint Networks}
\label{sec:constraint-networks}

FAPE makes an important use of constraint networks to manage the temporal aspects of the problem and to maintain partially instantiated chronicles.
The constraint networks have to efficiently store the constraints accumulated by the planner and answer various queries necessary for the identification of flaws and finding resolvers.

For instance, finding whether assertions are possibly conflicting requires testing their possible temporal overlap and  unification of their state variables.
For this purpose we have two distinct constraint networks: one responsible for temporal reasoning and one for handling binding constraints.
A special case of constraints are \emph{duration constraints} that refer to variables present in both networks.
such as the one enforcing travel time in \autoref{fig:move-action}.
We start by presenting the temporal and binding constraint networks independently, and then show how they are synchronized to handle duration constraints.

\subsubsection{Temporal Constraint Network}
\label{sec:stn}
\def \timepoints {\ensuremath{\mathcal{X}}\xspace}

Temporal variables (or timepoints) and temporal constraints are handled as a Simple Temporal Networks (STN) \cite{Dechter1991}.
Given a set of timepoints \timepoints, a temporal constraint is an inequality of the form $t_2 -t_1 \geq d$, where $t_1$ and $t_2$ are timepoints in \timepoints and $d$ is an integer constant.
Such a constraint means that the minimal delay from $t_1$ to $t_2$ must be at least $d$.
The consistency and the minimal network of an STN can be computed in polynomial time.

The Floyd-Warshall All-Pairs Shorter Path algorithm \cite{Floyd1962,Warshall1962} gives the minimal network where the minimum and maximum distances between all pairs of timepoints are explicitly computed. This is performed in
$\Theta(\sizeof{\timepoints}^3)$.
We use the incremental version of this algorithm, as presented by \textcite{Planken-thesis}, that has a quadratic complexity of $O(\sizeof{\timepoints}^2)$.

We denote as $\diststn(t_1,t_2)$ the minimal delay between $t_1$ and $t_2$, which is given in the minimal network.
The maximum delay from $t_1$ to $t_2$ is simply $-\diststn(t_2,t_1)$.

\subsubsection{Binding Constraint Network}
\label{sec:binding-csp}
\def \bindlowerscript {Bind}
\newcommand{\variables}{\ensuremath{\mathcal{X}}}
\newcommand{\bindingvars}{\ensuremath{\variables_{\bindlowerscript}}\xspace}
\newcommand{\bindingconstraints}{\ensuremath{\constraints_{\bindlowerscript}}\xspace}
\newcommand{\bindingdomains}{\ensuremath{\mathcal{D}_{\bindlowerscript}}\xspace}
\def \bindingqueue {\ensuremath{\mathcal{Q}_{\bindlowerscript}}\xspace}

A binding constraint network is a tuple $(\bindingvars,\bindingdomains,\bindingconstraints)$ where:
\begin{LIST}
\item \bindingvars is a set of object variables and integer variables.
\item \bindingdomains is a set of finite domains for each variable in \bindingvars.
  The initial domain of object variables is a subset of the objects in the problem having the type of the object variable. The initial domain of integer variables is the set of integers.
  
\item \bindingconstraints is a set of constraints on the variables of \bindingvars.
  A constraint is either:
  \begin{LIST}
  \item an equality constraint, e.g., $x_1=x_2$,
  \item a difference constraint, e.g., $x_1 \neq x_2$,
  \item a general relation constraint, e.g., $\function{travel-time}(x_1,\dots,x_{n-1}) = x_{n}$. Relation constraints are associated with a  table that gives the allowed tuple of values for the variables $x_{1..n}$,
  \item an inequality on a integer variable, e.g., $x_1 \leq 12$ where $x_1$ is a integer variable, or
  \item disjunctive equality constraint, e.g., $x=x_1 \vee x=x_2\vee \dots \vee x=x_n$.
  \end{LIST}
\end{LIST}

The main purpose of the constraint network is to detect inconsistencies in the set of constraints and to answer various queries on binding variables.
Typical queries on the constraint network include \i the domain of a variable, \ii knowing if two variables are equal, \iii knowing if two variables can be made equal or different. These constraints  correspond to NP-complete CSPs.
(Networks with inequality constraints alone are NP-Complete.)

To find a trade-off between accuracy and efficiency we use an arc-consistency algorithm (AC-3 \cite{Mackworth1977}) for constraint propagation.
More specifically, we maintain a work list of constraints \bindingqueue.
When a new constraint is added to the network, it is added to \bindingqueue.
Until $\bindingqueue$ is empty, propagation iterates overs the following three steps:
\begin{LIST}
\item[\i] A constraint $c$ is extracted from \bindingqueue.
\item[\ii] For any variable $x$ involved in $c$,  remove from $dom(x)$ any value that cannot satisfy this constraint.\\
\tab For a relation constraint $c = \tuple{R(x_1,\dots,x_{n-1})=x_n}$, we first remove from the relation any tuple of values $(v_1,\dots,v_n)$ that is not achievable (i.e., $\exists i\in [1,n] \suchthat v_i \not\in dom(x_i)$).
For all remaining tuples, projecting them on a single variable $x_i$ gives a set of allowed values ${Proj}^c_{x_i}$.
The domain of each variable $x_i$ is restricted to the values in ${Proj}^c_{x_i}$.\\
\tab For a disjunctive equality constraint $c = \tuple{x=x_1 \vee x=x_2\vee \dots \vee x=x_n}$, the domain of $x$ is restricted to $\Union_{i\in[i,n]} dom(x_i)$. 
\item[\iii] If the domain of any variable $x$ was updated during propagation, all constraints involving $x$ are added to \bindingqueue.
\end{LIST}

If a variable $x \in \bindingvars$ has an empty domain, the network is not consistent.
Otherwise the network is \emph{arc-consistent} which is used as an optimistic approximation of consistency by the planner.
Full consistency is only checked at the end to ensure that a plan with no flaw is indeed a solution to a planning problem (i.e., that the binding constraint network is consistent).
If the network is not fully consistent, the corresponding chronicle is a dead-end and the planner backtracks.

The network is used to identify flaws and filter out impossible resolvers.
The binding constraint network considers that two variables $x_1$ and $x_2$ are:
\begin{LIST}
\item \emph{Unified} if they have the same singleton domain or if there is an equality constraint between the two.
\item \emph{Separable} if they are not unified.
\item \emph{Unifiable} if they have intersecting domains and there is no inequality constraints between them (or between any two variables unified with $x_1$ and $x_2$ respectively).
\item \emph{Separated} if they are not unifiable, i.e.,  they have non intersecting domains or there is an inequality constraint between them.

\end{LIST}

\subsubsection{Duration Constraints}\label{sec:dur-var}

Consider the constraint $t_2 - t_1 \geq \delta$ where $t_1$ and $t_2$ are timepoints and $\delta$ is an integer variable in the binding constraint network. This \textit{duration constraint} involves variables of both networks. For instance, the introduction of an instance $a$ of the \act{move} action of \autoref{fig:move-action}, would result in the following constraints:
\begin{align*}
  \function{travel-time}(r,d,d') &= \delta  \\
 \wedge~ end(a) - start(a) &\geq \delta  \\
 \wedge~ start(a) - end(a) &\geq -\delta
\end{align*}
where $r$, $d$ and $d'$ are object variables, $start(a)$ and $end(a)$ are timepoints and $\delta$ is a duration variable.
The last two constraints are mixed constraints as they involve variables from the temporal network and from the binding constraint network.

A duration constraint is handled as follows:
\begin{LIST}
  \item When a new constraint $d \geq t_2 - t_1$ is added or inferred in the temporal network, we remove from the domain of $\delta$ any value that would force a duration greater than $d$:
    \[dom(\delta) \gets \set{ v \suchthat v \in dom(\delta) \AND d \geq v } \]
    If the domain of $\delta$ is modified by this operation, it will trigger propagation in the binding network.
  \item When the domain of $\delta$ is modified in the binding constraint network, a new temporal constraint $t_2 - t_1 \geq 
  \minimum \set{ v \suchthat v \in dom(\delta) }$ is added to the STN. I.e the least constraining instantiation of $\delta$ is propagated into the temporal network.
\end{LIST}

To handle the more general form of a duration constraint $t_2-t_1 \geq \funcf(x_1,\ldots, x_n)$ where $t_1$ and $t_2$ are timepoints of the temporal network, $x_1,\ldots, x_n$ are object variables in the binding network,
and \funcf is a function given as or transformed into a table, we reify $\funcf(x_1, \dots, x_n)$ by \i introducing a new integer variable $\delta$, \ii imposing the relation constraint $\funcf(x_1, \dots, x_n) = \delta$ in the binding network and \iii enforcing $t_2 - t_1 \geq \delta$ as above.

Hence the constraint $\tend - \tstart = \function{travel-time}(r,d,d')$ of \autoref{fig:move-action} is effectively reformulated into the following conjuncts that can be handled separately:
\begin{equation*}
  \function{travel-time}(r,d,d') = \delta  
    ~~\wedge~~ \tend - \tstart \geq \delta  
    ~~\wedge~~ \tstart - \tend \geq -\delta
\end{equation*}

\subsection{Search Space: Properties and Exploration}
\label{sec:search-space}

Here we analyze the characteristics of the search space of  \FapePlan.

\paragraph{Search space dimension.}
The search space of plan-space planners is infinite, since they explore the set of plans which is infinite \cite[Chap. 5]{Ghallab2004}.
Indeed, even for a simple goal such as going somewhere, there can be an infinite set of plans fulfilling it: one can go round and round in circles an arbitrary number of time before getting to the destination.
Similarly, the presence of recursive methods can make the search space of HTN planners infinite \cite{Erol1994a}

Since our search procedure is an extension to plan-space planning and can represent HTN problems, the search-space of \FapePlan is infinite as well.
Under these conditions, the planner must guarantee a methodical exploration of the search space to ensure that, if there is a solution plan, it will be found.
Algorithms such as depth-bounded search or iterative deepening meet those expectations.
More interestingly, best-first search (and A* in particular) also meet those expectations as long as the addition of an action to the chronicle has a strictly positive cost for its evaluation function.
Indeed when the number of actions in a plan approaches infinity, the evaluation functions will eventually prefer a shallower plan that is shorter.
Our planner uses a best-first search algorithm whose evaluation function allows for a methodical exploration (as we will detail in \autoref{sec:general-search-strategy}).

\paragraph{Search space acyclicity.}
Another condition for the completeness of the search is either an acyclic space or an adequate handling of cycles, if any, to avoid infinite loops.

\begin{proposition}
  The search space of \FapePlan is acyclic.
\end{proposition}
\begin{proof}
  To show the acyclicity of the search space, we show that given a chronicle, there is no infinite sequence of resolver applications that can lead to the same chronicle.

  We first observe that the $\procedure{Transform}(\chronicle,\rho)$ of \autoref{alg:plan} is incremental: it can add actions or constraints but never removes anything.
  Furthermore, each $\procedure{Transform}$ application results in one flaw being solved (even in the case where the resolution of an unsupported assertion is delayed, this unsupported assertion is only reconsidered when at least one other flaw has been fixed).
  Since there are only a finite number of flaws in a chronicle, the repeated application of \procedure{Transform} will either \i result in a solution plan with no flaw, \ii result in an inconsistent plan or \iii add a new action to the plan, possibly introducing new flaws.
  In the latter case, the added action can never be removed and the planner cannot transform it back into the original chronicle.
\end{proof}

\paragraph{Soundness \& Completeness.}
We now consider the soundness and completeness of our planning procedure.

\begin{restatable}{proposition}{FapeSoundComplete}
  \label{prop:fape-sound-complete}
  \FapePlan is sound and complete.
\end{restatable}

\begin{proof}[Proof (Sketch)]
The full version of this proof is given in Appendix \ref{proof:sound-complete}.

For soundness, we remark that each type of flaw identifies one of the conditions of \autoref{def:solution-plan}.
Thus if a plan has no flaw, it fulfills all three requirements.
The requirement that the solution plan be \emph{reachable} from the initial chronicle is always fulfilled as all resolvers are defined in terms of the allowed transformation.

Given that the search space is acyclic and that the exploration is methodical, showing completeness requires us to show that a solution appears in the search space if one exists.
We show this by proving that, for each flaw, the set of resolvers is complete and does not overlook any solution plan.
\end{proof}

\paragraph{Discussion.}
The definition of  resolvers results in a top down approach to hierarchical planning, similar to that of other HTN planners. %
Indeed, the search starts from an existing task network, which, at the end, is entirely decomposed over the course of planning.
In practice, this means that any action in the plan is inserted after the action it possibly descends from.
This scheme has the advantage of being well documented and understood.

This decomposition procedure is integrated into generative planning, using partially instantiated plans.
As a result, the planner is able to handle temporal planning problems spanning from fully generative to fully hierarchical, together with mixed specification of generative and hierarchical problems.

\section{Search Control}
\label{sec:search-control}

Section~\ref{sec:search-space} defined the search space of FAPE,
obtained by extending the PSP algorithm to handle hierarchical
decomposition. Compared to most planners, the search space of
constraint-based planners (including FAPE) is much more compact as the lifted representation allows representation of many partial plans in a single search node.
On the other hand, the maintenance of various constraint networks make the expansion of a search node computationally expensive and such planners still require good search strategies to perform efficiently.
In this section, we detail the strategies used for shaping and exploring the search space in FAPE.
These can be roughly separated into three categories:
\begin{itemize}
\item the strategy for \emph{choosing which flaw to solve next} plays an important role in plan-space planning as it allows shaping the search space.
  For instance, choosing to first decompose all unrefined tasks would lead to a very different search space than the one resulting from resolving conflicting assertions first.
\item the strategy of \emph{which partial plan to consider next} is critical to guide exploration of the search space towards a solution plan, even for planning problems involving only a handful of actions.
\item \emph{inference of necessary constraints} to allow an early detection of specific characteristics of a partial plan.
  An example of such inference capability would be the detection that a given unsupported assertion cannot be achieved early because its establishment requires a long chain of actions.
  Such new constraints can be used to reduce the search space by detecting dead-ends or filtering out impossible resolvers.
\end{itemize}

These three aspects are strongly interconnected since flaw ordering and inference can be seen as shaping up and restricting the search space that will be explored given the partial plan selection strategy.
We first explain the techniques used for automatic inference in FAPE (Sections \ref{sec:inst-supp-constr} to \ref{sec:causal-network}).
More specifically, \autoref{sec:reach-analys} will detail a reachability analysis to infer what cannot be achieved from a partial plan due to its temporal and hierarchical features and
\autoref{sec:causal-network} will provide techniques to reason about the current causal structure of a plan.
Finally, \autoref{sec:search-strategies} will focus on the strategies for flaw and plan selection.

\subsection{Instantiation and Refinement Variables}
\label{sec:inst-supp-constr}

\newcommand{\instantiationvar}[1]{\ensuremath{I_{#1}}\xspace}
\newcommand{\refinementvar}[1]{\ensuremath{R_{#1}}\xspace}
\newcommand{\dom}[1]{\ensuremath{dom(#1)}\xspace}

While the lifted representation is beneficial for the efficiency of the search, it makes reasoning on a partial plan harder because a given lifted assertion can be refined into many ground ones.
This can be detrimental as most heuristic computations for planning rely on a ground representation.
To facilitate the mapping between lifted partial plans and ground heuristic techniques, we introduce two types of variables that represent the possible instantiations of actions and the set of ground actions that can be used to refine a task.

\logicalpar
Any action template \act{act} is associated with a relation $\gamma_{\act{act-inst}}$ that contains all possible ground instances of \act{act}.
A ground instance of $\act{act}$ is one where all object variables are bound and satisfy all binding constraints in the template.
If \act{act} has a ground instance $id = \tuple{\act{act}(\constant{c1},\dots,\constant{cn})}$, then $\gamma_{\act{act-inst}}$ will contain the tuple $(\constant{c1}, \dots, \constant{cn}, id)$ where $id$ is a unique identifier of this instance.
\autoref{tab:move-isntantiations} gives an example of this relation for the \act{move} action of \autoref{fig:move-action}.

\begin{table}[htb]
  \centering
  \begin{tabular}{c c c c}
    \textbf{Robot} & \textbf{Origin} & \textbf{Destination} & \textbf{Action ID} \\
    \constant{r1} & \constant{d1} & \constant{d2} & $id_1$ \\
    \constant{r1} & \constant{d2} & \constant{d3} & $id_2$ \\
    \constant{r1} & \constant{d2} & \constant{d1} & $id_3$ \\
    \constant{r2} & \constant{d1} & \constant{d2} & $id_4$ \\
    \constant{r2} & \constant{d2} & \constant{d3} & $id_5$ \\
    \multicolumn{4}{c}{\textbf{\dots}}
  \end{tabular}
  \caption{
    Table for the relation $\gamma_{\act{move-inst}}$ that give the instantiations of the \act{move} action.
    In this example, $id_1$ identifies the ground action $\act{move}(\constant{r1},\constant{d1},\constant{d2})$.
  }
  \label{tab:move-isntantiations}
\end{table}

\logicalpar
A (lifted) action $a: \act{act}(x_1,\dots,x_n)$ in $\plan_\chronicle$ is associated with an instantiation variable \instantiationvar{a} that takes a value in the set of ground instances of \act{act}.
At a given time, \dom{\instantiationvar{a}} should contain every ground action that $a$ might become once all its parameters are instantiated:

\[\act{act}(\constant{c1},\dots,\constant{cn}) \in dom(I_a) \Longleftrightarrow \forall_{i \in 1..n}\ \constant{ci} \in dom(x_i)
\]

To enforce, this relationship, every time a (lifted) action instance $a = \tuple{\act{act}(x_1,\dots,x_n)}$ is added to the partial plan, a 
constraint $(\tuple{x_1,\dots,x_n,I_a},\gamma_{\act{act-inst}})$ is added to the binding constraint network.

This instantiation variable has two main benefits.
First it allows us to find all possible instantiations of an action which is a useful input for the reasoning techniques we will introduce later in this section.
This set of instantiations is iteratively refined through constraint propagation whenever the parameters of the action are updated.
Second, when a ground action is found to be impossible, it can simply be removed from the allowed tuples in the relation.
This change will trigger propagation in the binding constraint network and be reflected on the parameters of actions in the partial plan.

\logicalpar
Similarly, every (lifted) task $\tau : \tsk{tsk}(y_1,\dots,y_n)$ is associated with a refinement variable \refinementvar{\tau} with values being the ground actions whose task is \tsk{tsk}.
At a given time, \dom{\refinementvar{\tau}} contains any ground action that might be used as a refinement for $\tau$:

\[a \in \dom{\refinementvar{\tau}} \Longleftrightarrow 
   task(a) = \tsk{tsk}(\constant{c1},\dots,\constant{cn}) \AND \forall_{i \in 1..n}\ \constant{ci} \in dom(y_i)\]

More specifically, a task symbol \tsk{tsk} is associated with a relation $\gamma_\tsk{tsk-ref}$ that gives the possible refinements of a task.
If there is a possible ground action $id = \act{act}(\constant{c1},\dots,\constant{cn})$ whose task is $\tsk{tsk}(\constant{p1},\dots,\constant{pm})$, then the relation $\gamma_\tsk{tsk-ref}$ will contain the tuple $(\constant{p1},\dots,\constant{pm},id)$. 
This identifies that the action (uniquely identified by $id$) is a possible refinement of the task $\tsk{tsk}(\constant{p1},\dots,\constant{pm})$.
When a new unrefined task $\tsk{tsk}(y_1,\dots,y_n)$ is inserted into the partial plan, it is associated with a refinement variable $\refinementvar{\tau}$ %
and constraint $(\tuple{y_1,\dots,y_n,\refinementvar{\tau}},\gamma_{\tsk{tsk-ref}})$.

When an action $a$ is introduced as a refinement of task $\tau$, their respective instantiation and refinement variables are unified by the introduction of a constraint $\instantiationvar{a}=\refinementvar{\tau}$.

\logicalpar
Instantiation and refinement variables make explicit the relationships between lifted actions and tasks in a partial plan and the ground actions that will appear in a solution plan.
In search control, those variables are used to transform lifted partial plans into a grounded relaxed problem for reachability analysis and heuristic computation.
The result of reachability analysis is also used to remove unreachable actions from the domains of both instantiation and refinement variables, as we will see in the next subsection.

\subsection{Reachability Analysis}
\label{sec:reach-analys}

Reachability analysis has been a crucial component of many planning systems including \textsc{ff} \cite{Hoffmann2001a} and \textsc{popf} \cite{DBLP:conf/aips/ColesCFL10}.
This analysis is done by solving a relaxed version of the planning problem in order to infer optimistic estimates of the set of states reachable from the initial state.
This further allows inferring which actions and fluents might appear in a solution plan.
In classical planners, this analysis is often done by taking the delete-free relaxation of a planning problem.
A set of reachable actions is then incrementally expanded by adding, one by one, any action whose preconditions are achieved either in the initial state or by an action in the reachable set.

This technique is not directly applicable to temporal problems that might contain inter-dependent actions.
Indeed, if two actions $A$ and $B$ are interdependent (as in \autoref{fig:concurrent-actions}), inferring that $B$ is reachable requires knowing that $A$ is reachable, which would require prior knowledge that $B$ is reachable.
Planners such as \textsc{popf} get around this problem by further relaxing delete-free problems: each durative action $A$ is split in two instantaneous at-start and at-end actions, $A_{start}$ and $A_{end}$.
$A_{start}$ contains only the start conditions and start effects of $A$, while $A_{end}$ contains only the end conditions and end effects of $A$.
Since $A_{start}$ does not contain the end conditions $A$ this eliminates any possibility of inter-dependency \cite{Coles2008,Cooper2013}.
To illustrate how reachability works, in the above example if $A_{start}$ were reachable, $B_{start}$ would be reachable, allowing $B_{end}$ to be reachable.  This, in turn, allows $A_{end}$ to be reachable.  This works fine as long as $B$ is shorter than $A$.  However, it is overly optimistic when $B$ is longer than $A$, leading the analysis to conclude that both actions are reachable when they are not.

\begin{figure}[htb]
  \centering
  \begin{tikzpicture}[node distance=.4cm]%
    \node[timed-act/action,minimum width=6.3cm,minimum height=.2cm,inner sep=3pt] (A) {A \smaller (duration: 10)};
    \node[timed-act/condition] at (A.north east) {y};
    \node[timed-act/effect] at (A.south west) {x};
    \node[timed-act/action,below right=1.3cm and 0.7cm of A.west,anchor=west,minimum width=4cm,minimum height=.3cm,inner sep=3pt] (B) {B \smaller (duration: 7)}; \node[timed-act/condition] at (B.north west) {x}; 
    \node[timed-act/effect] at (B.south east) {y};
  \end{tikzpicture}
  \caption{Two interdependent actions: $A$ with a start effect $x$ and an end condition $y$, and $B$ with a start condition $x$ and an end effect $y$.}
  \label{fig:concurrent-actions}
\end{figure}
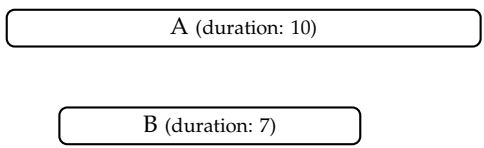

Hierarchical problems introduce many of these kinds of problems, because a higher level task requires that the subtasks be contained within it, and the subtasks require the presence of the higher level task.
An illustration of such interdependencies is given in \autoref{fig:transport-dependencies}.

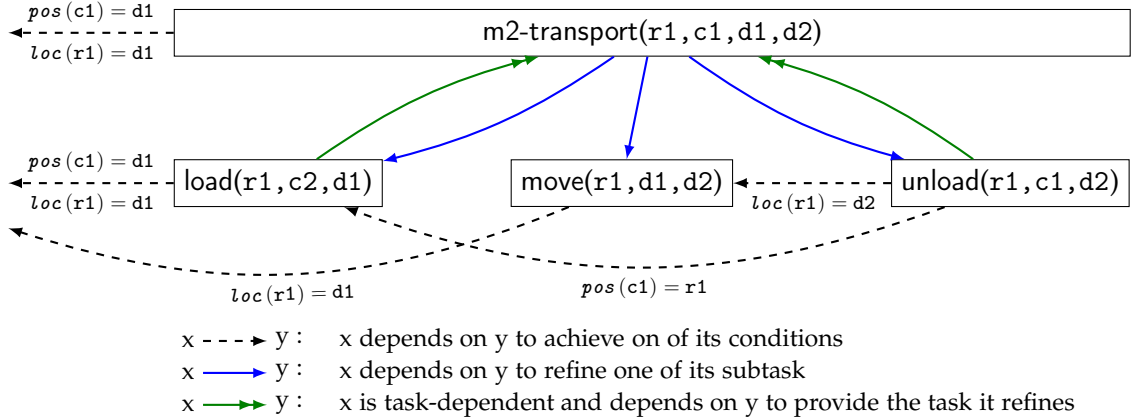
\begin{figure}[!phtb]\centering  
  \def \PosCOne {\sv{pos}(\constant{c1})}
  \tikzstyle{causal-dep} = [thick,dashed,->]
  \tikzstyle{ref-dep} = [thick,color=blue,->]
  \tikzstyle{task-dep} = [thick,color=darkgreen,->>]
  \begin{tikzpicture}[auto,scale=1, every node/.style={transform shape,color=black}]
    \node[hier-schema/action] (a1) {\act{load}(\constant{r1,c2,d1})};
    \node[hier-schema/action, right=1.7cm of a1] (a2) {\act{move}(\constant{r1,d1,d2})};
    \node[hier-schema/action, right=2.1cm of a2] (a3) {\act{unload}(\constant{r1,c1,d2})};
    \hla[]{ha1}{a1}{a3}{\act{m2-transport}(\constant{r1,c1,d1,d2})}
    \path[causal-dep,below,sloped,font=\smaller\smaller]
    (a2) edge[bend left=20] node {$\LocROne=\constant{d1}$}  ($(a1.west) + (-2.2,-.6)$)
    (a3) edge node {$\LocROne=\constant{d2}$} (a2)
    (a3) edge[bend left=20] node {$\PosCOne=\constant{r1}$} (a1)
    (ha1) edge node[above] {$\PosCOne=\constant{d1}$} node {$\LocROne=\constant{d1}$} ($(ha1.west) + (-2.2,0)$)
    (a1) edge node[above]  {$\PosCOne=\constant{d1}$} node {$\LocROne=\constant{d1}$}  ($(a1.west) + (-2.2,0)$)
    ;
    \path[ref-dep]
    (ha1) edge[bend left=10] (a1)
    (ha1) edge (a2)
    (ha1) edge[bend right=10] (a3);
    \path[task-dep]
    (a1) edge[bend left=10] (ha1)
    (a3) edge[bend right=10] (ha1);

    \begin{scope}[shift={(-1,-2.1)},y=.5cm,scale=0.85] %
      \draw (0,0) node[left] {x} edge[causal-dep] (1,0) node[right=1cm] {y~:} node[right=2cm] {x depends on y to achieve on of its conditions};
      \draw (0,-1) node[left] {x} edge[ref-dep]   (1,-1) node[right=1cm] {y~:} node[right=2cm] {x depends on y to refine one of its subtask};
      \draw (0,-2) node[left] {x} edge[task-dep]  (1,-2) node[right=1cm] {y~:} node[right=2cm] {x is task-dependent and depends on y to provide the task it refines};
    \end{scope}

  \end{tikzpicture}
  \caption{
    Dependencies between actions in a plan achieving a single $\tsk{transport}(\constant{c1,d2})$ task.
    The \act{m2-transport} action is the one of \autoref{fig:transport-method}. 
    Its three subtasks are refined by a \act{move} action (\autoref{fig:move-action}) and two \emph{task-dependent} \act{load} and \act{unload} actions.
  }     
  \label{fig:transport-dependencies}
\end{figure}

In this section, we describe a reachability analysis technique that takes into account the hierarchical properties of a problem and supports interdependent actions with no additional relaxation.
As in most existing planners, we consider delete-free actions in our relaxed model.
To give an intuition of what a delete-free relaxation would be in our planning model, consider a change assertion \change{t_s,t_e}{sv}{\constant{a}}{\constant{b}}.
This transition requires $sv$ to have the value $\constant{a}$ at time $t_s$ (i.e. that the fluent $\fluent{sv}{\constant{a}}$ holds at $t_s$) and states that $sv$ will have the value $\constant{b}$ at time $t_e$, meaning that the $\fluent{sv}{\constant{a}}$ no longer holds at time $t_e$.
This change assertion thus has a positive effect on the fluent $\fluent{sv}{\constant{b}}$ and a delete effect on the fluent $\fluent{sv}{\constant{a}}$.
In a delete-free model, we would consider that after the change assertion the state variable has both the values $\constant{a}$ and $\constant{b}$.
Those delete-free actions are extended with additional conditions and effects that account for hierarchical features of the original actions.

This relaxed model is used to compute a set of actions that might appear in a solution together with information about their earliest appearance.
This information is then used to infer constraints on partial plans as well as to detect dead ends in the search.

\subsubsection{Relaxed Problem}
\label{sec:elementary-actions}

\def \aflat {\ensuremath{a_{flat}}\xspace}
\def \condaflat {\ensuremath{C_{\aflat}}\xspace}
\def \effaflat {\ensuremath{E_{\aflat}}\xspace}

In order to perform a reachability analysis, we start by defining a relaxed planning problem.
A relaxed problem is a tuple $\tuple{F,A,I}$ where $F$ is a set of fluents, $A$ is a set of elementary actions and $I$ is the initial set of timed fluents.

The set of fluents $F$ contains all fluents in the original domain (i.e. all combinations of a state variable and a value).
Furthermore, for any task symbol $\tau$ in the original domain, $F$ is extended with $required(\tau)$, $started(\tau)$ and $ended(\tau)$,
which take a value in $\{\true,\false\}$. Those respectively represent that an action refining $\tau$ is needed, has started and has ended.

The elementary actions in $A$ are simple temporal actions with a set of conditions and a single positive effect. Conditions and effects are on fluents in $F$ and can represent causal or hierarchical requirements and effects of the original action.

The initial set of timed fluents $I$, represents a set of fluents whose appearance in a solution plan is already supported.
It is built from both the set of \emph{a priori} supported assertions in the original problem and from the effects of actions already in the partial plan.

Given a relaxed problem, the objective is to find the reachable subsets of $F$ and $A$ that can be reached from the facts in $I$.

\paragraph{Elementary actions.}
We start by transforming our actions (both high-level and primitive) into a set of simpler actions with only single effects. 
Our objective is to obtain delete-free actions that still encompass temporal and hierarchical aspects of the original action.
We start by associating any action $a$ of the planning domain with a flat action $a_{flat}$ with a set of conditions and effects such that:

\begin{itemize}
\item for any persistence condition $\timedfacttuple{t,t'}{sv = v}$ in $a$, \aflat has the condition $\timedfacttuple{t}{sv=v}$
\item for any assertion $\changetuple{t,t'}{sv}{v}{v'}$ in $a$, \aflat has the condition $\timedfacttuple{t}{sv=v}$ and the effect $\timedfacttuple{t'}{sv=v'}$
\item for any assignment assertion \assignmenttuple{t}{sv}{v} in $a$, $a_{flat}$ has the effect $\timedfacttuple{t}{sv=v}$
\item given the task $\tau_a$ achieved by $a$, \aflat has the effect $\timedfacttuple{\tstart}{started(\tau_a)=\true}$ and the effect $\timedfacttuple{\tend}{ended(\tau_a)=\true}$
\item if $a$ is task-dependent and its task is $\tau_a$, then \aflat has the additional condition $\timedfacttuple{\tstart}{required(\tau_a)=\true}$
\item for every subtask $\timedfacttuple{t,t'}{\tau}$ of $a$, $a_{flat}$ has:
  \begin{itemize}
  \item two additional conditions $\timedfacttuple{t}{started(\tau)=\true}$ and $\timedfacttuple{t'}{ended(\tau)=\true}$
  \item one additional effect $\timedfacttuple{t}{required(\tau)=\true}$
  \end{itemize}
\item \aflat contains all constraints of $a$
\end{itemize}

\begin{figure}\centering
  \begin{pcode}
    \phead{$\act{move}_{flat}(\constant{r1}, \constant{d1}, \constant{d2})$}
    \pkey{conditions}
    $\timedfact{\tstart}{\sv{loc}(\constant{r1})=\constant{d1}}$ \1
    $\timedfact{\tstart}{ \sv{occupant}(\constant{d1})=\constant{r1}}$ \1
    $\timedfact{t'}{\sv{occupant}(\constant{d2})=\nil}$ 
    \pkey{effects}
    $\timedfact{\tend}{\sv{loc}(\constant{r1})=\constant{d2}}$ \1
    $\timedfact{t}{ \sv{occupant}(\constant{d1})=\nil}$ \1
    $\timedfact{\tend}{\sv{occupant}(\constant{d2})=\nil}$ \1
    $\timedfact{\tstart}{started(\tsk{move}(\constant{r1},\constant{d1},\constant{d2}))=\true}$ \1
    $\timedfact{\tend}{ended(\tsk{move}(\constant{r1},\constant{d1},\constant{d2}))=\true}$
    \pkey{constraints}
    $\sv{connected}(\constant{d1},\constant{d2})$ \1
    $\tend - \tstart = 10$ \1
    $\tstart<t<t'<\tend$
  \end{pcode}
  \caption{Flattened version of the $\act{move}(\constant{r1,d1,d2})$ operator whose template is given in \autoref{fig:move-action}. The implicit temporal constraints are shown explicitly here.}
  \label{fig:flat-move}
\end{figure}

Such flat actions define a relaxed version of the problem in which actions have no ``delete effects'' and hierarchical relationships are compiled as additional conditions and effects.
It is important to note that the resulting `flat' problem relaxes some hierarchical aspects of the original one.
Indeed, a given subtask $\tuple{[t_s,t_e]\ \tau}$  yields two conditions $\persistence{t_s}{started(\tau)}{\true}$ and $\persistence{t_e}{ended(\tau)}{\true}$.
Those two conditions can be fulfilled by distinct actions, thus ignoring temporal constraints on the unique action that should have achieved the subtask in the original model.
Furthermore, a single ``subtask effect'' could allow the presence of multiple task-dependent actions.
This is however not a problem as this transformation is simply meant to expose some hierarchical features of the problem for an optimistic reachability analysis.

The flat actions are still temporally complex and might feature numerous timepoints related by temporal constraints.
As a second preprocessing step, each flat action is split into simpler elementary actions where all delays are fixed.
An elementary action contains a single effect and the necessary conditions to achieve this effect with \aflat.
Specifically, given a flat action \aflat, the set of elementary actions for \aflat is given by:

\begin{itemize}
\item for each effect $e = \tuple{[t_e] f}$ in \aflat, creating a new elementary action $a_e$ with $\tuple{[1] f}$ as the only effect of $a_e$.
In practice this places the effect exactly one time unit after the start of the elementary action, which facilitates reasoning of the relative placement of conditions and effect in an elementary action.
\item any condition $c$ in \aflat is added to each $a_e$ with a least-constraining timing constraint on when $c$ is needed relative to the effect $e$.
By least-constraining, we mean the latest time at which $c$ can be required to be true given the temporal constraints.
For a condition $\tuple{[t_c] sv'=v'}$ in \aflat, this is achieved by adding to $a_e$ a condition $\tuple{[1+max(t_c-t_e)] sv'=v'}$ where $max(t_c-t_e)$ is the maximal delay from $t_e$ to $t_c$ as given by the temporal constraints in the action.%
\footnote{Recall that in a delete-free problem, if a fluent holds at a given time $t$, it will hold for any subsequent time $t' \geq t$.}
\end{itemize}

\begin{figure}\centering
  \begin{pcode}
    \phead{$\act{move}_{\sv{loc}(\constant{r1})}(\constant{r1}, \constant{d1}, \constant{d2})$}
    \pkey{conditions}
    $[-9]\ \sv{loc}(\constant{r1})=\constant{d1}$ \1
    $[-9]\ \sv{occupant}(\constant{d1})=\constant{r1}$ \1
    $[0]\ \sv{occupant}(\constant{d2})=\nil$ \1
    $[-9]\ required(\tsk{move}(\constant{r1},\constant{d1},\constant{d2}))=\true$
    \pkey{effects}
    $[1]\ \sv{loc}(\constant{r1})=\constant{d2}$
    \pkey{constraints}
    $\sv{connected}(\constant{d1},\constant{d2})$
  \end{pcode}

  \begin{pcode}
    \phead{$\act{move}_{\sv{occupant}(\constant{d1})}(\constant{r1}, \constant{d1}, \constant{d2})$}
    \pkey{conditions}
    $[0]\ \sv{loc}(\constant{r1})=\constant{d1}$ \1
    $[0]\ \sv{occupant}(\constant{d1})=\constant{r1}$ \1
    $[9]\ \sv{occupant}(\constant{d2})=\nil$ \1
    $[0]\ required(\tsk{move}(\constant{r1},\constant{d1},\constant{d2}))=\true$
    \pkey{effects}
    $[1] \sv{occupant}(\constant{d1})=\nil$ 
    \pkey{constraints}
    $\sv{connected}(\constant{d1},\constant{d2})$
  \end{pcode}
  \caption{The first two elementary actions generated from $\act{move}_{flat}(\constant{r1,d1,d2})$}
  \label{fig:elementary-actions}
\end{figure}

\autoref{fig:elementary-actions} shows two elementary actions generated for the flat $\act{move}$ of \autoref{fig:flat-move}. Three additional elementary actions are needed to cover the last three effects of the action.

\paragraph{Timed fluents.}
We now describe how the initial set of timed fluents $I$ is derived from a partial plan.
It is meant to capture the fluents that are known to hold at a given time, regardless of whether they derive from timed initial literals or actions in the current partial plan.
A timed fluent $i \in I$ is denoted by $\tuple{[t] f}$ where $f$ is a fluent and $t$ is a time at which $f$ holds.
Given a partial plan $(\plan,\assertions,\constraints)$, $I$ is composed of:

\begin{itemize}
\item for any assertion $\tuple{\change{t_1,t_2}{sv}{v_1}{v_2}}$ or $\assignmenttuple{t_2}{sv}{v_2}$ in $\assertions$; if $sv'$ and $v_2'$ are instantiations of $sv$ and $v_2$ consistent with $\constraints$, $I$ contains $\tuple{\timedfact{t_2'}{sv'=v_2'}}$ where
 $t_2'$ is the smallest instantiation of $t_2$ consistent with \constraints, given by $\diststn(\timeorigin,t_2)$.
\item for any unrefined task $\tuple{\timedfact{t_1,t_2}{\tau}}$ in \plan; if $\tau'$ is an instantiation of $\tau$ consistent with \constraints, $I$ contains $\tuple{\timedfact{t_1'}{required(\tau')=\true}}$ where $t_1'$ is the smallest possible instantiation of $t_1$ consistent with \constraints.
\end{itemize}

\paragraph{Definitions.}
An elementary action is \emph{applicable} once all of its conditions are met.
An action with an effect $f$ is called an \emph{achiever} of the fluent $f$. 
A fluent becomes \emph{achievable} after one of its achievers becomes applicable.
As a consequence of using (delete-free) elementary actions, once a fluent is achievable or an action is applicable, it stays achievable/applicable at all subsequent time points.

Action $a$ is applicable at time $t$ (denoted by $applicable(a,t)$) if for all conditions $\tuple{[\delta] f}$ of $a$, $f$ is achievable at time $t+\delta$.
Similarly, a fact $f'$ is achievable  at time $t'$ (noted $achievable(f',t')$) if there exists an achiever of $f'$ applicable at time $t'-1$.

We say that an action $a$ (resp. a fluent $f$) is reachable if there exists a time $t$ such that $applicable(a,t)$ holds (resp. $achievable(f,t)$ holds).
The \emph{earliest appearance} of a reachable action $a$ (denoted by $ea(a)$) is the smallest $t$ for which $applicable(a,t)$ holds.
Similarly, the earliest appearance of a reachable fluent $f$ is the lowest $t$ for which $achievable(f,t)$ holds.

\subsubsection{Reachability analysis with inter-dependent actions}

\paragraph{The problem of inter-dependent actions.}

As shown by \textcite{Cooper2013}, the difficulty of doing reachability analysis with interdependent actions is due to the presence of \emph{after-conditions}: conditions that must hold when or after an effect of the action is achieved.
The end condition $y$ of the action $A$ of \autoref{fig:concurrent-actions} is an \emph{after-condition}.
Because the effect of our elementary actions are placed at time 1, an \emph{after-condition} can be easily detected as any condition \timedfacttuple{t}{f} where $t \geq 1$.
All other conditions are referred to as \emph{before-conditions}.
 
The approach taken by \textsc{popf} of splitting an action into instantaneous at-start and at-end actions means that the after-conditions of at-start actions are ignored.
As a result, a start effect can be considered as reachable even when an end condition is not.
In \autoref{fig:concurrent-actions}, the action $A$ would indeed become an action $A_{start}$ containing only an effect $x$ and an action $A_{end}$ with a condition $y$.
In this model, the effect $x$ can thus be produced independently of the after-condition on $y$.
This constitutes an additional relaxation resulting in the elimination of all interdependencies in the delete-free problem.

While this relaxation is reasonable for many generative planning problems, hierarchical problems typically feature many interdependencies between methods and their subtasks.
The next subsection describe a technique for reachability analysis that does not need any additional relaxation, thus taking all after-conditions into account.

\begin{algorithm}[hbtp]
  \caption{Algorithm for identifying reachable actions and fluents and computing their earliest appearance.}
  \label{algo:prop}
  \begin{algorithmic}[1]
    \State $A \gets \text{Elementary actions}$
    \State $F \gets \text{Fluents}$ 
    \State $I \gets \text{Initial timed fluents}$
    \State $Q \gets \emptyset$ \Comment Priority queue of $(action/fluent, time)$ ordered by increasing $time$
    \label{line:start-init}
    \aforall{$f \in F$}{
      \State $reachable(f) \gets \false$ }
    \aforall{$\tuple{\timedfact{t}{f}} \in I$}{
      \State $Q \gets Q \union \{(f,t)\}$ }
    \aforall{$a \in A$}{
      \State $reachable(a) \gets \false$
      \aif{$a$ has no before-conditions}{
        \State $Q \gets Q \union \{(n,0)\}$ } }

    \label{line:end-init}
    \State
    \awhile{Q non empty}{ \label{line:main-loop}
      \State \Call{DijkstraPass}{} \label{line:dij-pass-call} 

      \aforall{$a \in \text{A}$} {
        \aforall{$\tuple{\timedfact{\delta}{f}} \in \text{after-conditions of } a$} {
          \If{$\neg reachable(f)$}
            \label{line:rm-start}
            \State $(A,F) \gets \Call{RecursivelyRemove}{a}$
            \label{line:rm-end}
          \ElsIf{$ea(a) <  ea(f) - \delta$}
            \State $Q \gets Q \union \set{(a, ea(f) - \delta)}$    \label{line:enforce-end}
          \EndIf
        }
      }
      \aforall{$a \in A$}{
        \aif{$a$ is late}{
          \State $(A,F) \gets \Call{RecursivelyRemove}{a}$ \label{line:rm-late}
        }
      }
    }

    \State
    \aprocedure{DijkstraPass}{}{ \label{line:dij-pass-proc}
      \awhile{Q non empty}{
        \State $(n,t) \gets pop(Q)$
        \aif{$n$ already expanded in this pass}{
           \State \textbf{continue}
        }
        \aif{$reachable(n) \wedge ea(n) \geq t$}{ \label{line:inc-check-start}
          \State \textbf{continue} %
          \label{line:inc-check-end}
        }
        \State $reachable(n) \gets \true$ 
        \State $ea(n) \gets t$
        \aifelse{$n$ is an action with the effect \tuple{\timedfact{1}{f}}}{
          \State $Q \gets Q \union \set{(f,t+1)}$
        }{%
          \aforall{$a \in A$ with a condition on the fluent $n$}{\label{line:cond-edge-expansion-start}
            \aif{all before-conditions of $a$ are reachable}{ 
              \State $t' \gets max_{\tuple{\timedfact{\delta}{f}}\: \in\: \function{cond}(a)} ~~ea(f) - \delta$
              \State $Q \gets Q \union \{(a,t')\}$
              \label{line:cond-edge-expansion-end}
            }
          }
        }
      }
    }
  \end{algorithmic}
\end{algorithm}

\def \dmax {\ensuremath{d_{max}}\xspace}

\paragraph{Propagation.}
To handle after-conditions during reachability analysis, as detailed in \autoref{algo:prop}, we alternate two 
steps: \i a propagation that ignores all after-conditions by performing a Dijkstra-like propagation in the graph composed of all fluents and elementary actions; and \ii a second step that enforces all after-conditions.
Those two steps are complemented with a pruning mechanism that repeatedly detects actions that have unsolvable interdependencies.

\autoref{algo:prop} begins by selecting a set of \emph{assumed reachable} nodes (i.e. actions or fluents) from which to start propagation (lines \ref{line:start-init}-\ref{line:end-init}).
The obvious candidates are fluents known to be true at a given time, e.g., fluents achieved by assertions in the problem definition or by actions in the partial plan.
All such fluents have been previously inserted in the initial set of timed fluents ($I$) and are selected.
We also optimistically select all actions that have no \emph{before-conditions}, i.e., actions whose conditions are all after-conditions.
Those assumed reachable elements are inserted into a priority queue $Q$ of $\tuple{n,t}$ pairs where $n$ is either an elementary action or a fluent and $t$ is a candidate time for its earliest appearance.

The initial assumed reachable set is then iteratively extended with all fluents with an assumed reachable achiever and any action whose all before-conditions are assumed reachable.
This is done by a Dijkstra-like propagation (line \ref{line:dij-pass-call}) that extracts the items in $Q$ by increasing earliest appearances.
The corresponding actions (resp. fluents) are marked as reachable and the fluents (resp. actions) depending on them are inserted into $Q$.
More specifically, if a pair $\tuple{a,t}$ is extracted from $Q$ and $a$ is an action with the effect $\tuple{\timedfact{1}{f}}$, the pair $\tuple{f,t+1}$ is inserted into the queue.
If a pair $\tuple{f,t}$, with $f \in F$, is extracted from $Q$, all actions depending on $f$ that have all their before-conditions reachable are pushed onto $Q$ (lines \ref{line:cond-edge-expansion-start}-\ref{line:cond-edge-expansion-end}).

As a second step, we revise our optimistic assumptions by considering after-conditions:
\begin{itemize}%
\item Line~\ref{line:rm-end} removes any action $a$ with an after-condition on an unreachable fluent $f$.
  More specifically, the \procedure{RecursivelyRemove} procedure marks its parameter as unreachable and removes it from the set of actions.
  The removal is recursive: if a removed action is the only achiever for a fluent $f$ then $f$ is removed as well (and as a consequence all actions depending on $f$ will also be removed, etc.).
  Furthermore, if the first achiever of a fluent is removed and there is at least one other
  achiever for it, then the fluent is added back to $Q$  with an updated earliest appearance.

\item Line~\ref{line:enforce-end} takes an after-condition of an action $a$ on a reachable fluent $f$ and enforces the minimal delay $\delta$ between $ea(f)$ and $ea(a)$.
  If the current delay is not sufficient, $a$ is added to $Q$ and will be reconsidered upon the next Dijkstra pass.
\end{itemize}

Finally, \emph{late actions} are marked unreachable and removed from the graph (line~\ref{line:rm-late}).
We say that an action $a$ is \emph{late} if for any \emph{non-late} action $a'$, $\ ea(a') + \dmax < ea(a)\ $ where \dmax is the highest delay in the relaxed model (either of a timed initial literal or between a before-condition and an effect).
In practice, this means that actions are partitioned into non-late and late, these two sets being separated by a temporal gap of at least \dmax.
The intuition (demonstrated in the next subsection) is that the earliest appearance of a late action is being pushed back due to unachievable interdependencies with other late actions.

The two-step process is repeated (line~\ref{line:main-loop}) to take into account the newly updated reachability information.
In the subsequent runs, the Dijkstra algorithm will start propagating from the items updated by the previous iteration, with lines \ref{line:inc-check-start}-\ref{line:inc-check-end} making sure that the earliest appearance values $ea(n)$ are never decreased to a too optimistic value.
The algorithm detects a fix-point and exits if the queue is empty, meaning that after-conditions did not trigger any change.

\subsubsection{Analysis and Possible Variants}
\label{sec:reachability-properties}

We now explore some of the characteristics of Algorithm~\ref{algo:prop}.
The first Dijkstra pass acts as an optimistic initialization: it identifies a set of possibly reachable nodes and assigns them earliest appearance times.
All operations after this first pass will only \emph{(i)} shrink the set of reachable nodes; and \emph{(ii)} increase the earliest appearance times.

It is helpful to see the relaxed problem as a graph whose nodes are the fluents and elementary actions.
Edges either represent a condition (edges from a fluent to an action) or an effect (edges from an action to a fluent).

\begin{definition}[Causal loop]
  We denote as a causal loop a cycle of actions and fluents $f_0\rightarrow A_0 \rightarrow f_1 \rightarrow A_1 \dots A_n \rightarrow f_0$, such that each fluent $f_i$ is a condition of the elementary action $A_i$ and each action $A_i$ is an achiever of the fluent $f_{i+1}$.

  Each edge of this loop is associated with a delay that is respectively the delay from when an action $A_i$ can start to the moment its effect $f_{i+1}$ is achieved, or the delay from when a condition $f_i$ is needed to the moment its containing action $A_i$ can start.

  We say that a causal loop is self-supporting if its length (i.e. the sum of the delays on its edges) is less than or equal to 0.

  A causal loop is said to be unfeasible if its length is strictly positive.
\end{definition}

The notion of causal loop is crucial in the understanding of problems with interdependent actions.
The actions of \autoref{fig:concurrent-actions} form a self-supporting causal loop $A \xrightarrow{0} x \xrightarrow{0} B \xrightarrow{7} y \xrightarrow{-10} A$, which essentially means that $B$ can be used to produce the condition $y$ of $A$ early enough for $A$ to be executable.

On the other hand, if $B$ had a duration of $12$, we would have an unfeasible causal loop $A \xrightarrow{0} x \xrightarrow{0} B \xrightarrow{12} y \xrightarrow{-10} A$.
Indeed, $B$ does not ``fit'' in $A$ anymore and the planner must find another way to achieve either $x$ or $y$  to use those two actions.

\begin{restatable}{proposition}{EarliestAppearanceConvergence}\label{prop:ea}
  If a node (i.e. action or fluent) $n$ is reachable in the relaxed problem, then $ea(n)$ converges to a finite value.
  If a node $n'$ is not reachable then $ea(n')$ either remains at $\infty$ or diverges towards $\infty$ until it is removed from the graph.
\end{restatable}
\begin{proof}[Proof (Sketch)]
    We sketch the proof that is fully given in Appendix~\ref{sec:proof-earliest-appearance}. 
    An action or fluent $n$ is reachable if there is either a path from initial facts to $n$ or if $n$ is part of a self-supporting causal loop (i.e. cycle of negative or zero length).
  Consequently and because the earliest appearance can only increase, repeated propagations will eventually converge.
  On the other hand, an unreachable node either depends on an unreachable node or is involved only in causal cycles of strictly positive length.
  If the node was ever assumed reachable, its earliest appearance will thus be increased by \autoref{algo:prop} until it is removed from the graph.
\end{proof}

\begin{restatable}{proposition}{LateNodesUnreachable} \label{prop:late-impos}
  If a node is put in the late set, then it is not reachable.
\end{restatable} 
\begin{proof}[Proof (Sketch)]
  We sketch the proof that is fully given in Appendix~\ref{sec:proof-late-impos}. 
  The intuition is that the gap between non-late and late nodes appeared because late nodes are delaying each others due to positive causal cycles.
  We first show that any late node was delayed to its current time due to a dependency on another late node: because the temporal gap is bigger than all delays in the model, a non-late node could not have influenced a late node.
  It follows that any late node depends on at least one other late node.
  Furthermore a late node necessarily participates in a positive cycle or depends on a late node that does.
  From there, one can show that at least one node $n$ in this group is involved only in positive cycles.
  Any other possibility (path from initial timed literals or negative cycle) would have resulted in $n$ being less than \dmax away from a non-late node.
\end{proof}

\newcommand{\hfull}{\ensuremath{R_\infty}\xspace}
\newcommand{\hk}{\ensuremath{R_K}\xspace}
\newcommand{\hfive}{\ensuremath{R_5}\xspace} 
\newcommand{\hone}{\ensuremath{R_1}\xspace}
\newcommand{\hplus}{\ensuremath{R^{+}}\xspace}

It follows from propositions~\ref{prop:ea} and~\ref{prop:late-impos} that Algorithm~\ref{algo:prop} produces a reachability model (denoted as \hfull) that contains a fluent or action  $n$ and its earliest appearance $ea^*(n)$ iff $n$ is reachable in the relaxed problem.
In the worst case, computing this model has a pseudo-polynomial complexity since there may be as many as $\dmax$ iterations of the algorithm ($\dmax$ being the highest delay in the relaxed model).
As we will see in the experiments (\autoref{sec:evaluation-reachability}), this worst case doesn't tend to occur in practice; typically, quiescence occurs after a relatively small number of iterations.
The cost of each iteration is dominated by the Dijkstra pass of $O(\sizeof{N}\times \log(\sizeof{N}) +\sizeof{E})$, where $N$ is the number of fluents and actions and $E$ is the number of conditions and effects appearing in actions.

\paragraph{Discussion:} One might consider computing various approximations of \hfull by limiting the number of iterations to a fixed number $K$, making the algorithm polynomial in $O(K\times ( \sizeof{N}\times\log(\sizeof{N}) +\sizeof{E}))$ for producing a reachability model \hk.
In the special case where $K=1$, this is equivalent to 
performing a single Dijkstra pass and removing all actions with an unreachable
after-condition.
Increasing $K$ would allow the algorithm to better estimate the earliest appearances
and detect additional late nodes.

Another simplification is to ignore all after-conditions, which can be done by stopping Algorithm~1 after the first Dijkstra pass.
In practice, this model has all the characteristics of the temporal planning graph of {\sc popf}:
\i the separation of durative actions into at-start and at-end instantaneous actions is done by the transformation into elementary actions;
\ii the minimal delay between matching at-start and at-end actions is enforced by the presence of start conditions in the elementary actions representing the end effects; and
\iii any end condition appearing in the elementary action of a start effect would be ignored because it would be an after-condition.
This model, that we call \hplus, is thus a direct adaptation of the techniques used in {\sc popf} to our richer action representation.

Note that \hplus and \hfull are equivalent on all problems with no after-conditions.
Classical planning obviously falls in this category as well as any PDDL model with no at-start effect or no at-end condition. 
In fact, on such problems \hplus and \hfull are equivalent to building a temporal planning graph, with no additional computational overhead.

\subsubsection{Exploiting the results of a reachability analysis}
\label{sec:using-reachability}

For a given partial plan $\phi$, a reachability analysis provides us with:
\begin{itemize}
\item $Ra_\phi$, a set of actions reachable in the relaxed problem,
\item $Rf_\phi$, a set of fluents reachable in the relaxed problem,
\item $ea_\phi : (Ra_\phi \union Rf_\phi) \rightarrow \mathcal{N}$ a function associating each reachable action and fluent with an
  optimistic earliest time at which it can be added or achieved in a solution plan.
\end{itemize}

These are computed for any partial plan that is extracted from the priority queue for expansion.
Because all computed values are optimistic, \autoref{algo:prop} can be run incrementally by initializing the set of reachable nodes and earliest appearances with those computed for the previous partial plan.
While the worst case complexity of the incremental version is unchanged, our implementation suggests that it avoids a lot of redundant computation.
The results of a reachability analysis are used in many parts of the planner to prune parts of the search space and derive additional constraints on the current partial plan:
\begin{itemize} 
\item For any unsupported assertion $\alpha \in \assertions_\phi$, if its condition cannot be instantiated to a reachable fluent $f \in Rf_\phi$, then the partial plan is marked as a dead-end and search proceeds with the next best partial plan.
Otherwise we temporally constrain $\alpha$ to be at least as late as its earliest reachable instantiation. This is done by adding the following constraint to the STN:
\[
\diststn(\timeorigin,start(\alpha)) \geq \min \set{ea(f) \suchthat f \in Rf_\phi \cap dom(cond(\alpha)) }
\]
where $\timeorigin$ is the temporal origin and $dom(cond(\alpha))$ denotes all possible instantiations of the fluent that is required by $\alpha$.
\item We check that all unrefined tasks can be refined by a reachable action.
This is done by restricting the domain of any refinement variable (\autoref{sec:inst-supp-constr}) to reachable actions:

\[
\dom{\refinementvar{\tau}} \subseteq Ra_\phi
\]

If a task has no possible refinement (i.e. one refinement variable has an empty domain) then the partial plan is declared a dead-end.
Otherwise, the earliest start time of all unrefined tasks is updated to be at least as late as the earliest reachable action that can refine it.

\item When considering unsupported assertions or unrefined task flaws, we disregard any resolver involving an action with no reachable instances.
For instance, if there is no instances of the $\act{move}$ action in $Ra_\phi$, then the planner would not consider the insertion of $\act{move}$ to support an assertion on the location of a robot.
In this case, the planner would need to rely on assertions already in the partial plan.
\item All domain transition graphs (to be introduced in \autoref{sec:causal-network}) are updated by removing any transition provided by an unreachable action.
This update has indirect effects, since it allows more reliable information when reasoning on causal networks.
\item When creating the instantiation variables of \autoref{sec:inst-supp-constr}, the domain of these variables is constrained to be a subset of $Ra_\phi$.
This indirectly constrains the parameters of any newly added action to respect reachability requirements.
\end{itemize}

\subsection{Causal Network}
\label{sec:causal-network}

\def \cn {\ensuremath{G}\xspace}
\def \cl {\rightarrow}
\def \potsup {\ensuremath{\mathcal{S}}\xspace}

We define the causal network of a partial plan $\phi$ as the graph $\cn_\phi = \tuple{N,E}$ where $N$ is the set of assertions in $\assertions_\phi$ and $E$ contains an edge $x \rightarrow y$ iff there is a causal link stating that $x$ supports $y$.
This causal network is explicitly maintained by the planner by adding edges when new causal links are inserted and adding nodes when new assertions are introduced by newly added actions.
For a partial plan to be a solution, the corresponding causal network must be such that:
\begin{itemize}
\item every assertion $x \in N$ that is not a priori supported has an incoming edge (i.e. an incoming causal link),
\item any change assertion or a priori supported assertion $x \in N$ has at most one outgoing edge that targets a change assertion. In addition, $x$ might support several persistence conditions.
\end{itemize}

In this section, we show how this graph $\cn_\chronicle$ can be exploited to infer additional constraints on the partial plan and extract heuristic information.

\begin{definition}[Causal Chain]
  A causal chain is a sequence of change assertions $\tuple{\beta_1,\dots,\beta_n}$ such that for any element $\beta_i$ there is a causal link to its direct successor $\beta_{i+1}$.

  A causal chain spans over the temporal interval $[start(\beta_1),end(\beta_n)]$ and is said to be about the state variable $sv$ common to all its composing assertions. 
  Over its temporal interval, a causal chain fully constraints the evolution of its state variable.
\end{definition}

We say that two causal chains \emph{possibly overlap} if their state variables can be unified by consistent binding constraints and they span over two possibly overlapping temporal intervals. 
Two causal chains \emph{necessarily overlap} if every consistent instantiation overlaps.
Two necessarily overlapping causal chains result in an unsolvable threat because at least one change assertion of the first chain would temporally overlap a change assertion or a causal link of the second chain.

In order to facilitate the reasoning on the possible transitions that can be taken by a state variable, we now introduce Domain Transition Graphs.

\def \nodesdtg {V}
\def \edgesdtg {T}
\def \anydtg {\textsc{any}\xspace}
\newcommand{\distdtg}{\ensuremath{dist_{DTG}}}
\begin{definition}[DTG]
  The Domain Transition Graph (DTG) of a state variable $sv$ is a directed graph $(\nodesdtg,\edgesdtg)$ where $\nodesdtg$ is composed of the
  values that can be taken by $sv$ and a special node \anydtg.
  $\edgesdtg$ is a set of allowed transitions from one value to the other. An edge in $\edgesdtg$ is of the form $v_1 \xrightarrow{d} v_2$ meaning that the value of $sv$ can be changed from $v_1$ to $v_2$ in $d$ time units.
  In the special case where $v_1 = \anydtg$, it means that $sv$ can take the value $v_2$ regardless of its previous value (even if $sv$ had no known previous value).

  DTGs are used to reason on the changes that can be made to a state variable $sv$ through the addition of new actions in a partial plan.
  For any ground action $a$ that is reachable (according to reachability analysis):
  \begin{itemize}
  \item if the action contains a change assertion \changetuple{t_1,t_2}{sv}{v_1}{v_2}, then the DTG of $sv$ contains an edge $v_1 \xrightarrow{min(t_2-t_1)} v_2$,
  \item if the action contains an a priori supported assertion \assignmenttuple{t}{sv}{v}, then the DTG of $sv$ contains the edge $\anydtg \xrightarrow{1} v$.
  \end{itemize}
\end{definition}

We say that there is a \emph{feasible transition} of a state variable $sv$ from a value $v_1$ to a value $v_2$ if there is a path in the DTG of $sv$ from $v_1$ to $v_2$ or from \anydtg to $v_2$.
We denote as $\distdtg(v_1,v_2)$ the length of the shortest such path.
An example of a DTG is given in \autoref{fig:dtg-loc}.

\begin{figure}
  \centering
  \begin{tikzpicture}
    \def \n {7}
    \def \radius {1.5cm}
    \foreach \s in {0,...,6}
    {
      \node[circle] at ({360/\n * (- \s)}:\radius) (d\s) {$\constant{d\s}$};
    }
    \path[->, >=latex,bend left=20] (d0) edge node[above,sloped] {\tiny10} (d1);
    \path[->, >=latex,bend left=20] (d1) edge node[above,sloped] {\tiny10} (d2);
    \path[->, >=latex,bend left=20] (d2) edge node[above,sloped] {\tiny10} (d3);
    \path[->, >=latex,bend left=20] (d3) edge node[above,sloped] {\tiny10} (d4);
    \path[->, >=latex,bend left=20] (d4) edge node[above,sloped] {\tiny10} (d5);
    \path[->, >=latex,bend left=20] (d5) edge node[above,sloped] {\tiny10} (d6);
    \path[->, >=latex,bend left=20] (d6) edge node[above,sloped] {\tiny10} (d0);
  \end{tikzpicture}
  \caption{Example DTG of $\sv{loc}(r)$: where the location in which the robot $r$ can navigate are organized in a circular pattern. Moving from one place to the next takes 10 time units.}
  \label{fig:dtg-loc}
\end{figure}
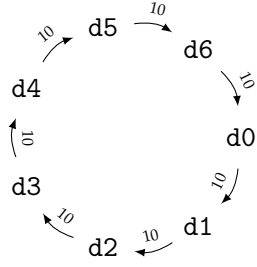

\subsubsection{Possible Supporters}
\label{sec:potential-supporters}

In order to appear in a solution plan, any unsupported assertion must eventually be linked to a supporting assertion. This link can take the form of a single causal link or of a chain of causal links going through statements not yet in the plan. We refer to the candidates for such supporting assertions as \emph{possible supporters}.

\begin{definition}[Possible supporter]
  Given the causal network $\cn_\phi$ of a partial plan $\phi$, a change assertion $\beta$ is a \emph{possible supporter} of an unsupported assertion $\alpha$ if there is a set of statements $\{s_1,\dots,s_n\}$ and a chain of causal links $\beta \cl s_1 \cl \dots \cl s_n \cl \alpha$ that can be added to $\cn_\phi$. 
\end{definition}

For an unsupported assertion $\alpha$, we consider a superset of the set of possible supporters, noted $\potsup_\alpha$. This set is incrementally updated to contain any change assertion $\beta \in N$ that satisfies the following required conditions:
\begin{itemize}
\item the state variables of $\alpha$ and $\beta$ are unifiable.
\item there is a feasible transition in the DTG from the value produced by $\beta$ to the value required by $\alpha$.
\item adding a chain of causal links from $\beta$ to $\alpha$ will not result in any unsolvable threat. We consider that there is an unsolvable threat, if the causal chain obtained by concatenating the current causal chains of $\alpha$ and $\beta$ would necessarily overlap an existing causal chain on the same state variable.
\end{itemize}

In the causal network example of \autoref{fig:ex-causal-network}, the possible supporters of  $\alpha$ would be $\beta$ and $\mu$  because $\beta/\mu$ can come before $\alpha$ and the DTG has a path from $d_1/d_6$ to $d_2$.
The possible supporters of $\gamma$ would be $\beta$ and $\alpha$ because $\beta/\alpha$ can come before $\gamma$ and the DTG has a path from $d_1/d_3$ to $d_4$.  
 The possible supporters of $\rho$ are $\alpha$ and $\mu$; indeed any causal chain from $\beta$ to $\rho$ would be threatened by $\alpha$ and by the causal chain of $\gamma$ and $\mu$.

 In search, we restrict resolvers of an unsupported condition $\alpha$ to the assertions in $\potsup_\alpha$. This removes infeasible resolvers and thus reduces the number of branches in the search tree.

\begin{figure}
  \centering
  \begin{tikzpicture}
    \pic[from=d_0,to=d_1,duration=10,assertionID=\beta] (beta) {change};
    \pic[from=d_2,to=d_3,duration=10,assertionID=\alpha,above right=1.5cm and 5.5cm of beta-c] (alpha) {change};
    \pic[from=d_4,to=d_5,duration=10,assertionID=\gamma,below right=1.5 and 4cm of beta-c] (gamma) {change};
    \pic[from=d_5,to=d_6,duration=10,assertionID=\mu,right=3cm of gamma-c] (mu) {change};
    
    \pic[value=d_1,duration=2,assertionID=\rho, right=11cm of beta-c] (rho) {persistence};

    \draw[causal link] (gamma-to) to[out=0,in=180] (mu-from);
    \def \sm {\smaller[2]}
    \path[->,color=gray,dashed,above,sloped,midway,bend left=15,thick] 
    (beta-to) edge node[midway,above] {\sm before} (alpha-from)
    (beta-to) edge[bend right=15] node[below] {\sm before} (gamma-from)
    (alpha-to) edge[] node {\sm before} (rho-from)
    (mu-to) edge[bend right=15] node[below] {\sm before} (rho-from);

  \end{tikzpicture}
  \caption{Partial view of a causal network of a state variable $\sv{loc}(r)$ with 4 change assertions $(\beta,\alpha,\gamma,\mu)$ and one persistence condition $\rho$, all on the same state variable $loc(r)$.
    There is a causal link from $\gamma$ to $\mu$, $\beta$ is temporally constrained to be before $\alpha$ and $\gamma$, and $\rho$ is temporally constrained to be after $\alpha$ and $\mu$.
    We further suppose $\beta$ to be initially supported.
    This causal network is to be considered together with the DTG of \autoref{fig:dtg-loc}.
  }
  \label{fig:ex-causal-network}
\end{figure}
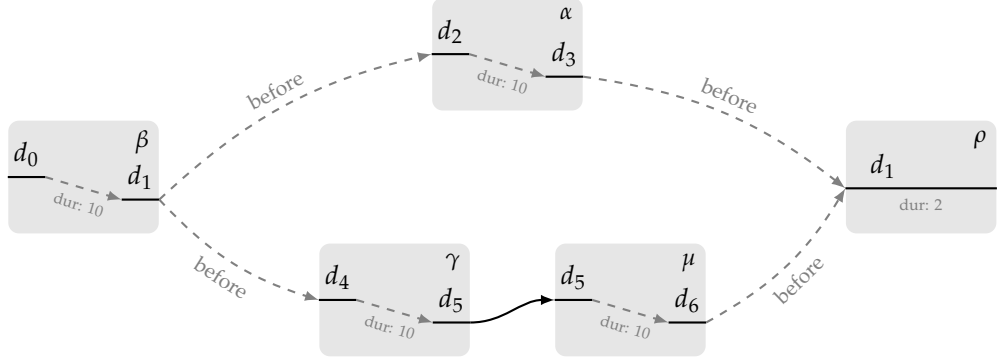

\subsubsection{Deriving Constraints from Potential Supporters}
\label{sec:causal-network-constraints}

We now consider what temporal constraints can be inferred from the necessary evolution of a causal network.
In order to keep the explanations and notations concise, we first assume that actions do not contain any \emph{a priori} supported assertions.

Given this assumption, an unsupported assertion $\alpha$ must eventually be linked to a change assertion $\beta \in \potsup_\alpha$ by a chain of causal links $\beta \rightarrow\dots\rightarrow \alpha$.
The length of the causal chain depends on the change assertions needed to go from the value produced by $\beta$ (denoted $\function{eff}(\beta)$) to the value needed by $\alpha$ (denoted $\function{cond}(\alpha)$).
Therefore, $start(\alpha)$ must be after $end(\beta)$ with a delay depending on the times of $\function{eff}(\beta)$ and $\function{cond}(\alpha)$.
More formally, this requirement is expressed by the following inequality,

\[
dist_{STN}(\timeorigin,start(\alpha)) \geq \min_{\beta \in \potsup_\alpha} \; dist_{STN}(\timeorigin, end(\beta)) + dist_{DTG}(\function{eff}(\beta),\function{cond}(\alpha))%
\]

where $dist_{STN}(\timeorigin,t)$ is the minimal delay in the STN between the origin of time \timeorigin and the time point $t$ and $dist_{DTG}(x,y)$ represents the length of the minimal path in the DTG to go from any instantiation of $x$ to any instantiation of $y$.
If this inequality does not hold, it is enforced by setting the earliest time of $start(\alpha)$ to be greater or equal to the right side of the inequality.

In the case where an assertion $\alpha$ has a single possible supporter $\beta$, one can devise a more specific version that does not use a triangular inequality:

\[
dist_{STN}(end(\beta),start(\alpha)) \geq dist_{DTG}(\function{eff}(\beta),\function{cond}(\alpha))
\]

Again this inequality is enforced by adding in the STN a minimal delay constraint between $end(\beta)$ and $start(\alpha)$.

In the case where some actions contain an \emph{a priori} supported assertion and that $\alpha$ can be achieved using one such assertion (i.e. $\distdtg(\anydtg,\function{cond}(\alpha)) \neq \infty$), the above rules are generalized by considering a virtual \emph{possible supporter} that could support it at time $\distdtg(\anydtg,\function{cond}(\alpha))$.

\begin{example}
  Let us now consider what applying those rules on the causal network of \autoref{fig:ex-causal-network} would allow us to infer.
  Assuming that $dist_{STN}(\timeorigin,st_\beta) = 0$, we would infer the following temporal constraints:
  \begin{align*}
    dist_{STN}(\timeorigin,start(\alpha)) & \geq \min~ \set{ dist_{STN}(\timeorigin,end(\beta)) + 10,~ \diststn(\timeorigin,end(\mu))+30 } \\
                                      & \geq 20 \\
    \diststn(\timeorigin,start(\gamma)) & \geq \min~ \set{ \diststn(\timeorigin,end(\beta)) + 30,~ \diststn(\timeorigin,end(\alpha)) + 10 } \\
                                      & \geq 40 \\
    \diststn(\timeorigin,start(\rho)) & \geq \min ~\set{ \diststn(\timeorigin,end(\alpha)) + 50,~ \diststn(\timeorigin,end(\mu)) + 20 } \\
                                      & \geq 80
  \end{align*}

  The important catch is the detection that the persistence condition $\rho = \tuple{loc(r)=d_1}$ cannot be satisfied before time 80.
  This is because the planner has already made commitments to other change assertions, between when the value $d_1$ is first achieved by $\beta$ and the moment it is required by $\rho$.\EndOfTheoremMark
\end{example}

\subsubsection{Estimating the number of additional assertions needed for a valid causal chain}
\label{sec:min-causal-chain}

We further use the causal network as part of heuristic evaluation in order to estimate how many additional assertions are needed in order to support an unsupported assertion.

For each unsupported assertion $\alpha$, the key idea is to find a chain of causal links going from an \emph{a priori} supported assertion to $\alpha$. 
We seek a minimal chain: filling out the missing parts should result in as few additional assertions as possible.
\autoref{fig:ex-causal-chain} gives an example of the minimal causal chain needed to support the persistence condition $\rho$ from \autoref{fig:ex-causal-network}.
Building such a causal chain requires the addition of 4 change assertions, resulting in as many new open goals to be solved.

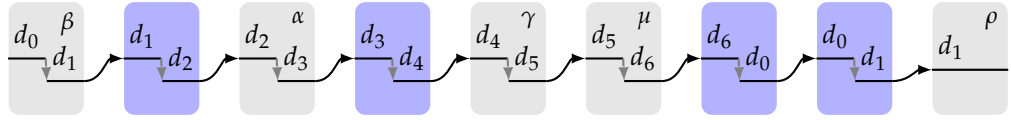
\begin{figure}[h]
  \centering
  \begin{tikzpicture}
    \tikzstyle{base} = [width=1cm]
    \tikzstyle{tmp} = [base,background=blue!30]
    \pic[base,from=d_0,to=d_1,assertionID=\beta] (beta) {change};
    \pic[tmp,from=d_1,to=d_2,right=1.4cm of beta-c] (tmp1) {change};
    \pic[base,from=d_2,to=d_3,assertionID=\alpha,right=1.4cm of tmp1-c] (alpha) {change};
    \pic[tmp,from=d_3,to=d_4,right=1.4cm of alpha-c] (tmp2) {change};
    \pic[base,from=d_4,to=d_5,assertionID=\gamma,right=1.4cm of tmp2-c] (gamma) {change};
    \pic[base,from=d_5,to=d_6,assertionID=\mu,right=1.4cm of gamma-c] (mu) {change};
    \pic[tmp,from=d_6,to=d_0,right=1.4cm of mu-c] (tmp3) {change};
    \pic[tmp,from=d_0,to=d_1,right=1.4cm of tmp3-c] (tmp4) {change};
    \pic[base,value=d_1,assertionID=\rho, right=1.4cm of tmp4-c] (rho) {persistence};

    \draw[causal link,out=0,in=180] (beta-to) to (tmp1-from);
    \draw[causal link,out=0,in=180] (tmp1-to) to (alpha-from);
    \draw[causal link,out=0,in=180] (alpha-to) to (tmp2-from);
    \draw[causal link,out=0,in=180] (tmp2-to) to (gamma-from);
    \draw[causal link,out=0,in=180] (gamma-to) to (mu-from);
    \draw[causal link,out=0,in=180] (mu-to) to (tmp3-from);
    \draw[causal link,out=0,in=180] (tmp3-to) to (tmp4-from);
    \draw[causal link,out=0,in=180] (tmp4-to) to (rho-from);
  \end{tikzpicture}
  \caption{One possible causal chain to support the persistence condition $\rho = \tuple{loc(r)=d1}$. In blue are temporal assertions that would need to be added for the causal chain to be complete.}
  \label{fig:ex-causal-chain}
\end{figure}

\def \hc {hc}

We now describe how we compute the heuristic value, $\hc(\alpha)$, an estimation of the number of additional assertions that are needed to build a complete causal chain to the open goal $\alpha$.
We first remark that our lifted representation means that there are multiple candidates for the instantiation of the condition of $\alpha$.
We thus introduce $\hc(f,\alpha)$ as the cost of building the causal chain to $\alpha$ if its condition is the ground fluent $f$.
Since the planner has the choice in the instantiation of variables, we consider the cost of building the causal chain to $\alpha$ to be the minimum of the cost of building it with any possible instantiation of its condition:

\begin{equation}
  \hc(\alpha) = \min_{\fluent{sv}{v} \in dom(cond(\alpha))} hc(\fluent{sv}{v},\alpha) \label{eq:hc-global}
\end{equation}

where $\hc(f,\alpha)$ is the cost of achieving $\alpha$ if its condition is $f$ and $dom(cond(\alpha))$ is the set of possible instantiations of the condition of $\alpha$.

We define $\hc(f,\alpha)$, the cost of achieving the ground condition $f = \fluent{sv}{v}$ of an assertion $\alpha$.
The cost of building a causal chain is expressed recursively as the cost of having its rightmost link plus the cost of building a chain up to this rightmost link.

\begin{equation}
  \hc(\fluent{sv}{v},\alpha) =
    \begin{cases}
      0 & \text{if $\alpha$ is a priori supported} \\[20pt]
      \displaystyle \min_{\tuple{sv=v'\mapsto v)} \in dom(\beta)} \hc(\fluent{sv}{v'},\beta) & \text{if } \exists \text{ a causal link } \beta \cl \alpha \\[20pt]
      \min 
        \begin{cases}
          \displaystyle  \min_{\gamma \in \potsup_\alpha,~ \tuple{sv=v'\mapsto v} \in dom(\gamma)} \hc(\fluent{sv}{v'},\gamma) & \\[13pt]%
          \displaystyle \min_{e=\tuple{v' \rightarrow v} \in DTG(sv)} c(e) + \hc(\fluent{sv}{v'},\alpha) &
        \end{cases} & \text{otherwise} \\
    \end{cases}
  \label{eq:hc-with-fluent}
\end{equation}

Intuitively, there is no additional cost if $\alpha$ is \emph{a priori} supported because there is no need for any causal support (i.e. $\hc(\cdot,\alpha) = 0$).

If $\alpha$ is supported by an incoming causal link $\beta \cl \alpha$, this causal link is necessarily the last link of the causal chain to $\alpha$ (for instance in \autoref{fig:ex-causal-network}, the last link of any causal chain to $\mu$ is the existing causal link $\gamma \cl \mu$).
Thus, the cost of achieving the condition $f$ of $\alpha$ is the cost of achieving the condition $f'$ of $\beta$, where $f'$ is an instantiation of the condition of $\beta$ such that $\beta$ produces $f$.

If $\alpha$ is in neither of these cases, we are left with two possibilities for the last link of its causal chain.
First, if $\alpha$ has a possible supporter $\gamma \in \potsup_\alpha$ and $\gamma$ can be instantiated to provide $f$, then a possibility is to have a causal link $\gamma \cl \alpha$.
In this case, the cost is that of achieving $\gamma$ with such an instantiation.
Second, there might be a possible action producing $f = \fluent{sv}{v}$ by adding a change assertion \change{t,t'}{sv}{v'}{v}.
Such a change would appear as an edge $e$ in the DTG of $sv$.
Since adding this link in the causal chain requires inserting a new action in the plan, we associate a cost $c(e)$ to this operation.
This cost is set to the number of unsupported assertions in the introduced action, e.g., for the action \act{move} of \autoref{fig:move-action} this cost would be 3 since its insertion would result in 3 new assertions.
Furthermore, we still need to build the causal chain to achieve the value $\fluent{sv}{v'}$ for $\alpha$.

\begin{example}
  Considering the causal network of \autoref{fig:ex-causal-network}, the equation below gives the estimated cost of building a causal chain to $\alpha$ (i.e. $\hc(\alpha)$).
  Since all variables in $\alpha$ are already bound, there is a single possibility for instantiating its condition (first line).
  From there, the only possibility to provide the value $\constant{d2}$ is to insert an additional $\act{move}(r,\constant{d1},\constant{d2})$ action, resulting in an additional cost of 3 (second line).
  To provide the value $\constant{d1}$, we can either have a causal link from $\beta$ or add another action $\act{move}(r,\constant{d0},\constant{d1})$ again with an additional cost of 3 (third line).
  Since $\beta$ is initially supported, it induces no extra cost and we can conclude that $\hc(\alpha) = 3$: building a complete causal chain to support $\alpha$ would require the insertion of three new assertions into the partial plan.
  \begin{align*}
    hc(\alpha) &= \hc(\fluent{loc(r)}{\constant{d2}}, \alpha) \\
               &= 3 + \hc(\fluent{loc(r)}{\constant{d1}}, \alpha) \\
               &= 3 + \min~ \{~ \hc(\fluent{loc(r)}{\constant{d0}}, \beta),~~ 3 + \hc(\fluent{loc(r)}{\constant{d0}}, \alpha) ~\} \\
               &= 3 + \min~\{~0, \dots~\} \\ 
               &= 3
  \end{align*}
\end{example}

In practice for computing $hc(\alpha)$, we use a distance computation in an equivalent graph where each node is a pair $\tuple{f,\alpha}$, $f$ being a fluent and $\alpha$ an assertion in the causal network.
An example of such a graph for the causal network of \autoref{fig:ex-causal-network} is given in \autoref{fig:ex-causal-chains-graph}.
Edges in the graph represent the different possible transitions defined in equation~\eqref{eq:hc-with-fluent}.
We distinguish a causal link from an existing assertion (green with a cost of 0) from a causal link from an additional assertion (dashed red with a non-zero cost).
We consider two special kinds of nodes: nodes representing \emph{a priori} supported conditions (empty circles) and nodes representing unsupported conditions (full circles).
Finding a minimal causal chain to an unsupported assertion $\alpha$ is equivalent to finding the shortest path in the graph from any \emph{a priori} supported node to any node representing an instantiation of $\alpha$.

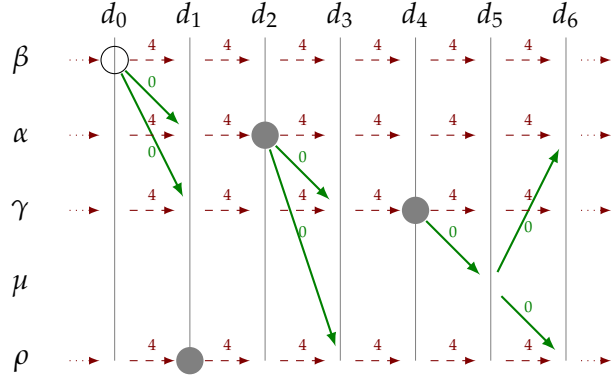
\begin{figure}[!tb]
  \centering
  \begin{tikzpicture}[
    ord/.style={anchor=east},
    short edge/.style={shorten >=0.2cm,shorten <=0.2cm},
    cnedge/.style={color=darkgreen,->,>=latex,short edge,thick},
    dtgedge/.style={dashed,color=darkred,->,>=latex,short edge},
    open goal/.style={circle,fill=gray},
    supported/.style={circle,draw}]

    \newcommand*{\xMin}{0}%
    \newcommand*{\xMax}{6}%
    \newcommand*{\yMin}{0}%
    \newcommand*{\yMax}{4.3}%
    \def \ybeta  {4}
    \def \yalpha {3}
    \def \ygamma {2}
    \def \ymu {1}
    \def \yrho {0}
    
    \newcommand{\drawdtgedges}[2]{
      \foreach \i in {0,...,6} {
          \ifthenelse{\i=6} {
            \path[dtgedge,dotted] (6,#1) edge (6.8,#1)
            (-0.8,#1) edge (0,#1);
          }{
            \draw[dtgedge] (\i,#1) -- node[above] {\tiny 4} ($ (\i + 1 ,#1) $);
          }
      }
    }

    \foreach \i in {\xMin,...,\xMax} {
      \draw [very thin,gray] (\i,\yMin) -- (\i,\yMax)  node [above,black] at (\i,\yMax) {$d_\i$};
    }
    
    \node[ord] at (-1,\ybeta) {$\beta$};
    \node[ord] at (-1,\yalpha) {$\alpha$};
    \node[ord] at (-1,\ygamma) {$\gamma$};
    \node[ord] at (-1,\ymu) {$\mu$};
    \node[ord] at (-1,\yrho) {$\rho$};
    
  \path[cnedge]
    (0,\ybeta) edge node[above] {\tiny 0} (1,\yalpha)
    (0,\ybeta) edge node[below] {\tiny 0} (1,\ygamma)
    (2,\yalpha) edge node[above]  {\tiny 0} (3,\ygamma)
    (2,\yalpha) edge node[above] {\tiny 0} (3,\yrho)
    (4,\ygamma) edge node[above] {\tiny 0} (5,\ymu)
    (5,\ymu) edge node[below] {\tiny 0} (6,\yalpha)
    (5,\ymu) edge node[above] {\tiny 0} (6,\yrho);

    \drawdtgedges{\ybeta}{0}
    \drawdtgedges{\yalpha}{2}
    \drawdtgedges{\ygamma}{4}
    \drawdtgedges{\yrho}{1}
    \node[supported] at (0,\ybeta) {};
    \node[open goal] at (2,\yalpha) {};
    \node[open goal] at (4,\ygamma) {};
    \node[open goal] at (1,\yrho) {};

  \end{tikzpicture}
  \caption{Virtual graph used for computing the minimal causal chain of the assertions of \autoref{fig:ex-causal-network}.}
  \label{fig:ex-causal-chains-graph}
\end{figure}

In the example of \autoref{fig:ex-causal-chains-graph}, we can easily find the causal chain of \autoref{fig:ex-causal-chain} by looking for the shortest path from $\tuple{d_0,\beta}$ to $\tuple{d1,\rho}$.
This shortest path takes 4 red edges for a final cost of 12, allowing us to conclude that $\hc(\rho) = 12$. This corresponds to the causal chains shown in \autoref{fig:ex-causal-chain}.

In practice, the computation of $\hc(\alpha)$ is done by a backward Dijkstra search: initializing the priority queue with nodes $\{ \tuple{f,\alpha} | f \in dom(cond(\alpha)) \}$.
Search progresses by selecting the node with the least cost in the priority queue and adding its direct ancestors to the queue with an updated cost.
Search continues until a node $\tuple{f',\beta}$, where $\beta$ is \emph{a priori} supported, is extracted from the queue.
The cost of this node gives the cost of the minimal causal chain to $\alpha$ (i.e. $\hc(\alpha)$).

\subsection{Search Strategies}
\label{sec:search-strategies}

The search algorithm is responsible for choosing which node of the search tree to expand in order to quickly find a solution of good quality.
The quality of the solution and the time spent finding it are often conflicting objectives.
By default in FAPE, the priority is given to the latter and plan quality is left as a secondary objective, used as a tie breaking criteria.

As is usually the case in plan-space planning, our search procedure (\autoref{alg:plan}) requires two choices to be made at each search iteration.
The first one is the nondeterministic choice of which partial plan to consider next that will define the order in which the search space is explored.
The second one is the choice of a flaw to be fixed in the selected partial plan.
As all flaws must eventually be fixed for a partial plan to become a solution, this choice is not a backtracking point but will have an impact on the shape of the search tree.
A search strategy is composed of two schemes that dictate the choices made in those two cases.

As can be expected, a good strategy is not universal as it must take into account many specifics of the problem at hand.
Because we support such a wide range of planning problems, from generative to HTN problems, we define two strategies.
The first one aims at being very general and having good performance on a wide range of problems while the second is specifically tailored for fully hierarchical problems.

\subsubsection{General Search Strategy}
\label{sec:general-search-strategy}

\paragraph{Plan selection.}

\def \qall {\ensuremath{Q_{All}}\xspace}
\def \qchildren {\ensuremath{Q_{Chi}}\xspace}

Our plan selection strategy is conceptually based on $A\epsilon$ \cite{Ghallab83} in that it contains two queues both sorted by the same priority function \funcf. 
The first one \qall contains all partial plans that have been generated and not expanded.
The second one \qchildren is a subset of \qall that is limited to children of the last expanded node.
Those two queues serve different purposes that can be seen as diversification versus intensification: the planner chooses the partial plan to consider next either as the globally most promising according to $\funcf$ or commits to its previous choice and tries to further advance the last chosen plan.

The choice of the queue to use is governed by a parameter $\epsilon$ and is done as follows:
if ~$\minimum_{\chronicle\in\qall} \funcf(\chronicle) < (1+\epsilon) \times \minimum_{\chronicle\in\qchildren} \funcf(\chronicle)$~ then the next partial plan is the one with the lowest $\funcf$ value in \qall. Otherwise it is the one with the lowest $\funcf$ value in \qchildren.
In practice, it means the search is restricted to the children of the last expanded node until they are significantly worse than the globally best partial plan, where the threshold of ``significantly'' is given by $\epsilon$.
The objective of this technique is to compensate for the inaccuracies in the definition of $\funcf$.

The priority function $\funcf(\chronicle)$ is defined as the sum of the following values:
\begin{enumerate}
\item the number of assertions in $\assertions_\chronicle$. This helps in estimating the search effort already done. It can be seen as a normalized version of the number of actions in the plan as other parts of $\funcf$ depend on the number of assertions rather than the number of actions.
\item the number of unrefined task in $\plan_\chronicle$. This conservatively estimates the search effort left due to unrefined tasks.
\item the number of unsupported assertions in $\assertions_\chronicle$.
\item the number of assertions involved in at least one threat. This is a conservative estimation of the search effort left due to threats. 
  We do not consider the number of threats itself as it can be, in the worst case, quadratic in the number of assertions in the plan, resulting in a very important impact on the value of $\funcf$.
  Furthermore the addition of a single temporal constraint is likely to solve many threats at once.
\item the expected number of assertions that must be added to the partial plan. This is computed as the sum, for every unsupported assertion $\alpha \in \assertions_\chronicle$, of $\hc(\alpha)$ where $\hc$ gives the number of assertions needed to reach the minimal causal chain, as defined in \autoref{sec:min-causal-chain}. 
\end{enumerate}

This definition of the priority function $\funcf$ can be seen as composed of the usual $g+h$ parts where $g$ is given by the first item while $h$ is given by the following 4.
More specifically for $h$, the items 2 to 4 give the search effort directly visible as flaws in the partial plan while the fifth item provides a heuristic estimation of the search effort required.

\paragraph{Flaw Selection.}

The flaw selection strategy aims at sorting flaws in order to select the one whose resolution would the more beneficial (or least detrimental) to search.
Choosing which flaw to select next is a tricky question as it mainly permits an organization of the search space.
A similar use case occurs in constraint satisfaction problems as the choice of the next variable to be given a value.
A very efficient heuristic for this choice is to choose the one with the smallest domain because it results in the smallest branching factor for the early stages of search.
As a result, the search space of the constraint solver can be kept small.
Things are more complex in planning as a plan has more types of components whose interactions are hard to take into account.
Our strategy is given by the ordered list below. If the first rule favors one flaw over another, then only this one is kept. Otherwise the next rule is used to break the tie.
\begin{enumerate}
\item Prefer a flaw that has a single resolver. A single resolver means that there is a single option to choose from. Thus, no mistake can be made in applying it.
\item Prefer unrefined tasks over other types of flaws.
  This priority has the advantage of giving the priority to task refinement, thus quickly reaching the point where the plan contains most of its actions with no unrefined tasks.
  This is useful as unrefined tasks are weakly accounted for in our plan selection strategy. Getting rid of those early means we will quickly get to a point where our $\funcf$ function is more informative. 
  \item Prefer unsupported assertions over other types of flaw. Given two unsupported assertions, choose the one on the maximally abstract state variables as defined by \textcite{Knoblock1994}.
\item Finally, prefer the flaw with the fewest resolvers.
\end{enumerate} 

Furthermore, it should be noted that if there is a flaw with no resolver, then this partial plan is necessarily a dead-end and the planner can proceed with the next partial plan.

\subsubsection{Forward Hierarchical Search Strategy}
\label{sec:hierarchical-search-strategy}

Our second search strategy is dedicated to HTN problems and is conceptually similar to the forward search techniques of HTN planners like SHOP2 or SIADEX.
The key idea is to hand back some control to the domain designer about which plans will be explored first.
For this reason, the planner will try the different refinements in the order defined in the domain and commit to them until they are proved unsound.
The planner also follow an early commitment strategy through its flaw selection strategy, which is meant to detect early any inconsistencies in the current plan.

\paragraph{Plan selection.}
Plan selection is done in a depth-first manner with chronological backtracking.
When a node is expanded, the choice of the next partial plan to expand among its children is as follows:
\begin{itemize}
\item if the last resolved flaw was an unrefined task, meaning that each possible partial plan matches a possible action for refining it, then give priority to the action defined first in the domain.
\item otherwise sort the candidate children by the priority function of the general search strategy.
\end{itemize}

This prioritization allows the domain designer to force the planner to explore plans in a predefined order.
The rest of the search decisions, such as the choice of how to resolve a threat, is left to our general search strategy.

\paragraph{Flaw selection.}
Similarly to the general search strategy, we define the flaw selection strategy as a sequence of rules whose application is meant to choose the next flaw to resolve: 
\begin{enumerate}
\item Prefer flaws with a single resolver.
\item Prefer the flaw that has the earliest interaction time. The interaction time of an unsupported assertion (resp. an unrefined task) is defined as
the earliest time of its start time point (e.g. $\diststn(\timeorigin,t_s)$ for an unsupported persistence $\persistencetuple{t_s,t_e}{sv}{v}$).
The interaction time of a threat is taken as the maximum of the interaction time of both assertions it involves.
\item Prefer threats over other types of flaws.
\item Prefer unrefined tasks over other types of flaws. When comparing two unrefined tasks, prioritize the one that was introduced first.
\item Like for the general strategy, prefer unsupported assertions and give priority to the maximally abstract ones.
\item Prefer flaws with the fewest resolvers.
\end{enumerate}

The general idea behind those rules is to bias the planner into a forward search by dealing with unsupported assertions and unrefined tasks that appear early in the partial plan.
Indeed, solving them will typically result in the introduction of causal chains involving assertions from the initial state.
The construction of those causal chains forces the instantiation of object variables involved in them and permits an easier verification of the conditions on actions.

While this is the general idea, it does not prevent the planner from inserting some of the later actions of the plan early during search by solving a flaw with a single resolver.

\section{Empirical Evaluation}
\label{sec:chap2-experiments}

\newcommand{\ipc}[1]{{\smaller \sc (ipc#1)}\xspace}
\newcommand{\laas}{{\smaller \sc (laas)}\xspace}
\newcommand{\genpb}{{-\smaller\sc Gen}}
\newcommand{\fullhierpb}{{-\smaller\sc FullHier}}
\newcommand{\parthierpb}{{-\smaller\sc PartHier}}

We have presented an algorithm able to plan both in a generative and a  hierarchical fashion.
Our approach is complemented with a number of techniques intended to improve the efficiency of the planner by \i inferring constraints to cut branches of the search tree, and \ii guiding the planner to efficiently explore its search space.

In this section, we first study how each of the techniques we have described contribute to the overall efficiency of the system.
We then compare FAPE with state-of-the-art temporal planners from the International Planning Competition (IPC).
We show FAPE to be competitive in a fully generative and domain-independent setting and  that the addition of hierarchical knowledge further improves its performance.

\subsection{Evaluation of the Different Components of the Planner}

\subsubsection{Evaluation of reachability analysis}
\label{sec:evaluation-reachability}

We start by evaluating our proposed reachability analysis independently of other techniques.
The motivation to do so comes from the fact that \i it generalizes the delete-free analysis done by classical planners to a temporal setting and \ii the other adaptation of this technique by \textcite{Coles2008} is easily represented in our framework by considering an additional relaxation.
We can hence make a direct comparison between them.

\paragraph{Tested Configurations.}
We distinguish 5 configurations of the planner depending on how far it pushes the reachability analysis:
\begin{itemize}
\item \hfull is the configuration where no limitation is put on the number of iterations for reachability analysis.
\item $R_k$, with $k=5$ and $k=1$, denotes the configuration where the number of iterations is limited to $k$.
  This makes the algorithm strongly polynomial and reduces the overhead when many iterations are needed to converge.
  On the other hand, the algorithm might incorrectly label unreachable actions as reachable and fluents/actions will typically be found to have more optimistic earliest appearance times.
\item \hplus denotes the configuration where all after-conditions are ignored.
  In practice, it means that the propagation will stop right after the first Dijkstra propagation (in the middle of the first iteration).
  This configuration is equivalent, for our more expressive temporal model, to the reachability analysis performed by {\sc popf} and related planners \cite{DBLP:conf/aips/ColesCFL10,Coles2008,Benton2012,Coles:2012una}
\item $\emptyset$ denotes the configuration where no reachability analysis is performed.
  In this case, the planner does not ground the problem, which reduces its overhead.
\end{itemize}
To allow for an unbiased comparison, the reasoning on causal networks (\autoref{sec:causal-network}) is deactivated in all configurations as it requires reasoning on the ground problem (which is inaccessible for the last configuration).

\paragraph{Test Domains.}
We evaluate our reachability analysis technique on several temporal domains with and without hierarchical features, the former involving many interdependencies between high-level actions and their subactions.
The  \emph{satellite}, \emph{rovers}, \emph{logistics}, \emph{blocks}, \emph{driverlog} and \emph{hiking} domains are the eponymous domains from the International Planning Competition and are all temporally simple (i.e. they have no after-conditions).
The domain definitions of those problems have been manually translated into ANML and their problems were automatically translated by domain dependent parsers.

The \emph{handover} domain is a robotics problem presented by \textcite{dvorak-2014-planrob}, the \emph{dock-worker} domain is the  dock worker domain of \textcite{Ghallab2004}.
\emph{race} is a robotics domain adapted from CHIMP \cite{Stock2015}, where a waiter-robot must serve clients. It notably features navigation constraints expressed using hierarchical features and deadlines specifying when a client must be served.
\emph{springdoor} is another robotics problem where the robots must move objects between several places with closed doors.
Opening a door requires complex interactions between several actions (turning the knob, pushing the door and releasing the knob) which can be performed by the robot that must pass through the door if it carries nothing, or by another robot if its hands are full.
\emph{baking} is a domain presented by \textcite{Cushing2007a} of baking pottery in kilns where the action of baking must be concurrent with an action showing that the kiln is switched on.

Hierarchical versions of the domains have their names appended with `\fullhierpb' or '\parthierpb' if they are fully or partly hierarchical respectively.

\begin{table}[tb!p]
  \centering
  \begin{tabular}{|l||c|c|c|c|c|} \hline 
     & \providecommand{\hfull}{$h_\infty$}\hfull & \providecommand{\hfive}{$h_5$}\hfive & \providecommand{\hone}{$h_1$}\hone & \providecommand{\hplus}{$h^+$}\hplus & \providecommand{\hnone}{$\emptyset$}\hnone \\ 
  \hline
\ipc4 airport\genpb & \textbf{5} & \textbf{5} & \textbf{5} & \textbf{5} & \textbf{5} \\ 
  \ipc4 airport-tw\genpb & \textbf{6} & \textbf{6} & \textbf{6} & \textbf{6} & \textbf{6} \\ 
  \ipc2 blocks\genpb & 24 & 24 & 24 & 24 & \textbf{26} \\ 
  \laas dock-worker\genpb & \textbf{9} & \textbf{9} & \textbf{9} & \textbf{9} & \textbf{9} \\ 
  \ipc8 driverlog\genpb & \textbf{1} & \textbf{1} & \textbf{1} & \textbf{1} & \textbf{1} \\ 
  \laas handover\genpb & \textbf{4} & 3 & 3 & 3 & 3 \\ 
  \ipc2 logistics\genpb & \textbf{7} & \textbf{7} & \textbf{7} & \textbf{7} & 4 \\ 
  \ipc4 pipesworld-dl\genpb & 5 & 5 & 5 & 5 & \textbf{6} \\ 
  \ipc5 rovers\genpb & 34 & 34 & 34 & 34 & \textbf{38} \\ 
  \ipc4 satellite-tw\genpb & \textbf{9} & \textbf{9} & \textbf{9} & \textbf{9} & 8 \\ 
  \ipc8 satellite\genpb & 16 & 16 & 16 & 16 & \textbf{17} \\ 
  \textbf{Total Generative} & 120 & 119 & 119 & 119 & \textbf{123} \\[5pt]
  \ipc2 blocks\fullhierpb & \textbf{10} & \textbf{10} & 9 & 4 & 9 \\ 
  \ipc2 blocks\parthierpb & 28 & 28 & 28 & 30 & \textbf{32} \\ 
  \laas dock-worker\fullhierpb & \textbf{22} & \textbf{22} & \textbf{22} & 20 & 20 \\ 
  \laas dock-worker\parthierpb & \textbf{12} & \textbf{12} & \textbf{12} & \textbf{12} & \textbf{12} \\ 
  \laas handover\fullhierpb & \textbf{10} & 1 & 1 & 1 & 1 \\ 
  \ipc8 hiking\fullhierpb & \textbf{20} & 13 & 12 & 12 & 12 \\ 
  \ipc2 logistics\parthierpb & \textbf{28} & \textbf{28} & \textbf{28} & 27 & 27 \\ 
  \laas race\fullhierpb & \textbf{13} & \textbf{13} & \textbf{13} & \textbf{13} & \textbf{13} \\ 
  \ipc8 satellite\parthierpb & \textbf{18} & \textbf{18} & \textbf{18} & \textbf{18} & \textbf{18} \\ 
  \laas baking\parthierpb & \textbf{6} & \textbf{6} & \textbf{6} & \textbf{6} & \textbf{6} \\ 
  \ipc8 turnandopen\fullhierpb & \textbf{8} & \textbf{8} & \textbf{8} & 0 & 0 \\ 
  \textbf{Total Hierarchical} & \textbf{175} & 159 & 157 & 143 & 150 \\ 
  \textbf{Total} & \textbf{295} & 278 & 276 & 262 & 273 \\ 
   \hline

  \end{tabular} 
  \caption{Number of solved problems for various domains with a 30 minutes timeout.  The best results are shown in bold.}
  \label{tab:solved-per-dom}
\end{table} 

\paragraph{Results and Discussion.}
Table~\ref{tab:solved-per-dom} and Figure~\ref{fig:solved-over-time} present the number of problems solved using different reachability models.
On generative problems (first half of the table), only a few domains benefit from reachability analysis: \emph{handover}, \emph{logistics} and \emph{satellite-tw}.
On other generative domains, the cost of grounding the problem is slightly higher or on-par with the benefits of using reachability analysis.
As expected, all configurations using reachability analysis perform identically on temporally simple problems.

This state of affairs changes on hierarchical problems as they contain many interdependencies.
In those domains, \hfull outperforms the other configurations: solving the highest number of 
problems on all but one domain. \hfive and \hone are respectively second and third best performers while
\hplus does not provide significant pruning of the search space; the computational overhead makes it perform slightly worse than no 
reachability checks (denoted by $\emptyset$).

\begin{figure}[htbp!]
  \centering
  \input{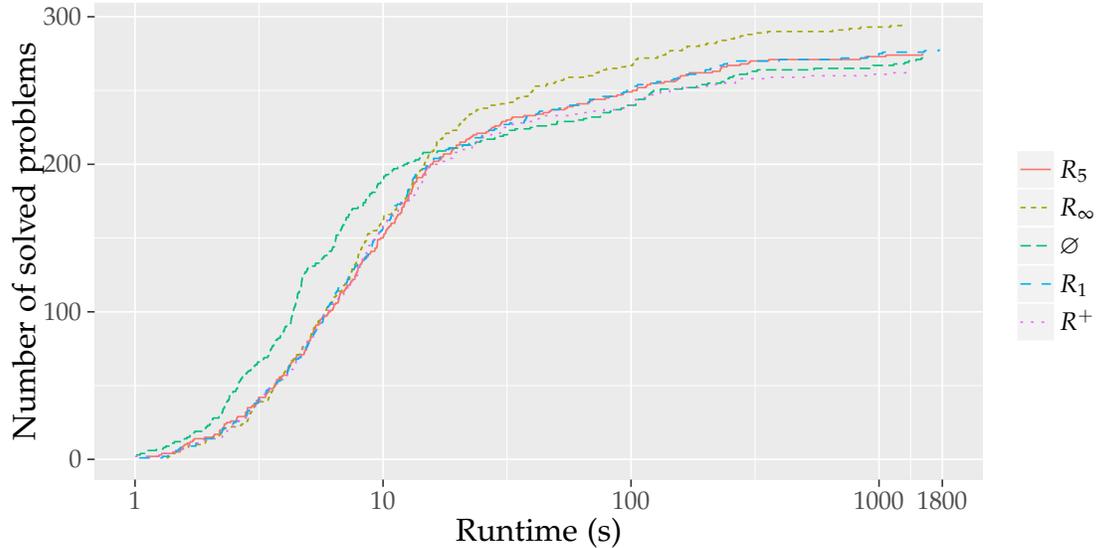}
  \caption{Number of solved tasks by each configuration within a given time amount.}
  \label{fig:solved-over-time}
\end{figure}

\begin{table}[tbh!]
  \centering
  \begin{tabular}{|l||c|c|c|c||c|} \hline 
     & \providecommand{\hfull}{$h_\infty$}\hfull & \providecommand{\hfive}{$h_5$}\hfive & \providecommand{\hone}{$h_1$}\hone & \providecommand{\hplus}{$h^+$}\hplus & \# iter \hfull \\ 
  \hline
\ipc4 airport\genpb & \textbf{13.6} & \textbf{13.6} & \textbf{13.6} & \textbf{13.6} & 1 \\ 
  \ipc4 airport-tw\genpb & \textbf{53.4} & \textbf{53.4} & \textbf{53.4} & \textbf{53.4} & 1 \\ 
  \ipc2 blocks\genpb & 0 & 0 & 0 & 0 & {1} \\ 
  \laas dock-worker\genpb & \textbf{55.3} & \textbf{55.3} & \textbf{55.3} & 0 & 6.0 \\ 
  \ipc8 driverlog\genpb & \textbf{41.8} & \textbf{41.8} & \textbf{41.8} & \textbf{41.8} & 1 \\ 
  \laas handover\genpb & \textbf{88.8} & 87.0 & 87.0 & 55.6 & 29.0 \\ 
  \ipc2 logistics\genpb & \textbf{39.4} & \textbf{39.4} & \textbf{39.4} & \textbf{39.4} & 1 \\ 
  \ipc4 pipesworld-dl\genpb & 0 & 0 & 0 & 0 & {1} \\ 
  \ipc5 rovers\genpb & \textbf{45.7} & \textbf{45.7} & \textbf{45.7} & \textbf{45.7} & 1 \\ 
  \ipc4 satellite-tw\genpb & \textbf{1.3} & \textbf{1.3} & \textbf{1.3} & \textbf{1.3} & 1 \\ 
  \ipc8 satellite\genpb & 0 & 0 & 0 & 0 & {1} \\ \hline
  \ipc2 blocks\fullhierpb & \textbf{83.1} & \textbf{83.1} & \textbf{83.1} & 0 & 7.4 \\ 
  \ipc2 blocks\parthierpb & \textbf{79.9} & \textbf{79.9} & \textbf{79.9} & 0 & 2.0 \\ 
  \laas dock-worker\fullhierpb & \textbf{80.8} & 48.8 & 48.8 & 0 & 36.7 \\ 
  \laas dock-worker\parthierpb & \textbf{76.3} & 58.8 & 58.8 & 0 & 57.7 \\ 
  \laas handover\fullhierpb & \textbf{99.0} & 87.4 & 87.4 & 0 & 47.2 \\ 
  \ipc8 hiking\fullhierpb & \textbf{74.1} & \textbf{74.1} & \textbf{74.1} & 0 & 24.7 \\ 
  \ipc2 logistics\parthierpb & \textbf{95.7} & \textbf{95.7} & \textbf{95.7} & 10.4 & 2.8 \\ 
  \laas race\fullhierpb & \textbf{94.3} & \textbf{94.3} & \textbf{94.3} & 0 & 22.5 \\ 
  \ipc8 satellite\parthierpb & \textbf{15.3} & \textbf{15.3} & \textbf{15.3} & \textbf{15.3} & 2.0 \\ 
  \laas baking\parthierpb & \textbf{87.6} & \textbf{87.6} & \textbf{87.6} & 0 & 5.8 \\ 
   \hline

  \end{tabular} 
  \caption{
    Percentage of ground actions detected as unreachable from the initial state.
    For each problem instance, the percentage is obtained by comparing the number of ground actions detected as unreachable from the initial state to the original number of ground actions.
    Those values are then averaged over all instances of a domain.
    The last column gives the average number of iterations needed by \hfull to converge on its initial propagation.
  }
  \label{tab:percent-reachable-actions}
\end{table}

Table~\ref{tab:percent-reachable-actions} presents the percentage of actions detected as unreachable by different configurations.
As expected, \hfull, \hfive, \hone and \hplus perform identically on temporally simple problems.
However, \hplus is largely outperformed on all but one hierarchical domain.
The good performance of \hone with respect to \hplus shows that a single complete iteration is often sufficient to capture most of the problematic after-conditions.
However, on more complex problems such as \emph{dock-worker}, \emph{hiking\fullhierpb}, \emph{handover\fullhierpb}, more iterations are beneficial either in terms of detected unreachable actions or in terms of solved problems.
The initial propagation can take as much as 58 iterations for \hfull.
The subsequent propagations are typically faster as they are made incrementally.
As expected, a single iteration was sufficient to converge on all temporally simple problems.

\subsubsection{Evaluation of other components}
\label{sec:evaluation-other}

\def \cfape {\ensuremath{\mathcal{C}_{\textit{full}}}\xspace}
\def \cgen {\ensuremath{\mathcal{C}_{\textit{gen}}}\xspace}
\def \clift {\ensuremath{\mathcal{C}_{\textit{lift}}}\xspace}
\def \cnocn {\ensuremath{\mathcal{C}_{\neg \textit{CN}}}\xspace}
\def \cnodelcheck {\ensuremath{\mathcal{C}_{\neg \textit{DelCk}}}\xspace}
\def \cnodecvars {\ensuremath{\mathcal{C}_{\neg \textit{DecVars}}}\xspace}
\def \cnoae {\ensuremath{\mathcal{C}_{A^*}}\xspace}

\def \flat {{\smaller \textsc{Flat}}\xspace}
\def \parthier {{\smaller\textsc{PartHier}}\xspace}
\def \fullhier {{\smaller\textsc{FullHier}}\xspace}

We next evaluate how each of the techniques presented contributes to the overall efficiency of the system.
For this purpose, we consider 6 configurations of the planner each one having a specific technique deactivated.
We compare this against the full configuration of FAPE.
\begin{itemize}
\item \cfape is the full configuration with all the techniques previously discussed enabled. 
  It uses the \emph{general search strategy} (\autoref{sec:general-search-strategy}) by default and switches to the \emph{forward hierarchical strategy} (\autoref{sec:hierarchical-search-strategy})  when facing a fully hierarchical domain.
\item \cgen is the configuration where FAPE always uses the general search strategy.
  It differs from \cfape on fully hierarchical domains where the forward hierarchical search strategy would have been preferred.
\item \clift is a fully lifted version of FAPE.
  More specifically it does not use reachability analysis and thus all features that require grounding are deactivated (i.e. causal networks and instantiation/refinement variables).
  This configuration is essential to measure the benefits and penalties for grounding the problem.
\item \cnoae use $A^*$ in place of our variant of $A^\epsilon$ (\autoref{sec:general-search-strategy}).
\item \cnodelcheck does not check that there is a sufficient delay from the start of a task to the moment its possible effect is required to support an unsupported assertion.
  Hence it can result in additional resolvers being considered for supporting an unsupported assertion (\autoref{sec:unsupported-assertions}).
\item \cnodecvars does not use refinement variables to disregard the resolvers of unrefined tasks that involve unreachable actions (\autoref{sec:inst-supp-constr}).
\item \cnocn does not use the causal network to \i infer constraints, \ii prune impossible resolvers and \iii improve the heuristic with needed assertions (\autoref{sec:causal-network}).
\end{itemize}

It should be noted that our test configurations do not consider variations for the flaw ordering strategies.
We have tested other such strategies but they resulted in poor performance of the planner.
While in general other strategies can be efficient (e.g. many ground state-space planners use a ``threat first'' strategy), their efficiency is strongly coupled with the rest of the strategies used in the planner.
This factor can explain the dominance of the current strategy as it has evolved with the rest of the system.

\paragraph{Overview of results.}
We start by giving a broad overview of the results given in \autoref{tab:feature-results}.
The table contains the number of problems solved by each configuration with a 30 minutes timeout and a memory limit of 3GB.

\begin{table}[tbp]
  \centering
  \def \nb {~\smaller }
  \begin{tabular}{l c c c c c c c}
     & \providecommand{\cfape}{$C_{full}$}\cfape & \cgen & \providecommand{\clift}{$C_{lift}$}\clift & \cnoae & \cnodelcheck & \cnodecvars & \cnocn \\ 
  \hline
\ipc4 airport\genpb & \textbf{6} & \textbf{6} & 5 & \textbf{6} & \textbf{6} & \textbf{6} & 5 \\ 
  \ipc4 airport-tw\genpb & \textbf{7} & \textbf{7} & 6 & \textbf{7} & \textbf{7} & \textbf{7} & 6 \\ 
  \ipc2 blocks\genpb & 25 & 25 & \textbf{26} & 23 & 25 & 25 & 24 \\ 
  \laas dock-worker\genpb & \textbf{9} & \textbf{9} & \textbf{9} & \textbf{9} & \textbf{9} & \textbf{9} & \textbf{9} \\ 
  \ipc8 driverlog\genpb & \textbf{4} & \textbf{4} & 1 & 0 & \textbf{4} & \textbf{4} & 1 \\ 
  \laas handover\genpb & \textbf{4} & \textbf{4} & 3 & \textbf{4} & \textbf{4} & \textbf{4} & \textbf{4} \\ 
  \ipc2 logistics\genpb & \textbf{23} & \textbf{23} & 4 & 12 & \textbf{23} & \textbf{23} & 7 \\ 
  \ipc4 pipesworld-dl\genpb & \textbf{6} & \textbf{6} & \textbf{6} & \textbf{6} & \textbf{6} & \textbf{6} & 5 \\ 
  \ipc5 rovers\genpb & 34 & 34 & \textbf{38} & 32 & 34 & 34 & 33 \\ 
  \ipc4 satellite-tw\genpb & \textbf{10} & \textbf{10} & \textbf{10} & 7 & \textbf{10} & \textbf{10} & 8 \\ 
  \ipc8 satellite\genpb & 16 & 16 & \textbf{17} & 12 & 16 & 16 & 14 \\ 
  \textbf{Total Generative} & \textbf{144} & \textbf{144} & 125 & 118 & \textbf{144} & \textbf{144} & 116 \\[5pt]
  \ipc2 blocks\fullhierpb & 10 & 7 & \textbf{15} & 10 & 10 & 4 & 10 \\ 
  \ipc2 blocks\parthierpb & 27 & 27 & \textbf{32} & 26 & 27 & 29 & 28 \\ 
  \laas dock-worker\fullhierpb & \textbf{22} & \textbf{22} & \textbf{22} & \textbf{22} & 14 & 20 & \textbf{22} \\ 
  \laas dock-worker\parthierpb & \textbf{12} & 9 & \textbf{12} & \textbf{12} & \textbf{12} & 7 & \textbf{12} \\ 
  \laas handover\fullhierpb & \textbf{10} & \textbf{10} & 1 & \textbf{10} & \textbf{10} & \textbf{10} & \textbf{10} \\ 
  \ipc8 hiking\fullhierpb & \textbf{20} & 2 & 2 & \textbf{20} & 17 & 12 & \textbf{20} \\ 
  \ipc2 logistics\parthierpb & \textbf{28} & \textbf{28} & 27 & 27 & \textbf{28} & \textbf{28} & \textbf{28} \\ 
  \laas race\fullhierpb & \textbf{13} & 10 & 5 & \textbf{13} & \textbf{13} & \textbf{13} & \textbf{13} \\ 
  \ipc8 satellite\parthierpb & \textbf{18} & \textbf{18} & \textbf{18} & 14 & \textbf{18} & \textbf{18} & 14 \\ 
  \laas baking\parthierpb & \textbf{6} & \textbf{6} & \textbf{6} & \textbf{6} & \textbf{6} & \textbf{6} & 5 \\ 
  \ipc8 turnandopen\fullhierpb & \textbf{8} & 0 & 0 & \textbf{8} & \textbf{8} & 0 & \textbf{8} \\ 
  \textbf{Total Hierarchical} & \textbf{174} & 139 & 140 & 168 & 163 & 147 & 170 \\ 
  \textbf{Total} & \textbf{318} & 283 & 265 & 286 & 307 & 291 & 286 \\ 
   \hline

  \end{tabular}
  \caption{Number of problems solved by each configuration with a 30 minutes timeout.}
  \label{tab:feature-results}
\end{table}

The first important point is to see that while all features contribute to the efficiency of the planner, none is critical to its overall performance.
Indeed, the absence of a feature resulted in 11 to 53 less problems being solved which is only a small subset of the overall 318 problems solved by the full configuration.
However, the difficulty of a planning problem is typically exponential in the number of goals in the problem, so solving a few additional problems for a domain therefore has more significance than it might seem.

As expected the \cgen, \cnodelcheck and \cnodecvars variants behave exactly as \cfape on generative problems.
The use of our \emph{forward hierarchical strategy} is however critical in some hierarchical domains, as it can be seen by the poor performance of \cgen on the \emph{hiking} and \emph{turnandopen} domains.
The use of delay checks (absent in \cnodelcheck) and of refinement variables (absent in \cnodecvars) is less critical but contributes to the scaling up of the planners on some hierarchical domains.
The gain of using causal networks (absent in \cnocn)  and $A^\epsilon$ (absent in \cnoae) is also globally noticeable and benefits many domains while not being critical to any of them. 

The most important results are the ones related to \clift as this configuration does not ground the problem.
Thus it can be expected to scale up better on problems that would feature many ground actions.
The gains of non-grounding are noticeable on the \emph{blocks}, \emph{rovers} and \emph{satellite} whose most difficult problems contain many objects.
A limitation of our test set (and of problems from the IPC in general) is that the difficulty of the problem (e.g. length of the plan) is directly correlated to the number of objects in the problem.
Thus we have no instances in our test set where \clift would find a trivial solution plan while \cfape would fail because it could not ground the problem.
On the rest of the test set, we can see the gain, in terms of search control, of applying reachability analysis. This is most noticeable
with the hierarchical versions of \emph{handover}, \emph{hiking} and \emph{turnandopen}.
Nevertheless, we believe it is important to keep the ability to perform a fully lifted search even if it means not using some of the heuristics we have developed.
A view of the overhead required for grounding can be witnessed in \autoref{fig:runtime-grounded-lifted}, where we can see that the lifted version tends to solve simple problems faster because it does not need to ground the problem.

\begin{figure}[htpb!]
  \centering
  \input{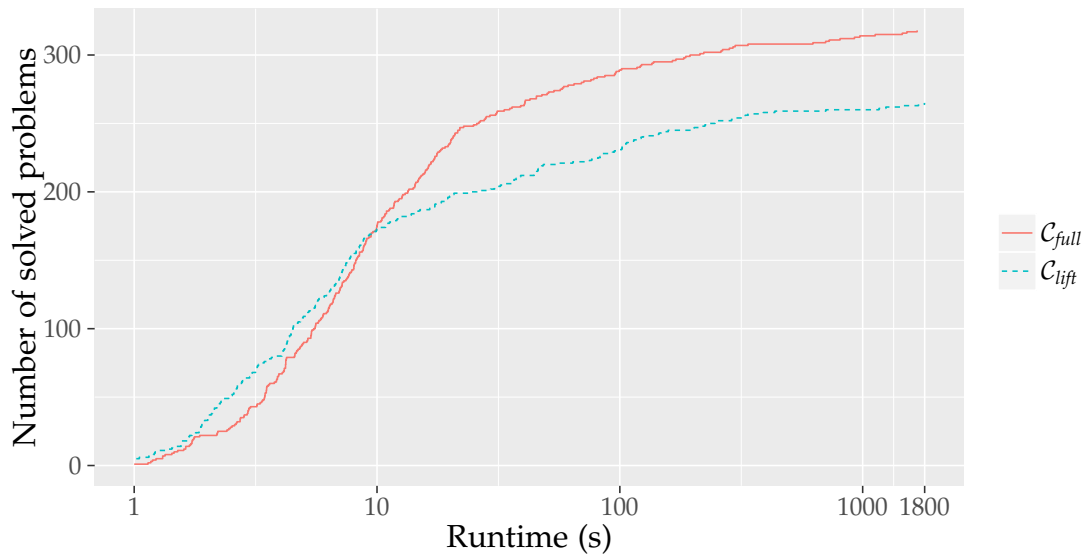}
  \caption{Time needed to solve problems with \cfape and \clift. The Java Virtual Machine takes an important share of time to load and perform optimization during the first seconds of the run. This overhead is however comparable for both configurations.}
  \label{fig:runtime-grounded-lifted}
\end{figure}

\autoref{tab:expanded-nodes} gives the average number of nodes expanded by each configuration on problems solved by all configurations.
On those problems, it should be apparent that the full configuration of FAPE actually requires very little search to find solutions.
Other configurations, and \clift especially, tend to explore a larger part of the search space.
On all these problems, the average branching factor is between 2 and 3.5, with only small variations between the different configurations.

\begin{table}[tbh!]
  \centering
  \smaller
  \begin{tabular}{lrrrrrrr}
     & \providecommand{\cfape}{$C_{full}$}\cfape & \cgen & \providecommand{\clift}{$C_{lift}$}\clift & \cnoae & \cnodelcheck & \cnodecvars & \cnocn \\ 
  \hline
\ipc4 airport\genpb & 164 & 164 & 226 & \textbf{154} & 164 & 164 & 314 \\ 
  \ipc4 airport-tw\genpb & 204 & 204 & 284 & \textbf{195} & 204 & 204 & 450 \\ 
  \ipc2 blocks\genpb & 162 & 162 & 201 & \textbf{139} & 162 & 162 & 182 \\ 
  \laas dock-worker\genpb & 90 & 90 & \textbf{85} & 152 & 90 & 90 & 90 \\ 
  \laas handover\genpb & \textbf{17} & \textbf{17} & 15\,931 & 25 & \textbf{17} & \textbf{17} & 73 \\ 
  \ipc2 logistics\genpb & \textbf{26} & \textbf{26} & 24\,016 & 36 & \textbf{26} & \textbf{26} & 3\,440 \\ 
  \ipc4 pipesworld-dl\genpb & \textbf{230} & \textbf{230} & 898 & 612 & \textbf{230} & \textbf{230} & 1\,436 \\ 
  \ipc5 rovers\genpb & 105 & 105 & \textbf{98} & 153 & 105 & 105 & 104 \\ 
  \ipc4 satellite-tw\genpb & 50 & 50 & \textbf{46} & 1\,824 & 50 & 50 & 676 \\ 
  \ipc8 satellite\genpb & \textbf{125} & \textbf{125} & 292 & 230 & \textbf{125} & \textbf{125} & 184 \\ 
  \textbf{Total Generative} & \textbf{1\,174} & \textbf{1\,174} & 42\,076 & 3\,520 & \textbf{1\,174} & \textbf{1\,174} & 6\,949 \\[3pt]
  \ipc2 blocks\fullhierpb & 62 & 263 & \textbf{39} & 62 & 88 & 72 & 70 \\ 
  \ipc2 blocks\parthierpb & 84 & 84 & \textbf{76} & 121 & 84 & 89 & 566 \\ 
  \laas dock-worker\fullhierpb & 38 & \textbf{33} & 631 & 38 & 41 & 121 & 39 \\ 
  \laas dock-worker\parthierpb & \textbf{90} & 498 & 101 & \textbf{90} & 93 & 75\,088 & \textbf{90} \\ 
  \laas handover\fullhierpb & \textbf{12} & \textbf{12} & 163 & \textbf{12} & \textbf{12} & 16 & \textbf{12} \\ 
  \ipc8 hiking\fullhierpb & \textbf{34} & 36 & 61 & \textbf{34} & 236 & 55 & \textbf{34} \\ 
  \ipc2 logistics\parthierpb & \textbf{50} & \textbf{50} & 56 & 50 & 50 & 84 & 516 \\ 
  \laas race\fullhierpb & \textbf{17} & 17 & 376 & \textbf{17} & 32 & 29 & \textbf{17} \\ 
  \ipc8 satellite\parthierpb & \textbf{87} & \textbf{87} & 112 & 403 & \textbf{87} & 110 & 105 \\ 
  \laas baking\parthierpb & \textbf{58} & 847 & \textbf{58} & \textbf{58} & \textbf{58} & \textbf{58} & 243 \\ 
  \textbf{Total Hierarchical} & \textbf{531} & 1\,925 & 1\,675 & 885 & 781 & 75\,723 & 1\,691 \\ 
  \textbf{Total} & \textbf{1\,705} & 3\,099 & 43\,751 & 4\,405 & 1\,955 & 76\,897 & 8\,640 \\ 
   \hline

  \end{tabular}
  \caption{
    Average number of expanded partial plans on problems solved by all configurations.
    This number does not integrate flaws with a single resolver as FAPE would resolve any such flaw immediately, before adding the node to the search tree.
  }
  \label{tab:expanded-nodes}
\end{table}

\paragraph{Partial vs Full Hierarchy.}

A fair number of domains use partial hierarchies.
A simple and relevant example of the benefits of using partial hierarchies is the \emph{blocks-\parthier} domain, given in Appendix~\ref{sec:blocks-parthier}. 
In this domain, the single task-dependent primitive action is \act{stack}.
When compared to a flat version of the problem it features an additional high-level action $\act{DoStack}(a,b)$ that either does nothing if $b$ is on $a$ or decomposes to $\act{stack}(a,b)$ if it is not.
In the resulting problem, the planner is only allowed to perform one \act{stack} action per \tsk{DoStack} task in the problem.
On the other hand, it can use as many \act{pickup}, \act{putdown} and \act{unstack} actions as necessary to establish the conditions of the \act{stack} actions.
This formulation puts a constraint on the plan: do not use more \act{stack} actions than necessary.
As result, the planner solves more instances of the partially hierarchical domain and does it on average 3.7 times faster than for the generative version.
When compared to the fully hierarchical version of the domain (Appendix~\ref{sec:blocks-fullhier}), the partially hierarchical one is obviously simpler and requires less domain engineering.
Furthermore the partially hierarchical version is more easily solved by FAPE.
This should not be too surprising: FAPE already does a decent job of solving the flat version of the problem and this simple extension simply provides some additional help.
On the other hand the hierarchical version is a very different problem.
While extending the fully hierarchical version with more domain knowledge could probably make it competitive with the partially hierarchical one, there is no need to do it because the simpler partially hierarchical version already performs well.

A similar approach is taken in the \emph{handover-\parthier} and \emph{logistics-\parthier}.
In \emph{handover}, the only task-dependent actions are the ones involving manipulation of an object (\act{pick}, \act{place} and \act{handover}) while the movements of the robots are left free.
Similarly, in \emph{logistics} the only hierarchical parts are the \act{load} and \act{unload} actions.
The movements of the fleet of planes and trucks are left free.
If the planner was to entirely decompose the original task network, it would have a set of \act{load} and \act{unload} actions and would simply need to plan the fleet movements between those.

\subsection{Empirical Comparison with IPC Planners}
\label{sec:pddl-empirical-comparison}
 
\newcommand \fapehier {FAPE-\textsc{\smaller Hier}\xspace}
\newcommand \fapegen {FAPE-\textsc{\smaller Gen}\xspace}

\newcommand\tikzmark[2]{%
\tikz[remember picture,baseline] \node[inner sep=2pt,outer sep=0] (#1){#2};%
}

\newcommand\link[2]{%
\begin{tikzpicture}[remember picture, overlay, >=stealth, shift={(0,0)}]
  \draw[->,color=gray] (#1) to (#2);
\end{tikzpicture}%
}

\paragraph{Experimental setup.}

For comparison with state of the art PDDL temporal planners, we consider the 11 temporal domains from the International Planning Competition (IPC) all of which have PDDL2.1 versions. 
We have manually written ANML versions of the domains that closely mirror the original PDDL model: domains have the same actions and a direct mapping from predicates to corresponding state variables.
For each domain, we wrote a domain-dependent automated translator that parsed the original PDDL problems and output ANML problem files.

We use hierarchical versions of a subset of those domains (\emph{blocks}, \emph{hiking}, \emph{logistics}, \emph{turnandopen}).
Those domains were chosen either because they have a natural expression with partial hierarchies (\emph{blocks} and \emph{logistics}) or because FAPE had difficulties in solving the generative versions of the problem (\emph{hiking} and \emph{turnandopen}).

Three of the selected domains have ``advanced'' temporal features. Namely, \emph{airport-tw} and \emph{satellite-tw} have temporal windows that respectively restrict the instants at which a plane can take off and at which a satellite can transmit data.
In addition, the goals of \emph{pipesworld-dl} are associated with deadlines that must be met by the solution plan.
All domains, in their ANML and PDDL versions, are available online in FAPE's public repository.%
\footnote{Available at \url{https://github.com/laas/fape}}

\paragraph{Planners.}

We compare FAPE to POPF \cite{DBLP:conf/aips/ColesCFL10}, OPTIC \cite{Benton2012} and Temporal Fast-Downward (TFD \cite{Eyerich2012}).
POPF is a complete PDDL2.1 planner based on temporally-lifted progression planning \cite{Cushing2007}.
As such, it can be seen as a forward-search planner taking a late-commitment approach in the ordering of actions and uses the $h^{\textit{FF}}$ heuristic adapted for temporal planning.
OPTIC is a recent extension to POPF that supports more advanced PDDL features, including PDDL2.2 timed initial literals and PDDL3 preferences.

Temporal Fast-Downward (TFD \cite{Eyerich2012}) is a temporal extension of the successful Fast-Downward classical planner using a decision-epoch mechanism \cite{Cushing2007}.
It performs heuristic search in the space of time-stamped states, using an adapted version of the context-enhanced additive heuristic \cite{Helmert2008}.
TFD supports PDDL2.1 syntax but is not complete as it only supports a limited class of problems with required concurrency.
More specifically it cannot handle problems with interdependent actions.

POPF and TFD were runner-up in the temporal satisficing track of IPC-2011 and IPC-2014 respectively.
We did not consider YAHSP, the winner of those two tracks, in our comparison as it only supports temporally simple problems and is therefore strictly less expressive than the other planners considered here.

We distinguish two versions of FAPE.
\fapegen denotes the purely generative version of FAPE that only considered generative domain encodings with no hierarchical information.
It uses the general search strategy and has no domain-dependent knowledge.
\fapehier uses the hierarchical versions of the \emph{blocks}, \emph{hiking}, \emph{logistics} and \emph{turnandopen} domains together with the \emph{forward hierarchical} search strategy.
For generative domains, it uses the general search strategy and is equivalent to \fapegen.

\paragraph{Results.}
All tests were performed on an Intel Core i7 with 3GB of RAM and allowed to run for 30 minutes.
The results are given in \autoref{tab:pddl-num-solved} in terms of the number of problems solved by each planner within the time limit.

The performance of the purely generative version of FAPE is comparable with that of POPF and OPTIC.
TFD is ahead in terms of number of problems solved.
The addition of hierarchical knowledge in 4 of the domains allows FAPE to solve 36 more instances.
As a result, it outperforms POPF and OPTIC, but still lags behind TFD in number of solved problems.
However it should be noted that the advantage of TFD can be entirely reduced to its excellent performance on the \emph{airport} domain. 
This domain is particularly problematic for FAPE because it contains conditional effects that are not natively supported by FAPE and that were worked around by partially grounding the problem during the translation to ANML. 
This forced grounding however interacts badly with FAPE' search strategy that relies on lifted exploration and very negatively impacts its performance on this particular domain.

\begin{table}[htb!]
  \centering
  \def \cfape {{FAPE-\smaller\sc Gen}}
  \def \cfapehier {{FAPE-\smaller\sc Hier}}
  \def \genpb {}
  \def \discreet#1 {\color{black!100}{#1}}
  \begin{tabular}{l c c c c c} 
     & \providecommand{\cfape}{$C_{full}$}\cfape & \providecommand{\cfape}{$C_{full}$}\cfapehier & POPF & OPTIC & TFD \\ 
  \hline
\ipc4 airport\genpb & 6 & \discreet{6} & 7 & 7 & \textbf{37} \\ 
  \ipc4 airport-tw\genpb & 7 & \discreet{7} & \textbf{17} & 7 & 1 \\ 
  \ipc2 blocks\genpb * & 25 & 27 & 32 & 32 & \textbf{35} \\ 
  \ipc8 driverlog\genpb & \textbf{4} & \discreet{\textbf{4}} & 0 & 0 & 0 \\ 
  \ipc2 logistics\genpb * & 22 & \textbf{27} & \textbf{27} & \textbf{27} & \textbf{27} \\ 
  \ipc4 pipesworld-dl\genpb & 6 & \discreet{6} & 6 & \textbf{13} & 2 \\ 
  \ipc5 rovers\genpb & \textbf{34} & \discreet{\textbf{34}} & 26 & 26 & 29 \\ 
  \ipc4 satellite-tw\genpb & \textbf{10} & \discreet{\textbf{10}} & 6 & 4 & 0 \\ 
  \ipc8 satellite\genpb & 16 & \discreet{16} & 3 & 4 & \textbf{17} \\ 
  \ipc8 turnandopen\genpb* & 0 & 8 & 8 & 9 & \textbf{18} \\ 
  \ipc8 hiking\genpb* & 0 & \textbf{20} & 10 & 9 & 19 \\ 
  \textbf{Total} & 130 & 165 & 142 & 138 & \textbf{185} \\ 
   \hline

  \end{tabular}
  \caption{Number of problems solved in 30 minutes for various temporal IPC domains.
    The best performance in given in bold.
    \fapehier uses hierarchical versions of the starred domains and generative versions of the others.
  }
  \label{tab:pddl-num-solved}
\end{table}

A focused subset of the results is given in \autoref{tab:pddl-num-solved-temporal} for problems with deadlines and timewindows.
Those are the only domains of the test set where time is strictly needed, i.e., on all other domains every solution plan has a valid totally ordered counterpart.
While being the overall best performer, TFD exhibits poor performance on those domains, solving only 3 problems.

\begin{table}[htb!]
  \centering
  \def \cfape {{FAPE \smaller(\sc Gen/Hier)}}
  \begin{tabular}{l c c c c}
     & \providecommand{\cfape}{$C_{full}$}\cfape & POPF & OPTIC & TFD \\ 
  \hline
\ipc4 airport-tw\genpb & \textbf{17} & {7} & \textbf{7} & 1 \\ 
  \ipc4 pipesworld-dl\genpb & 6 & 6 & \textbf{13} & 2 \\ 
  \ipc4 satellite-tw\genpb & \textbf{10} & 6 & 4 & 0 \\ 
  \textbf{Total} & 23 & 19 & {24} & 3 \\ 
   \hline

  \end{tabular}
  \caption{Results limited to domains featuring deadlines or timewindows. \fapehier does not appear separately as we only considered generative versions of those domains.}
  \label{tab:pddl-num-solved-temporal}
\end{table}

\section{Related work and discussion}
\label{sec:soa}

\subsection{PDDL Temporal Planners}
\label{sec:comparison-pddl}

\paragraph{STRIPS and PDDL.}
The original PDDL language \cite{McDermott1998} and its ancestors STRIPS and ADL, define actions as state-transition functions with a uniform duration.
This was extended with the introduction of PDDL2.1 for the purpose of the third International Planning Competition (IPC) \cite{Fox2003}.
The philosophy behind PDDL2.1 is to see a durative action as two instantaneous at-start and at-end actions that both produce instantaneous state-transitions. Those ``snap'' actions are linked together by duration constraints that restrict the possible delays between the start and the end times of the action as well as durative conditions that require some condition to hold in all states traversed while performing the action.
While this seems like a minor extension, it allows the expression of temporal planning problems with required concurrency \cite{Cushing2007} and with interdependent actions \cite{Cooper2013}.

As pointed out by \textcite{Smith2003}, a strong limitation of the language is that conditions and effects can only be placed at the start and end of the action.
While this limitation can be avoided by splitting a complex temporal action with intermediate time points into multiple subactions \cite{Fox2004a}, encoding such durative actions by hand is difficult and error prone.
A possible approach is to compile a more expressive language into PDDL2.1 to benefit from its large ecosystem (as done for PDDL-TE \cite{Cooper2010} or for ANML \cite{anml2008}).

The PDDL2.1 language has been further extended with \emph{timed initial literals} that allow truth assignments on predicates at arbitrary times (in PDDL2.2 \cite{Edelkamp2004}) and temporally extended goals expressed as constraints on the state trajectory (in PDDL3.0 \cite{Gerevini2005}).
Even though those extensions are essential in representing real world problems, they have not been much used in temporal planning, e.g.,  none of the participants of the temporal track in the IPC-14 supported them natively.%
\footnote{See \url{https://helios.hud.ac.uk/scommv/IPC-14/planners_actual.html} for the supported features of temporal planners.}

\paragraph{Forward-chaining planners.}
Not surprisingly, most temporal planners participating in planning competitions evolved from classical planners.
Similarly to the classical planning tracks , the temporal satisficing tracks have been dominated by forward-chaining temporal planners.
We can partition forward-search planners into three categories depending on their search space:

\emph{First-fit temporal planners} are essentially classical planners that temporally schedule a sequential solution.
The most surprising example is the baseline planner that (unofficially) won the temporal track of IPC-2008 by greedily rescheduling the sequential solution provided by {\sc MetricFF} \cite{Hoffmann2003}.
A more advanced implementation of this approach is YAHSP \cite{Vidal2004,Vidal2011a,vidal2014yahsp3} whose second and third versions respectively won the temporal tracks of IPC-2011 and IPC-2014.
While such planners have the advantage of being simple, they are incomplete as they can only solve temporally simple problems that do not require concurrency between actions \cite{Cushing2007}.

\emph{Decision-epoch planners} maintain a timestamp (called the decision epoch) at which they can schedule the actions.
Successors of search nodes are generated by either starting a new action at the timestamp or advancing the timestamp (typically to be just after the next effect).
This technique has been the base of many influential temporal planners such as SAPA \cite{Do2003},  Temporal Fast-Downward (TFD) \cite{Eyerich2012}, and others \cite{Haslum2001,Haslum2006,Bacchus2000}.
These planners support some cases of required concurrency but are still not complete for temporally expressive problems \cite{Cushing2007}.

\emph{Temporally lifted planners} separate the problems of what actions to add to the plan and when to schedule them by using an STN to keep track of temporal constraints.
These ideas have been first introduced by CRIKEY \cite{Halsey2004} and have been developed in its successors: CRIKEY3 \cite{Coles2008}, POPF \cite{DBLP:conf/aips/ColesCFL10}, COLIN \cite{Coles:2012una} and OPTIC \cite{Benton2012}.
Unlike First-Fit and Decision Epoch planners, those temporally lifted planners are complete for the semantics of PDDL2.1 and can solve problems with required concurrency or interdependent actions \cite{Cushing2007}.

Regardless of their search space, all forward-chaining planners rely on heuristics.
Most of the heuristics are based on a \emph{temporal relaxed planning graph} (TRPG) built by ignoring the delete effects of actions.
These heuristics are generally adapted from the one that have been successful in classical planning such as $h^{add}$  \cite{Bonet2001}, $h^{cea}$ \cite{Helmert2008} or $h^{FF}$ \cite{Hoffmann2001}.

\paragraph{Other notable approaches.}

The GraphPlan framework \cite{Blum1997281} has seen many extensions to handle temporal planning, as first demonstrated by \textcite{Smith1999} whose planner finds a solution by extracting it from a planning graph with temporally annotated nodes.
In the same line, LPGP \cite{Long2003} and TLP-GP \cite{Maris2008} both decouple the causal parts of the problem, dealt with in a GraphPlan framework, and the temporal parts that are addressed by a linear programming solver or a disjunctive temporal network.

LPG \cite{Gerevini2003,Gerevini2006} builds an action graph (with similarities to planning graphs) through stochastic local search.
Its latest version is able to handle problems with required concurrency by splitting durative actions into instantaneous ones while considering temporal constraints in an STN \cite{Gerevini2010}.

Temporal plan-space planning has been represented by VHPOP \cite{Younes2003}, a ground plan-space planner that uses an adaptation of the $h^{add}$ heuristic to guide itself in the set of ground partial plans.
CPT by \textcite{Vidal2006} is a more involved ground plan-space planner that seeks minimal makespan plans.
This is done by placing an upper bound on the makespan of the plan and trying to prove through inference and search whether such a plan exists.
The inference in CPT relies on dedicated pruning rules, handled as a constraint satisfaction problem.

ITSAT \cite{Rankooh2015}  is the first satisfiability-based planner to support PDDL2.1 problems with required concurrency.
This is achieved by first solving, with a SAT solver, an atemporal problem where all durative actions have been split into instantaneous ones.
The planner then tries to find a schedule for this plan by considering the temporal constraints in an STN.
If no such schedule exists, the problem is extended with additional clauses forbiding the cause of the temporal inconsistency and the procedure is restarted.

\subsection{Hierarchical planners}
\label{sec:comarison-hierarchical}

HTN planning has initially been developed around a plan-space approach with planners such as NOAH~\cite{Sacerdoti75}, Nonlin~\cite{McAllester1991}, SIPE~\cite{Wilkins1990,Wilkins1995}, O-Plan~\cite{Tate1994} and UMCP~\cite{Erol1994a}.
For totally ordered HTNs, it is possible to plan tasks in order they will be executed, and thus, to knows the current world state at each step of the planning process. This was proposed in  SHOP \cite{Nau2000}, and extended in SHOP2 \cite{Nau2003} for partially ordered tasks and durative actions. A Multi-Timeline
Preprocessing technique is proposed to translate PDDL2.1 operators such as to keep track of temporal information in the current
state. Each operator is augmented with start time, duration and read-time and write-time primitives for time bookkeeping  upon
instantiation.  A more general procedure to transform durative actions in HTN methods is proposed in \cite{Goldman2006}: it maps an action into a task
network with two or three operators, one for the start of the durative action, one for its end, and an epsilon length ``spacer'' pseudo
action. It also adds a time fluent and proposes a modification of the planner heuristic to handle these extensions. 
\cite{yorke05} proposes to extend HTN planning with temporal propagation on associated `local' STNs; these are limited to adjacent nodes in the task network, and thus much smaller than a global STN. 

SIADEX  \cite{Castillo2006,FdezOlivares:2006vg} is an elaborate temporal state-based HTN planner. It  allows the placement of effects at arbitrary timepoints within durative actions. Conditions are restricted to be placed at the action's start.
Like SHOP2, SIADEX builds an inherently sequential solution through action chaining. However, it also performs an online scheduling of the plan by constraining an action to start after all its preconditions are true (using an STN to keep track of temporal constraints).
XEPlanner~\cite{Tang:2012hp} is an HTN planner designed specifically for addressing emergencies in dynamic situations. It supports durative actions, temporally-enhanced methods and axioms. Planning is done with an anytime heuristic algorithm handling the task
 network  together with an STN. 
GSCCB-SHOP2 \cite{Qi:2017jm} extends SHOP2 for handling time and resources. Specific state-updating rules are used for resource reasoning, together with a backtrack consistency checking for managing simple temporal constraints.

\emph{HTN Planning with task insertion} extends the traditional HTN formalism by allowing the use of primitive actions at arbitrary places in the solution plan \cite{Geier2011}.
It was first explored with PANDA, a lifted hierarchical planner reasoning in plan space \cite{Schattenberg2009}.
PANDA allows the use of high-level actions that can be decomposed into lower-level ones.
Both high-level and primitive actions can be freely inserted by the planner to resolve open goals, outside of any decomposition tree.
Parameters of the partial plan are kept lifted and handled in a binding constraint network.
PANDA only supports limited qualitative time with its plan-space representation allowing it to represent partially ordered plans.
Many heuristics have been studied to be used with PANDA to account for causal and hierarchical features in a best first search setting \cite{Schattenberg2009,Elkawkagy2012,Bercher2014}.
Their effect on the scalability of PANDA is however unconvincing and does not allow the planner to handle complex plans.%
\footnote{On the provided test data, Breadth-First Search is slightly slower but overall competitive with a Best-First search guided by the proposed heuristics.}
HTN planning with task insertion is also supported by HiPOP, a ground hierarchical planner reasoning in plan-space \cite{Bechon2014,Bechon2016}.
HiPOP considers primitive actions in the form of PDDL2.1 operators and \emph{abstract actions} that can be transformed into a partially ordered set of primitive or abstract actions through the application of methods.

The HTN planning approach  of \cite{sohrabi:2009ui} handles qualitative preferences on actions together with temporally extended state preferences and constraints  (as in PDDL3), expressed in a subset of LTL. Time \textit{per se} is not explicit in the representation. The modal operators are compiled out at preprocessing time into additional predicates. A Branch\&Bound algorithm, benefitting from HTN methods and pruning heuristics, is used to find a most preferable plan. The approach is used in \cite{sohrabi:2013tr} for the automated composition of software components, as in the composition of web services and stream processing systems.

Most HTN planners do not rely on heuristic guidance or reachability analysis but instead perform a depth search guided by on the domain-specific control knowledge provided by their hierarchical methods.
Interest in using the delete-relaxation for hierarchical planning has been studied theoretically \cite{alford2014} and in particular with temporal domains in a preliminary version of our work \cite{bit-monnot2016}.
Heuristics for non-temporal HTN have have been explored in the context of the plan-space hierarchical planner PANDA where no state-information is available for method selection \cite{Schattenberg2009,Bercher2014} as well as in forward progression search \cite{Hoeller2020,Hoeller2020a}.

\subsection{Timeline based Representations and Planners}
\label{sec:comparison-timeline}

Timeline-based representations focus on scheduling various temporal intervals representing values taken by state variables.
In these approaches, actions are typically composed of a set of temporally qualified assertions representing the action conditions and effects over various state variables (or timelines).
Coordination between the different timelines is made by temporal constraints that relate the various assertions of an action.
An early proposal was the one by \textcite{Allen1983a} based on Allen's temporal algebra \cite{Allen1983}.
A timepoint centered view was proposed by \textcite{Ghallab1994} for the IxTeT planner.
While many generative and hierarchical planners have since chosen this paradigm (e.g. \cite{Chien2000a,Tate1994,Do2011,Cesta2009,El-kholy1996,Frank2003,Barreiro2012,Muscettola2002}), no dominating language has emerged for the encoding of such planning problems.

While much work on temporal planning can be seen an incremental evolution from classical planning, research on temporal planning largely predates the introduction of durative actions in PDDL2.1.
Indeed the observation by  \citeauthor{Vere1983}, that the Partial-Order Causal Link technique can be generalized to rich temporal models, led to numerous planners with advanced temporal representation capabilities \cite{Vere1983,Ghallab1994,Muscettola1994,Penberthy1994,Frank2003}.
IxTeT \cite{Ghallab1994} is a notable least-commitment planner that allows reasoning on time and resources in a plan-space approach.
A large part of the internal representation and reasoning is handled by specific constraints satisfaction problems representing constraints on timepoints and parameters of actions.
IxTeT has a domain-independent search strategy based on an extended notion of least-commitment.

Another line of work emerged with HSTS \cite{Muscettola1994} from the objective of tightly integrating planning and scheduling.
Instead of actions, HSTS relied on the notion of compatibilities to describe the possible interactions between various timelines.
The planner's objective is to find  fully defined timelines that respect all compatibilities.
HSTS was notably used  for the Remote Agent Planner (RAX-PS) that was demonstrated on board for controlling the operations of the Deep Space One spacecraft \cite{nayak1999}.
HSTS matured into EUROPA \cite{Frank2003,Barreiro2012}, and its language NDDL, whose central paradigm is to see planning as a dynamic constraint satisfaction problem where choices of the planner simply results in the addition of constraints to underlying constraints networks.
EUROPA relies on domain-dependent knowledge to guide a depth-first search.
Efforts to transpose domain-independent heuristics into EUROPA but have seen limited results \cite{Bernardini2007,Bernardini2008}.
 
The European Space Agency launched the Advanced Planning and Scheduling Initiative (APSI) that resulted in the definition of the Timeline-based Representation Framework (APSI-TRF) \cite{Cesta2009}.
While not a planner per se, APSI-TRF aims at being a timeline-based deliberation layer to provide facilities for the implementation of timeline-based planners.
It has  been used as a building block for the OMPS \cite{Cesta2008}  and GOAC-APSI  \cite{Fratini2011} planners. Both use a search algorithm similar to EUROPA. APSI has also been used for MrSPOCK \cite{Cesta2009}, a long term planner for the Mars Express mission that works by greedily constructing a long term plan optimized with a genetic algorithm.

A similar approach is taken in the meta-CSP framework which addresses a planning and scheduling problem as a higher-level constraint satisfaction problem that requires cross reasoning on several lower level CSPs (referred to as ground CSPs). Meta-constraints enforce high-level requirements on the solution plan, playing the role of the flaw detection functions of other timeline based planners.
The detected flaws are handled by posting additional constraints on the ground CSPs.
The key idea is to permit an easy integration of application specific components by the addition of supplementary ground CSPs and meta-constraints.
The meta-CSP framework has recently been used in CHIMP, a  planner with a timeline based representation \cite{Stock2015}.
Like other HTN planners, CHIMP uses a notion of task that can be decomposed into partially ordered task through the application of methods.

ASPEN is another timeline based planner \cite{Fukunaga1997,Smith1998,Chien2000,Chien2000a}, which uses the AML language  \cite{Fukunaga1997,Chien2000a}.
Its timelines parameters are handled in a dedicated constraint network.
It uses an iterative repair technique \cite{Zweben1993} to perform local search in place of the depth-first search adopted by other timeline based planners.
Efficiency is sought by the definition of a hierarchical structure where activities can be refined into sub-activities, allowing the planner to quickly bootstrap its search with a minimal (possibly flawed) plan.

The Action Notation Modeling Language (ANML \cite{anml2008}) is a  proposal to overcome the absence of apparent causal structure of NDDL, the limited support for generative planning of AML and the lack of hierarchy in IxTeT.
ANML has a strong emphasize on generative planning with the direct inheritance of earlier timeline-based planning models.
It comes with a clear notion of action with conditions and effects taking the form of temporally qualified assertions at arbitrary timepoints.
In addition, ANML provides some facilities for hierarchical planning:
each action instance is associated with its own predicate that is set to true on the action start and to false on the action end and can be used to express subtasks.
This definition departs from the traditional definition of HTN problems as it allows for task-sharing: a single action can support multiple tasks just like a single effect can support multiple conditions.
The original language definition can be characterized as HTN with task-insertion \cite{alford2016hierarchical}: additional actions can be placed at arbitrary places in the solution plan.
The latest version of ANML proposes to restrict the possible placement of actions relatively to higher-level actions by marking them as \emph{motivated} \cite{AnmlManual}.
Conceptually, the presence of such an action must be ``motivated'' by the presence of an higher level action that requires its presence and temporally envelops it and is closely related to our \emph{task-dependency} concept.
To the best of our knowledge no planner exists that support all ANML features and in particular, none that supports any of its hierarchical features.
TAMER \cite{valentini2020} is a forward-chaining planner that supports a form of intermediate conditions and effects of the ANML language. More specifically, it supports conditions and effects constrained to be a fixed time-amount after the start or before the end of an action.
LCP \cite{bit-monnot-cp-2018} supports a more complete set of temporal features, equivalent to the ones in FAPE, by constructing a sequence of scheduling problems that are solved with an SMT solver.
While we are aware of other prototypes of ANML planners (at NASA Ames, Adventium Lab and Fundazione Bruno Kessler) none of them have been the subject of a publication nor go as far as FAPE in the support of temporal and hierarchical features.
FAPE is thus the first planner to support most of the expressive temporal and hierarchical features of the ANML modeling language.

\section{Conclusion}
\label{sec:conclusion}

We reported here on FAPE, a Flexible Acting and Planning Environment based on timelines. To our knowledge, FAPE is the first planner supporting both the temporal and hierarchical features of the expressive ANML modeling language.
ANML has significant advantages because it consistently blends flexible timelines with hierarchical refinement methods, when available.
We presented a planning algorithm for the proposed representation, discussed the specifics of its search space, and proved its soundness and completeness. A significant contribution of the  presented work is the development of well informed heuristics and inference methods for  this algorithm. The approach takes into account an original reachability analysis supporting causal networks that are explicitly maintained by the planner and used to focus the search. A  comprehensive experimental evaluation, using the standard benchmarks of the field, allowed us to assess several search strategies for FAPE and to compare its performance to other planners, with or without hierarchical decomposition knowledge. Our evaluation reflects that the proposed techniques for this expressive representation are competitive.

Furthermore, there is certainly room for a number of improvements of our techniques and implementation, and a large opportunity for optimizing planning domain knowledge, which we have not yet explored. Since the source code of the planner and all the domains presented here are openly available, we do hope that this article will generate interest in the temporal planning  community and trigger efforts to address these topics and extensions.
An important extension that deserves to be further investigated is the support for resources, for which many techniques have been devised by the constraint-based planning and scheduling  community \cite{Vilim2007}.

A strong motivation for the development of FAPE is to support the integration of planning and acting. The latter involves opportunistically instantiating the unbound variables remaining in a synthesized plan, in particular for temporal variables through a dispatching algorithm, and refining planned actions into executable commands. The integration also requires plan repair techniques, and the assessment of when replanning is preferable to repairing. These developments, not presented here, have been addressed and integrated within the design of an activity manager interleaving planning and acting \cite[Chap. 5]{bit-monnot2016}. Experiments with a PR2 robot and in simulation indicates that the proposed representation is very convenient for handling a temporally rich domain at the planning as well as the acting levels.

In a temporally rich domain, an actor has to relate its actions to exogenous events, which is feasible with known techniques when the occurrence of these events is fully observable \cite{morris2014}. In many application areas, and in service robotics in particular, full observability is not a realistic assumption. One has to check whether a plan is dynamically controllable despite partial observability, and if not, to decide what needs to be observed to make it controllable, and determine how to consistently integrate the required sensing actions with other planned activities. The proposed approach turned out to be quite convenient for supporting the corresponding developments within FAPE \cite{bit-monnot-2016-ijcai}.

 For the sake of space, this paper does not cover the two issues of acting and partial observability. However, they need to be mentioned as they open several promising avenues for future investigations and developments using the ANML representation.

{
\let\i\oldi
\newpage
\printbibliography[heading=bibintoc]
}
\newpage
\appendix

\section{Overlength Proofs}

\subsection{Proof of Soundness and Completeness (\autoref{prop:fape-sound-complete})}
\label{proof:sound-complete}

\FapeSoundComplete* %

\begin{proof}
  ~\\\textbf{Soundness.}
  Soundness requires proving that any plan returned by \procedure{FapePlan} is indeed a solution, i.e., that it respects the four conditions of \autoref{def:solution-plan}.
  
  The first one, reachability of $\chronicle^*$ from $\chronicle_0$, can be shown since \i any transformation made to the plan is done by applying a resolver and \ii all resolvers can be expressed in terms of the allowed transformations (task refinement, free action insertion and restriction insertion).
  As reachability admits any sequence of transformations, the application of any number of resolvers, regardless of their order will result in a plan that is reachable from the original one.

  Each of the last three conditions corresponds to a flaw type.
  Hence, if any of those conditions is not met, the partial plan would have a flaw that the planner will need to resolve.
  The requirement that the plan be flaw free, together with the type of flaws considered is thus sufficient to guarantee that the last three conditions are met.

  \logicalpar \textbf{Completeness.}
  We rely on the study by \textcite[Sec. 2.6 to 2.8]{Schattenberg2009} and \textcite{kambhampati1995} of the general refinement planning procedure where a set of \emph{deficiency detection functions} identify flaws in a partial plan and a set of \emph{modification generation functions} generate modifications of the plan that fix the flaws (i.e. resolvers).
  Our procedure is an instantiation of this more general scheme with three detection functions (one for each flaw) and their modification generation functions implicitly defined by the set of resolvers associated with a flaw.
  Showing completeness of a particular refinement planning procedure requires us to show that \i no solution plan is rejected because it has flaws, \ii for a given flaw, our resolvers cover all the ways of addressing it \cite[Def 3.2]{Schattenberg2009}. 

  We first show that no solution plan is ruled out because of the presence of a flaw.
  We assume that a partial plan $\chronicle$ has a flaw of a given type and show that it cannot be a solution plan or that it will be transformed into an equivalent plan.
  \begin{itemize}
  \item if $\chronicle$ has an unrefined task $\tau$ then it is not a solution plan according to \autoref{def:solution-plan}.
  \item if $\chronicle$ is detected with an unsupported assertion $\alpha$, it means that we have not explicitly added a causal link from a supporting assertion $\beta$.
    Assuming \chronicle to be solution, it means that $\alpha$ is supported by the presence of a chain of assertions that must eventually originate in a change assertion or in an a priori supported assertion $\beta$.
    $\beta$ is thus a causal support of $\alpha$ and there is a chain of persistences preventing any change to their state variable during $[end(\beta),start(\alpha)]$.
    Even though this situation could trigger a flaw, the resolver would be a causal link $\beta \rightarrow \alpha$ that would simply make the support explicit in an equivalent plan.
  \item Similarly, the planner could detect a conflict between two assertions that cannot be conflicting due to implicit constraints. 
    Indeed, the use of arc-consistency in place of full consistency in the binding constraint network could make FAPE miss such implicit constraints.
    Here again, the planner would simply provide a resolver making the constraint explicit and resulting in an equivalent plan.
  \end{itemize}

  We now show that no solution plan is missed due to an incomplete set of resolvers:
  \begin{itemize}
  \item \textbf{Unsupported Assertion.}
    It is easy to see, by the requirement for causal support, that all assertions in a solution plan must have an incoming causal link (or an equivalent set of persistence assertions) from a supporting assertion.
    For completeness, we need to show that all possible supporters are considered regardless of whether they are in the current partial plan or will be inserted later.
    Our approach distinguishes the assertions already present in the plan from those that will be inserted later.
    The \emph{direct supporter} resolver provides an option for the planner to select any assertion already in the plan.
    
    Regarding the assertions not yet in the plan we distinguish two cases: the supporter can be introduced by refining an existing task or it can be introduced by an action not derived from an existing task.
    The former is handled by the \emph{task supporter} resolvers that allow choosing any of the future assertions derived from an existing task.
    For the latter, let us observe that the containing action will be part of another refinement tree not yet in the plan, whose root is necessarily a free action.
    Our \emph{free action} resolvers allow the planner to consider the addition of free actions as the source of new supporting assertion, regardless of whether they appear in the free action it self or in a descendant action obtained by decomposing its subtasks.

    To summarize our resolvers allow consideration of all possible sources of supporting assertions, namely: \i those already in the plan, \ii those introduced by the extension of existing refinement trees, and \iii those introduced by the extension of new refinement trees.

  \item \textbf{Unrefined task.} 
    Assuming that no previous commitment was made to the support of unsupported assertions, the set of resolvers is complete as it considers all possible task decomposition transformations.
    Let us assume now that the planner made an earlier choice regarding the support of an assertion $\alpha$: it decided that $\alpha$ must be supported by a descendant of an unrefined task $\tau$ (i.e. $\tau \in \descendanttasks{\alpha}$).
    This can lead the planner to disregard an action template $A$ for the refinement of $\tau$ because $A$ would not have any possible effect for supporting $\alpha$.
    This pruning is, however, sound because there cannot be an assertion supporting $\alpha$ deriving from $A$, i.e., there is no solution plan involving the decomposition of $\tau$ with $A$, given our previous commitment.
    Furthermore, the possibility of using $A$ to refine $\tau$ will be considered in other branches of the search space: those derived from another resolver for $\alpha$.

  \item \textbf{Conflicting assertions.} For two assertions $\alpha_1$ and $\alpha_2$ to be conflicting in a partial plan $(\plan,\assertions,\constraints)$, it is necessary that a conjunction of constraints be entailed by $\constraints$: \i they overlap (i.e. $end(\alpha_1) \geq start(\alpha_2) \wedge end(\alpha_2) \geq start(\alpha_1)$), 
  \ii all the arguments to their state variables are equal, and \iii their values are different.
    The negation of this conjunction of constraints thus forms a disjunction of constraints that must hold in any solution plan (otherwise, $\alpha_1$ and $\alpha_2$ would be conflicting).
    Each of the disjunct corresponds to a resolver of the threat. Because at least one such disjunct must hold in a solution plan, the set of resolvers for conflicting assertions is therefore complete.
  \end{itemize}
\end{proof}

\subsection{Proof of the convergence of earliest appearances (\autoref{prop:ea})}
\label{sec:proof-earliest-appearance}

\EarliestAppearanceConvergence*

\begin{proof}
\def \etrulyreach {\ensuremath{E^{TrulyReach}}\xspace}
\def \eassreach {\ensuremath{E^{AssReach}}\xspace}
\def \enotreach {\ensuremath{E^{NotTrulyReach}}\xspace}
We first suppose that all actions have a single condition.
With this assumption, there are two sets of trivially reachable elements: fluents in initial timed literals (i.e. in $I$) and actions and fluents appearing in self-supporting causal loops.
An action/fluent $n$ is reachable if there is a path (i.e. sequence of actions/fluents) from one of those trivially reachable elements to $n$.
Note that all fluents in timed initial literals are part of the assumed reachable set.
Furthermore, a self-supporting causal loop necessarily contains an action with an after-condition \cite{Cooper2013}.
Since after-conditions are ignored at first, it means that this action will be part of the assumed reachable set.
Furthermore, propagation will never remove those from the assumed reachable set because we have an upper bound on their earliest appearance.
We call the set of truly reachable elements \etrulyreach which is a subset of the set of assumed reachable elements \eassreach.
We define as \enotreach the subset of \eassreach that is not reachable.

An element $d$ is reachable if there exists an element  $s \in \etrulyreach$, such that there is a path from $s$ to $d$.
The true earliest appearance of $d$ is given by $\minimum_{s\in\etrulyreach}ea(s) + spp(s,d)$, where $spp(s,d)$ is the length of the shortest path from $s$ to $d$.

At first, our algorithm is optimistic, which means that we consider as reachable all nodes with a path from an element of $\etrulyreach \union \enotreach$.
To show that earliest appearances of reachable nodes eventually converge to their true values, we first show that the earliest appearances of nodes in \enotreach indefinitely increase until they are removed from the model.

A node $n \in \enotreach$ is necessarily an action with an after-condition that was optimistically ignored.
The fact that $n$ is not reachable means that its after-condition $p$ is not reachable, meaning that there is no path from an element of \etrulyreach to $p$.
If $p$ is not assumed reachable, then $n$ will be removed.
Otherwise we can distinguish two cases depending on which node in $\enotreach$ provides the earliest start time of $p$:
\begin{itemize}
\item  $ea(p) = ea(n) + spp(n,p)$. In this case, this is an infeasible causal loop involving $n$ and the earliest appearance of $n$ will be increased at each iteration.
\item there is another node $n' \in \enotreach$ such that $ea(p) = ea(n') + spp(n',p)$. In this case, we can recursively perform similar reasoning on $n'$: it is either part of an infeasible causal loop or depends on a node $n'' \in \enotreach \setminus \{n,p\}$.
  In both cases, it depends on a node involved in a causal loop and its earliest appearance  would be increased, meaning that the earliest appearances of $p$ and then $n$ would increase as well.
\end{itemize}

We have shown that the earliest appearances of all nodes in $\enotreach$ indefinitely increase until they are removed.
This is also the case for all unreachable nodes that were once assumed reachable because the sources of all their shortest paths is a node in \enotreach.

On the other hand, we will eventually reach a point where:
\[\minimum_{s\in\etrulyreach}ea(s) + spp(s,d) = \minimum_{s'\in\etrulyreach\union\enotreach}ea(s') + spp(s',d)\]
because the earliest appearance of any $s' \in \enotreach$ diverges towards infinity.
As a consequence, the earliest appearances of reachable nodes will eventually converge to their true values.

This result can be extended to actions with more than one condition by observing that, in a given iteration, the algorithm only uses a single condition of the action: the one that would make it the latest.
\end{proof}

\subsection{Proof that late nodes are unreachable (\autoref{prop:late-impos})}
\label{sec:proof-late-impos}

\LateNodesUnreachable*

The proof is split into several definitions and proposition to facilitate reading.
We use a graph formalism for the rest of this proof: a node is either a fluent or an elementary action.
There is an edge from $x$ to $y$ if $x$ is an action with effect $y$ or if $y$ is an action with condition $x$.
Each edge $e$ has a label $lbl(e)$ that is the delay from the condition to the action' start or from the action' start to the effect.

\newcommand{\sdset}{\ensuremath{\Omega}\xspace}

\subsubsection{Self-dependent set}
\label{sec:self-dependent-set}

We first identify sufficient conditions to declare a set of nodes
unreachable.  A node $X$ is the predecessor of a node $Y$ (noted
$pred(Y) = X$) if the latest value of $ea(Y)$ was updated by an edge
from $X$ to $Y$.  This is similar to the predecessor labels propagated
in a Dijkstra algorithm.  While \autoref{algo:prop} does not maintain this
information explicitly, it would be easy to add a predecessor field
for each node that would be updated every time its earliest appearance is
modified.

\begin{definition}[Predecessor cycle]
  A \emph{predecessor cycle} is a sequence
  $A_1 \rightarrow a_1 \rightarrow A_2 \rightarrow a_2 \dots A_n \rightarrow a_n \rightarrow A_1$ of edges
  where the source of an edge is the predecessor of its target
  (e.g. $pred(a_1) = A_1$).  
  Upper case nodes are action nodes and lower case nodes are fluents.
\end{definition}

\begin{proposition}\label{prop:pos-length}
  A predecessor cycle is of strictly positive length (i.e. the sum of
  the labels on the edges is strictly greater than 0).
\end{proposition}
\begin{proof}
  A cycle of predecessors means that an update of the first element
  (e.g. $A_1$) triggered an update of its direct successors
  (e.g. $a_1$) and all its indirect successors (e.g. $A_2$, $a_2$,
  $a_n$) including itself.  Since the earliest appearance can only be
  increased as a result of an update, the cycle has a strictly
  positive length (otherwise it would not have resulted in a greater
  value).
\end{proof}

\begin{proposition}\label{prop:rm-edge}
  In a predecessor cycle, at least one effect edge $A_i \rightarrow a_i$ can be
  removed without altering the problem.
\end{proposition}
\begin{proof}
  A predecessor cycle represents an invalid combination of first
  achievers of fluents in the cycle.  It means that having the action
  $A_i$ as the first achiever for the fluent $a_i$ (for all
  $i \in [1,n]$) would result in the condition that starting $A_1$ at
  a given time requires that $A_1$ had started at an earlier
  time.  This is trivially not possible, hence at least one fluent $a_k$
  in the cycle must be first achieved by an action other than
  $A_k$.  Since we are dealing with a delete-free model, the effect
  $a_k$ can be removed from $A_k$ without altering the problem.
\end{proof}

\begin{definition}[Self-dependent set]\label{def:sdset}
  We say that a set \sdset of action and fluent nodes is a
  \emph{self-dependent set} if:
  \begin{itemize}
  \item All nodes in \sdset have a predecessor in \sdset. 
  \item For any fluent $f$ in \sdset, all achievers of $f$ are in
    \sdset.
  \end{itemize}
\end{definition}

It should be noted that the first element of
Definition~\ref{def:sdset} implies that \emph{(i)} all elements of a
self-dependent set have been updated at least once \emph{(ii)} all
actions in \sdset depend on at least one fluent in \sdset. 

\begin{proposition}
  If a node $n$ is part of a self-dependent set \sdset, then $n$ is not
  reachable.
\end{proposition}
\begin{proof}
  We first show that there is a predecessor cycle composed exclusively
  of nodes in \sdset.  All nodes in \sdset have a predecessor and this
  predecessor is in \sdset.  Since \sdset is finite, there is at least
  one node in \sdset that is an indirect predecessor of itself.

  Since we have a predecessor cycle, we can safely remove an edge in
  this cycle without altering the problem.  This means that one fluent
  $f$ in \sdset is deprived of one of its achievers.   $f$ gets a new
  predecessor and its earliest appearance is updated.  Note that the
  new $pred(f)$ is still in \sdset by definition of a self-dependent
  set.

  In this new model, \sdset is still a self-dependent set.  The above
  steps can thus be repeated until one fluent of \sdset has no
  achievers left.  This fluent and all actions depending on it can be
  deleted.  The nodes that are left from \sdset still form a
  self-dependent set. The above procedure can thus be repeated until
  all fluents and actions of \sdset have been proved infeasible.
\end{proof}

\begin{example}
  The graph in \autoref{fig:ex-sdset} shows a problem with no possible
  actions and fluents.  We display a possible combination of
  predecessor edges (in red) to highlight the presence of a
  self-dependent set.  One self-dependent set in this problem is
  $\{b,A_1,A_2,a,B\}$.

  We have a cycle of predecessors $b \rightarrow A_1 \rightarrow a \rightarrow B \rightarrow b$ with an accumulated delay
  (sum of the labels) of 1.  This cycle can be read as \emph{``If $A_1$
    is the first achiever of $a$ and $B$ is the first achiever of $b$
    then $a$ can only be achieved at time $t$ if it was achieved at time
    $t-1$.''} This is of course not possible: either $b$ or $a$ needs
  another first achiever.  The only possibility is to select $A_2$ as
  first achiever for $a$ and $a$ can be removed from the effects of
  $A_1$.

  In this equivalent model, $A_2$ is the new predecessor of $a$ which
  results in a predecessor cycle $b \rightarrow A_2 \rightarrow a \rightarrow B \rightarrow b$.  Consequently, achieving
  $a$ and $b$ require selecting another first achiever for one of
  them.  Since we have no other options left, all nodes in this cycle
  are not possible.

  \begin{figure}[H] 
    \centering
    \begin{tikzpicture}[node distance=1cm]
      \node[action] (X) {$X$};
      \node[action,below of=X] (Y) {};
      \node[action,below of=Y] (A1) {$A_1$};
      \node[action,below of=A1] (A2) {$A_2$};
      \node[action,below of=A2] (B) {$B$};
      \node[fluent,left=2cm of Y] (x) {$x$};
      \node[fluent,right=2cm of Y] (y) {$y$};
      \node[fluent,left=2cm of A2] (b) {$b$};
      \node[fluent,right=2cm of A2] (a) {$a$};

      \path
      (X) edge[eff,pred] node[onedge] {1} (x)
      (A1) edge[eff,pred] node[onedge] {1} (y)
      (A1) edge[eff,pred] node[onedge] {1} (a)
      (A2) edge[eff] node[onedge] {1} (a)
      (B) edge[eff,pred] node[onedge] {1} (b)

      (x) edge[cond] node[onedge] {1} (A1)
      (b) edge[cond,pred] node[onedge] {1} (A1)
      (b) edge[cond,pred] node[onedge] {1} (A2)
      (a) edge[cond,pred] node[onedge] {-2} (B)
      (y) edge[cond,pred] node[onedge] {-5} (X)
      ;
    \end{tikzpicture}
    \caption{Problem in graph representation.  Edges in red represent a
      possible assignment of predecessors at some point in the
      propagation.}
    \label{fig:ex-sdset}
  \end{figure}
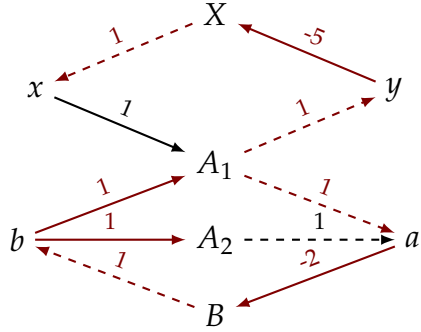
\end{example}

\subsubsection{From propagation to the identification of self-dependent set}
\label{sec:from-prop-self}

\newcommand{\lateset}{\ensuremath{\mathcal L}\xspace}

We have identified sufficient conditions to declare a group of nodes
unreachable.  We now show how the identification of such a set can be
integrated into Algorithm~\ref{algo:prop}.

We say that a set of nodes \lateset is late if any node in \lateset has an
earliest appearance at least \dmax time units greater than any node
not in \lateset ; \dmax being the maximum delay of any edge of the graph.

\begin{equation}
\forall x \notin \lateset, y \in \lateset, ea(x) + \dmax < ea(y)
\end{equation}

The intuition behind the definition of a late set is that all late
nodes are separated from non-late nodes by a temporal gap.  Furthermore,
this temporal gap is big enough so that the earliest appearances of
late nodes could not have been influenced by a non-late node (i.e. the
predecessor of a late node is a late node).

\begin{proposition}
  If \lateset is a set of late nodes, then \lateset is a self-dependent set.
\end{proposition}
\begin{proof}
  We prove the two conditions of a set to be self-dependent.

  Let $Y$ be the current predecessor of a node $X \in \lateset$ and $e_{YX}$ the
  edge from $Y$ to $X$.  At the last update of $X$, $ea(X)$ was set to
  $ea(Y) + lbl(e_{YX})$.  We know that $ea(Y)$ can only increase and that
  $lbl(e_{YX}) \leq \dmax$.  Consequently, we still have
  $ea(Y) + \dmax \geq ea(X)$, meaning that $Y$ is necessarily in \lateset.  We have shown
  that if $X$ is in \lateset, $pred(X)$ is in \lateset.

  Next let $x$ be a fluent in \lateset.  Because a fluent takes the
  minimum earliest appearance of all its achievers, no such achiever
  can be more than \dmax time units before it.  All achievers of a
  fluent in \lateset are therefore in \lateset as well.
\end{proof}

\begin{corollary}
  Any node in a late set is not reachable.
\end{corollary}

\begin{proposition}
  Any non-reachable node will eventually be part of a late set.
\end{proposition}
\begin{proof}
  All earliest appearances of reachable nodes converge towards a finite
  value.  On the other hand, earliest appearances of non-reachable nodes
  diverge towards $+\infty$.  At some point, the earliest appearance of unreachable nodes
  will be greater by \dmax than the latest reachable node.
\end{proof}

A possible implementation for detecting unreachable nodes in \autoref{algo:prop} is thus to check whether some nodes are identifiable as late during propagation and delete all those nodes.
Such an implementation is guaranteed to finish and remove all unreachable nodes.

\section{Example of ANML domains}

\begin{multicols}{2}

  \lstset{style=anml}

\subsection{Dock Worker}
\label{sec-dock-anml}

Below is a partial view of the dock-worker domain that served as an illustration of the representation in this paper.

\vspace*{-15pt}
\lstset{escapeinside={<@}{@>}}
\begin{lstlisting}
/*** Types, <@\textcolor{darkgreen}{functions}@>, state variables ***/
type Truck with {
  // loc: Time x Truck => Dock
  fluent Dock loc;
};
type Container with {
  fluent (Dock or Truck or Ship) pos;
};
...

// travel_time: Dock x Dock => Integer
constant int travel_time(Dock d1, Dock d2);
// connected: Dock x Dock => Boolean
constant bool connected(Dock d1, Dock d2);


/*** Actions ***/

// move <@\textcolor{darkgreen}{action}@> of Figure 1
action move(Truck r, Dock d1, Dock d2) {
  duration := travel_time(d1,d2);
  [all] r.loc == d1 :-> d2;
  [start,t] occupant(d1) == r :-> NIL;
  [t2,end] occupant(d2) == NIL :-> r;
  t < t2;
};

// the two high-level <@\textcolor{darkgreen}{actions}@> achieving the transport task (Figure 2).
action transport(Container c, Dock d) {
  motivated; // task-dependant

  // m1-transport
  :decomposition{
    [all] c.pos == d;
  };
 
  // m2-transport
  :decomposition{
    constant Truck r;
    constant Dock ds;
    connected(ds,d);
    ds != d;
    [start] r.loc == ds;
    [start] c.pos == ds;
    // three totally <@\textcolor{darkgreen}{ordered}@> subtasks
    [all] ordered( 
      load(r,c,ds),
      move(r,ds,d),
      unload(r,c,d)
    );
  };
};


/*** Instances and Constants ***/
instance Dock dock1, dock2, dock3;
instance Truck r1, r2;
instance Ship ship1;

travel_time(dock1, dock2) := 7;
travel_time(dock2, dock3) := 9;


/*** Problem Statement (Figure 3) ***/

// initial state and expected evolution
[start] r1.loc := dock1;
[start] r2.loc := dock2;
[start] c1.pos := ship1;
[start+10] ship1.docked := pier1;
[t_undock] ship1.docked := NIL;
start+20 <= t_undock; t_undock <= start+30;

// goals and tasks
[end] r1.loc == dock1;
[end] r1.loc == dock2;
[start,end] contains transport(c1,dock3);

\end{lstlisting}
  
\subsection{Blocks-\parthier}
\label{sec:blocks-parthier}

\begin{lstlisting}
type Location; 
type Block < Location;

predicate clear(Block b);
predicate handempty();
function Location on(Block b);

instance Location TABLE, HAND;

action pickup(Block b) {
  duration := 5;
  [all] clear(b) == true;
  [all] on(b) == TABLE :-> HAND;
  [all] handempty == true :-> false;
};

action putdown(Block b) {
  duration := 5;  
  [all] clear(b) == true;
  [all] on(b) == HAND :-> TABLE;
  [all] handempty == false :-> true;
};

action stack(Block b, Block c) {
  motivated; // i.e. task-dependent
  duration := 5;  
  [all] on(b) == HAND :-> c;
  [all] handempty == false :-> true;
  [all] clear(c) == true :-> false;
  [all] clear(b) == true;
};

action unstack(Block b, Block c) {
  duration := 5;  
  [all] on(b) == c :-> HAND;
  [all] handempty == true :-> false;
  [all] clear(b) == true;
  [all] clear(c) == false :-> true;
};

action DoStack(Block a, Block b) {
  motivated; // i.e. task-dependent
  :decomposition {
    [all] on(a) == b;
  };
  :decomposition {
    [all] stack(a,b);
  };
};

\end{lstlisting}

\subsection{Blocks-\fullhier}
\label{sec:blocks-fullhier}

\begin{lstlisting}
type Location; 
type Block < Location;

predicate clear(Block b);
predicate handempty();
function Location on(Block b);

instance Location TABLE, HAND;

action pickup(Block b) {
  motivated; // i.e. task-dependent
  duration := 5;
  [all] clear(b) == true;
  [all] on(b) == TABLE :-> HAND;
  [all] handempty == true :-> false;
};

action putdown(Block b) {
  motivated; // i.e. task-dependent
  duration := 5;  
  [all] clear(b) == true;
  [all] on(b) == HAND :-> TABLE;
  [all] handempty == false :-> true;
};

action stack(Block b, Block c) {
  motivated; // i.e. task-dependent
  duration := 5;  
  [all] on(b) == HAND :-> c;
  [all] handempty == false :-> true;
  [all] clear(c) == true :-> false;
  [all] clear(b) == true;
};

action unstack(Block b, Block c) {
  motivated; // i.e. task-dependent
  duration := 5;  
  [all] on(b) == c :-> HAND;
  [all] handempty == true :-> false;
  [all] clear(b) == true;
  [all] clear(c) == false :-> true;
};

action uncover(Block a) {
  motivated; // i.e. task-dependent
  :decomposition {
    [all] clear(a) == true;
  };
  :decomposition {
    [start] clear(a) == false;
    constant Block onA;
    [start] on(onA) == a;
    [all] ordered(
      uncover(onA),
      unstack(onA,a),
      putdown(onA));
  };
};

action DoStack(Block a, Block b) {
  motivated; // i.e. task-dependent
  :decomposition {
    [all] on(a) == b;
  };
  :decomposition {
    [start] on(a) == TABLE;
    [all] ordered(
      uncover(a),
      uncover(b),
      p: pickup(a),
      s: stack(a, b));
      end(p) = start(s);
  };
  :decomposition {
    constant Block other;
    other != TABLE;
    [start] on(a) == other;
    [all] ordered(
      uncover(a),
      uncover(b),
      u: unstack(a, other),
      s: stack(a, b));
    end(u) = start(s);

  };
};
\end{lstlisting}
\end{multicols}

\end{document}